%% file: neurips_2024.tex
\documentclass{article}

\PassOptionsToPackage{numbers, compress}{natbib}

\usepackage[final]{configurations/neurips_2024}

\usepackage[utf8]{inputenc} 
\usepackage[T1]{fontenc}    
\usepackage{hyperref}       
\usepackage{url}            
\usepackage{booktabs}       
\usepackage{amsfonts}       
\usepackage{nicefrac}       
\usepackage{microtype}      
\usepackage{xcolor}         

\usepackage{graphicx}
\usepackage{subcaption}
\usepackage[compact]{titlesec}
\renewcommand{\paragraph}[1]{\noindent\textbf{#1}}
\usepackage[page]{appendix}
\usepackage{authblk}

\usepackage{enumitem}
\usepackage{soul}
\usepackage{multirow}
\usepackage{xcolor}
\usepackage{tikz}
\usetikzlibrary{bayesnet}
\usetikzlibrary{arrows}
\usepackage{caption}
\usetikzlibrary{backgrounds}

\usepackage{wrapfig}

\usepackage{algorithmic}
\usepackage[ruled,vlined,noend,linesnumbered]{algorithm2e}
\SetKwRepeat{Do}{do}{while}

\usetikzlibrary{arrows.meta,positioning,bending}

\usepackage{amsmath,amssymb,amsthm}
\usepackage{thmtools,thm-restate}
\usepackage{mathtools}

\usepackage[capitalize,noabbrev]{cleveref}
\theoremstyle{plain}
\newtheorem{theorem}{Theorem}[section]

\newtheorem{lemma}[theorem]{Lemma}

\theoremstyle{definition}
\newtheorem{definition}[theorem]{Definition}
\theoremstyle{remark}

\usepackage{bm}
\input{configurations/math_commands.tex}

\newcommand{\zdd}[1]{\zz_{\text{d},#1}}

\newcommand{\node}[1]{\mathsf{#1}}
\newcommand{\set}[1]{\mathbf{#1}}
\newcommand{\setset}[1]{\mathcal{#1}}
\newcommand{\setsetset}[1]{\mathbb{#1}}
\newcommand{\graph}{\mathcal{G}}
\newcommand{\graphp}{\mathcal{G'}}
\newcommand{\graphpp}{\mathcal{G''}}

\newcommand{\parents}{\text{Pa}_{\graph}}
\newcommand{\parentsp}{\text{Pa}_{\graphp}}

\newcommand{\purechildren}{\text{PCh}_{\graph}}
\newcommand{\purechildrenp}{\text{PCh}_{\graphp}}

\usepackage{titletoc}
\newcommand\DoToC{%
  \startcontents
  \printcontents{}{1}{\hrulefill\vskip0pt}
  \vskip0pt \noindent\hrulefill
  }

\newcommand{\ours}[1]{\textcolor{blue}{ {#1}}}

\title{ Learning Discrete Concepts in Latent Hierarchical Models }

\author[1]{\textbf{Lingjing Kong}}
\author[1,2]{\textbf{Guangyi Chen}}
\author[3]{\textbf{Biwei Huang}}
\author[1,2]{\textbf{Eric P. Xing}}
\author[1]{\textbf{Yuejie Chi}}
\author[1,2]{\textbf{Kun Zhang}}

\affil[1]{Carnegie Mellon University}
\affil[2]{Mohamed bin Zayed University of Artificial Intelligence}
\affil[3]{University of California San Diego}

\begin{document}

\maketitle

\input{sections/abstract}

\input{sections/introduction}

\input{sections/related_work}

\input{sections/formulation}
\input{sections/theory}

\input{sections/synthetic_data_experiments}

\input{sections/connection}

\input{sections/experiments}

\input{sections/conclusion}

\clearpage

\input{sections/acknowledgement}


\bibliography{references}
\bibliographystyle{unsrtnat}

\newpage
\appendix
\onecolumn

\input{sections/appendix}



\end{document}

%% file: configurations/math_commands.tex
\DeclareMathOperator*{\argmax}{arg\,max}

\newcommand{\Pb}[1]{{\mathbb{P}}\left(#1\right) }
\newcommand{\EEb}[2]{{\mathbb E}_{#1}\left[#2\right] }

\providecommand{\cc}{\mathbf{c}}
\providecommand{\dd}{\mathbf{d}}

\providecommand{\xx}{\mathbf{x}}
\providecommand{\yy}{\mathbf{y}}
\providecommand{\zz}{\mathbf{z}}

\providecommand{\cC}{\mathcal{C}}

\providecommand{\cG}{\mathcal{G}}

\providecommand{\cI}{\mathcal{I}}

\providecommand{\cO}{\mathcal{O}}

\providecommand{\cS}{\mathcal{S}}

\providecommand{\cX}{\mathcal{X}}

\providecommand{\R}{\mathbb{R}} 
\providecommand{\N}{\mathbb{N}} 

\providecommand{\mA}{\mathbf{A}}
\providecommand{\mB}{\mathbf{B}}
\providecommand{\mC}{\mathbf{C}}
\providecommand{\mD}{\mathbf{D}}
\providecommand{\mE}{\mathbf{E}}

\providecommand{\mL}{\mathbf{L}}

\providecommand{\mP}{\mathbf{P}}

\providecommand{\mV}{\mathbf{V}}

\providecommand{\mX}{\mathbf{X}}

\providecommand{\mZ}{\mathbf{Z}}

\newcommand{\abs}[1]{\left\lvert#1\right\rvert}

%% file: sections/abstract.tex
\begin{abstract}
    Learning concepts from natural high-dimensional data (e.g., images) holds potential in building human-aligned and interpretable machine learning models.
    Despite its encouraging prospect, formalization and theoretical insights into this crucial task are still lacking.
    In this work, we formalize concepts as discrete latent causal variables that are related via a hierarchical causal model that encodes different abstraction levels of concepts embedded in high-dimensional data (e.g., a dog breed and its eye shapes in natural images).
    We formulate conditions to facilitate the identification of the proposed causal model, which reveals when learning such concepts from unsupervised data is possible.
    Our conditions permit complex causal hierarchical structures beyond latent trees and multi-level directed acyclic graphs in prior work and can handle high-dimensional, continuous observed variables, which is well-suited for unstructured data modalities such as images.
    We substantiate our theoretical claims with synthetic data experiments.
    Further, we discuss our theory's implications for understanding the underlying mechanisms of latent diffusion models and provide corresponding empirical evidence for our theoretical insights.
\end{abstract}

%% file: sections/introduction.tex
\section{Introduction}

Learning semantic discrete concepts from unstructured high-dimensional data, such as images and text, is crucial to building machine learning models with interpretability, transferability, and compositionality, as empirically demonstrated by extensive existing work~\citep{gal2022image,jahanian2020steerability,härkönen2020ganspace,shen2020interpreting,wu2020stylespace,ruiz2023dreambooth,burgess2019monet,locatello2020objectcentric,du2021unsupervised,du2021unsupervised3d,liu2023unsupervised}.
Despite these empirical successes, limited work is devoted to the theoretical front: the notions of concepts and their relations are often heuristically defined. For example, concept bottleneck models~\cite{koh2020concept,zarlenga2022concept} use human-specified annotations and recent methods~\cite{oikarinen2022clip,moayeri2023text,moayeri2023text2concept} employ pretrained multimodal models like CLIP~\cite{radford2021learning} to explain features with neural language.
This lack of rigorous characterization impedes a deeper understanding of this task and the development of principled learning algorithms. 

In natural images, the degree/extent of certain attributes (e.g., position, lighting) is often presented in a continuous form and main concepts of practical concern are often discrete in nature (e.g., object classes and shapes).
Moreover, these concepts are often statistically dependent, with the dependence potentially resulting from some higher-level concepts.
For example, the correlation between a specific dog's eye features and fur features may arise from a high-level concept for breeds (Figure~\ref{fig:teaser}).
Similarly, even higher-level concepts may exist and induce dependence between high-level concepts, giving rise to a hierarchical model that characterizes all discrete concepts at different abstraction levels underlying high-dimensional data distributions. 
In this work, we focus on concepts that can be defined as discrete latent variables and related via a hierarchical model.
Under this formalization, the query on the recoverability of concepts and their relations from unstructured high-dimensional distribution (e.g., images) amounts to the following causal identification problem:
\vspace{-0.1cm}
\begin{center}
    \textit{Under what conditions is the discrete latent hierarchical causal model identifiable from high-dimensional continuous data distributions?}
\end{center}
\vspace{-0.2cm}


Identification theory for latent hierarchical causal models has been a topic of sustained interest.
Recent work~\citep{xie2022identification,huang2022latent,dong2023versatile} investigates identification conditions of latent hierarchical structures under the assumption that the latent variables are continuous and influence each other through linear functions.
The linearity assumption fails to handle the general nonlinear influences among discrete variables.
Another line of work focuses on discrete latent models.
\citet{Pearl88,choi2011learning} study latent trees with discrete observed variables.
The tree structure can be over-simplified to capture the complex interactions among concepts from distinct abstract levels (e.g., multiple high-level concepts can jointly influence a lower-level one).
\citet{gu2023bayesian} assume that binary latent variables can be exactly grouped into levels and causal edges often appear between adjacent levels, which can also be restrictive.
Moreover, these papers assume observed variables are discrete, falling short of modeling the continuous distribution like images as the observed variables.
Similar to our goal, \citet{kivva2021learning} show the discrete latent variables adjacent to the potentially continuous observed variables can be identified.
However, their theory assumes the absence of higher-level latent variables and thus cannot handle latent hierarchical structures.


In this work, we show identification guarantees for the discrete hierarchical model under mild conditions on the generating function and causal structures. 
Specifically, we first show that when continuous observed variables (i.e., the leaves of the hierarchy) preserve the information of their adjacent discrete latent variables (i.e., direct parents in the graph), we can extract the discrete information from the continuous observations and further identify each discrete variable up to permutation indeterminacy.
Given these ``low-level'' discrete latent variables, we establish graphical conditions to identify the discrete hierarchical model that fully explains the statistical dependence among the identified ``low-level'' discrete latent variables.
Our conditions permit multiple paths within latent variable pairs and flexible locations of latent variables 
, encompassing a large family of graph structures including as special cases non-hierarchical structures~\citep{kivva2021learning}, trees~\citep{Pearl88,choi2011learning,drton2015marginal,zhang2004hierarchical} and multi-level directed acyclic graphs (DAGs)~\citep{gu2023bayesian,anandkumar13learning} (see example graphs in Figure~\ref{fig:comparison_graphs}).
Taken together, our work establishes theoretical results for identifying the discrete latent hierarchical model governing high-dimensional continuous observed variables, which to the best of our knowledge is the first effort in this direction.
We corroborate our theoretical results with synthetic data experiments.


As an implication of our theorems, we discuss a novel interpretation of the state-of-the-art latent diffusion (LD) models~\citep{rombach2021highresolution} through the lens of a hierarchical concept model.
We interpret the denoising objective at different noise levels as estimating latent concept embeddings at corresponding hierarchical levels in the causal model, where a higher noise level corresponds to high-level concepts.
This perspective explains and unifies these seemingly orthogonal threads of empirical insights and gives rise to insights for potential empirical improvements.
We deduce several insights from our theoretical results and verify them empirically.
In summary, our main contributions are as follows.
\begin{itemize}[leftmargin=1em, topsep=0.5pt, partopsep=0pt, itemsep=-0.0em]
    \setlength\itemsep{-0.0em}
    \item We formalize the framework of learning concepts from high-dimensional data as a latent-variable identification problem, capturing concepts at different abstraction levels and their interactions. 
    
    \item We present identification theories for the discrete latent hierarchical model. 
    To the best of our knowledge, our result is the first to address discrete latent hierarchical model beyond trees~\citep{Pearl88,zhang2004hierarchical} and multi-level DAGs~\citep{gu2023bayesian} while capable of handling high-dimensional observed variables.
    \item We provide an interpretation of latent diffusion models as hierarchical concept learners. 
    We supply empirical results to illustrate our interpretation and showcase its potential benefits in practice.
\end{itemize}

%% file: sections/related_work.tex
\section{Related Work} \label{sec:related_work}

\paragraph{Concept learning.}
In recent years, a significant strand of research has focused on employing labeled data to learn concepts in generative models' latent space for image editing and manipulation~\citep{gal2022image,jahanian2020steerability,härkönen2020ganspace,shen2020interpreting,wu2020stylespace,ruiz2023dreambooth}.
Concurrently, another independent research trajectory has been exploring unsupervised concept discovery and its potential to learn more compositional and transferable models~\citep{burgess2019monet, locatello2020objectcentric,du2021unsupervised,du2021unsupervised3d,liu2023unsupervised}.
Concurrently, a plethora of work has been dedicated to extracting interpretable concepts from high-dimensional data such as images. 
Concept-bottleneck~\citep{koh2020concept} first predicts a set of human-annotated concepts as an intermediate stage and then predicts the task labels from these intermediate concepts. This paradigm has attracted a large amount of follow-up work~\citep{zarlenga2022concept,yuksekgonulpost,kim2023probabilistic,havasi2022addressing,shang2024incremental,chauhan2023interactive}. 
A recent surge of pre-trained multimodal models (e.g., CLIP~\citep{radford2021learning}) can explain the image concepts through text directly~\citep{oikarinen2022clip,moayeri2023text,moayeri2023text2concept}. 

\paragraph{Latent variable identification.}
Complex real-world data distributions often possess a hierarchical structure among their underlying latent variables.
The identification conditions of latent hierarchical structures are investigated under the assumption that the latent variables are continuous and influence each other through linear functions~\citep{xie2022identification,huang2022latent,dong2023versatile} and nonlinear functions~\citep{kong2023identification}.
In addition, prior work~\citep{Pearl88,zhang2004hierarchical,choi2011learning,gu2023bayesian} studies fully discrete cases and thus falls short of modeling the continuous observed variables like images.
To identify latent variables under nonlinear transformations, a line of work~\citep{khemakhem2020variational,khemakhem2020icebeem,hyvarinen2016unsupervised,hyvarinen2019nonlinear} assumes the availability of auxiliary information (e.g., domain/class labels) and that the latent variables' probability density functions have sufficiently different derivatives over domains/classes.
Another line of studies~\citep{brady2023provably,lachapelle2023additive} refrains from the auxiliary information by assuming sparsity and mechanistic independence, disregarding causal structures among the latent variables.

Please refer to Section~\ref{app:related_work} for more extensive related work and discussion.

%% file: sections/formulation.tex
\section{Discrete Hierarchical Models} \label{sec:formulation}

\paragraph{Data-generating process.}
We formulate the data-generating process as the following latent-variable model. Let $ \xx $ denote the continuous observed variables $ \xx:= [ x_{1}, \cdots, x_{n} ] \in \cX \subset \R^{d_{\xx}} $ which represents the high-dimension data we work with in practice (e.g., images). \footnote{
    We use the unbolded symbol $x_{i}$ to distinguish each observed variable $x_{i}$ from the collection $\xx$. 
    Our theory allows $x_{i}$ to be multi-dimensional.
}
Let $\dd:= [ d_{1}, \cdots, d_{n_{d}} ]$ be discrete latent variables that are direct parents to $\xx$ (as shown in Figure~\ref{fig:teaser}(b)) and take on values from finite sets, i.e., $d_{i} \in \Omega_{i}^{(d)}$ for all $i \in [d_{i}]$ and $ 2 \le \abs{\Omega_{i}^{(d)}} < \infty$. 
We denote the joint domain as $\Omega^{(d)} := \Omega_{1}^{(d)} \times \cdots \times \Omega_{n_{d}}^{(d)}$.
These discrete variables are potentially related to each other causally (e.g., $d_{4}$ and $d_{5}$ in Figure~\ref{fig:teaser}(c)) or via higher-level latent variables (e.g., $d_{1}$ and $d_{2}$ in Figure~\ref{fig:teaser}(c)).
Let $\cc:= [ c_{1}, \cdots, c_{n_{c}} ] \in \cC \subset \R^{d_{c}} $ be continuous latent variables that represent the continuous information conveyed in observed variables $\xx$.
The generating process is defined in Equation~\ref{eq:discrete_generating} and illustrated in Figure~\ref{fig:teaser}(a).
\setlength{\abovedisplayskip}{-0.1cm}
\setlength{\belowdisplayskip}{0.1 cm}
\begin{align} \label{eq:discrete_generating}
    \xx := g( \dd, \cc ),
\end{align}
where we denote the generating function with $g: [\dd, \cc] \mapsto \xx$.
We denote the resultant bipartite graph from $[\dd, \cc]$ to $\xx$ as $\Gamma$.
In this context of image generation, the discrete subspace $\dd$ gives a description of concepts present in the image $\xx$ (e.g., a dog's appearance, background objects), and the continuous subspace $\cc$ controls extents/degrees of specific attributes (e.g., sizes, lighting, and angles).


\begin{wrapfigure}{r}{9cm}
    \centering
    \vspace{-0.4cm}
    \includegraphics[width=9cm]{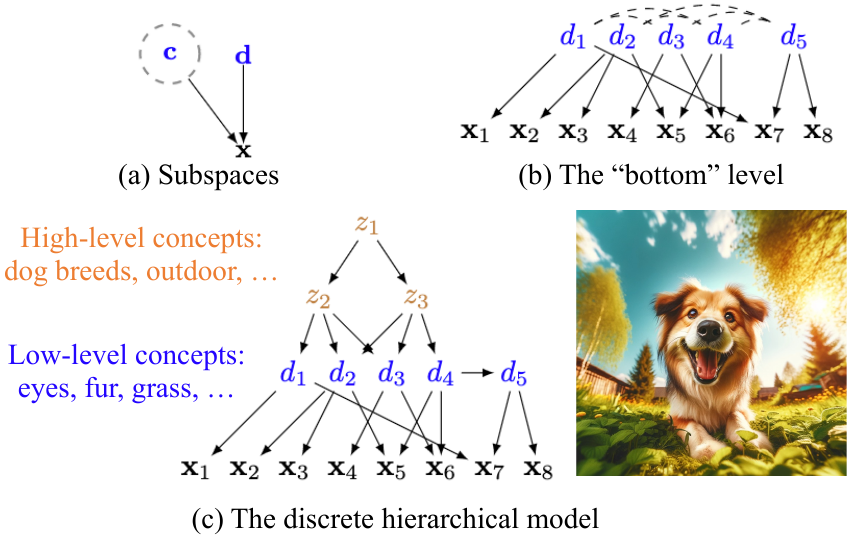}
    \vspace{-0.45cm}
    \caption{\footnotesize
        \textbf{Latent hierarchical graphs.}
        The dashed circle in (a) indicates that the continuous variable $\cc$ can be viewed as an exogenous variable.
        Dashed edges in (b) indicate potential statistical dependence.
    }
    \label{fig:teaser}
    \vspace{-0.75cm}
\end{wrapfigure}
\paragraph{Discrete hierarchical models.}
As discussed above, discrete variables $  d_{1}, \dots, d_{n_{d}} $ represent distinct concepts that may be dependent either causally or purely statistically via higher-level concepts, as visualized in Figure~\ref{fig:teaser}(c).
For instance, the dog's eye features and nose features are dependent, which a higher-level concept ``breeds'' could explain.
We denote such higher-level latent discrete variables as $ \zz:= [ z_{1}, \cdots, z_{n_{z}} ]$, where $z_{i} \in \Omega_{i}^{(z)}$ for all $i \in [z_{i}]$ and $ 2 \le \abs{\Omega_{i}^{(z)}} < \infty$ and $\Omega^{(z)} := \Omega_{1}^{(z)} \times \cdots \times \Omega_{n_{z}}^{(z)}$.
Graphically, these variables $\zz$ are not directly adjacent to observed variables $\xx$ (Figure~\ref{fig:teaser}(c)).
High-level discrete variables $\zz$ may constitute a hierarchical structure until the dependence in the system is fully explained.
Since the discrete variables encode major semantic concepts in the data, this work primarily concerns discrete variables $\dd$ and its underlying causal structure.
The continuous subspace $\cc$ can be viewed as exogenous variables and is often omitted in the causal graph (e.g., Figure~\ref{fig:teaser}(b)).
We leave identifying continuous attributes in $\cc$ as future work.

Given this, we define the discrete hierarchical model as follows.
The discrete hierarchical model (Figure~\ref{fig:teaser}(c)) $\cG:= (\mE, \mV)$ is a DAG that comprises discrete latent variables $ d_{1}, \cdots, d_{n_{d}}, z_{1}, \cdots, z_{n_{z}} $.
We denote that directed edge set with $\mE$ and the collection of all variables with $\mV:= \{\mD, \mZ \}$, where $\mD$ and $\mZ$ are vectors $\dd$ and $\zz$ in a set form and all leaf variables in $\cG$ belong to $\mD$.
We assume the distribution over all variables $\mV$ respects the Markov property with respect to the graph $\cG$.
We denote all parents and children of a variable with $ \text{Ch}(\cdot) $ and $\text{Pa} (\cdot) $ respectively and define the neighbors as $ \text{ne} (\cdot):= \text{Ch}(\cdot) \cup \text{Pa}(\cdot) $. 
We say a variable set $\mA$ are pure children of $\mB$, iff $\parents(\mA) = \cup_{A_i \in \mA} \parents(A_i) = \mB$ and $\mA \cap \mB=\emptyset$.
As shown in Figure~\ref{fig:teaser}(c), $x_{1}$ is a pure child of $d_{1}$.

\paragraph{Objectives.}
Formally, given only the observed distribution $p(\xx)$, we aim to:
\setlength{\abovedisplayskip}{0pt}
\setlength{\belowdisplayskip}{0pt}
\begin{enumerate}[leftmargin=2em,topsep=0.5pt,itemsep=-0.0em]
    \item \label{obj:component} identify discrete variables $\dd$ and the bipartite graph $\Gamma$;
    \item \label{obj:hierarchical} identify the hierarchical causal structure $\cG$.
\end{enumerate}

%% file: sections/theory.tex
\section{Identification of Discrete Latent Hierarchical Models} \label{sec:theory}

We present our theoretical results on the identifiability of discrete latent variables $\dd$ and the bipartite graph $\Gamma$ in Section~\ref{subsec:discrete_component} (i.e., Objective~\ref{obj:component}) and the hierarchical model $\cG$ in Section~\ref{subsec:hierarchy_identification} (i.e., Objective~\ref{obj:hierarchical}).

\paragraph{Additional notations.}
We denote the set containing components of $\xx$ with $\mX$, the set of all variables with $\mV^{*}:= \mV \cup \mX$, the entire edge set with $ \mE^{*} := \mE \cup \Gamma $, and the entire causal model with $\cG^{*} := ( \mV^{*}, \mE^{*} )$.
As the true generating process involves $\dd$, $\cc$, $g$, $\Gamma$, and $\cG$ (defined in Section~\ref{sec:formulation}), we define their statistical estimates with $\hat{\dd}$, $\hat{\cc}$, $\hat{g}$, and $\hat{\Gamma}$ through maximum likelihood estimation over the full population $p(\xx)$ while respecting conditions on the true generating process.
We use $\abs{ \text{Supp} (\mL) }$ for the cardinality of a discrete variable set $\mL$'s support (all joint states) and $\mP_{\mA, {\mB}}$ for the joint probability table whose two dimensions are the states of discrete variable sets $\mA$ and $\mB$ respectively.

\subsection{General Conditions for Discrete Latent Models}

It is well known that causal structures cannot be identified without proper assumptions.
For instance, one may merge two adjacent discrete variables $d_{1} \in \Omega^{(d)}_{1} $ and $d_{2} \in \Omega^{(d)}_{2}$ into a single variable $\tilde{d}\in \Omega^{(d)}_{1} \times \Omega^{(d)}_{2}$ while preserving the observed distribution $p(\xx)$.
We introduce the following basic conditions on the discrete latent model to eliminate such ill-posed situations.


\setlength{\abovedisplayskip}{0pt}
\setlength{\belowdisplayskip}{0pt}
\begin{restatable}[General Latent Model Conditions]{condition}{basicconditions} \label{cond:basic_model} {\ }
    \begin{enumerate}[label=\roman*,leftmargin=2em, topsep=0.5pt, partopsep=0pt, itemsep=-0.0em]
    \setlength\itemsep{-0.0em}
    
        \item \label{asmp:nondegeneracy} [Non-degeneracy]: $ \Pb{ \dd = k_{1}, \zz = k_{2}} > 0 $, for all $(k_1, k_{2}) \in \Omega^{(d)} \times \Omega^{(z)}$; for all variable $ v \in \mV^{*} $, $ \Pb{ v | \text{Pa}(v) = k_{1} } \neq \Pb{ v | \text{Pa}(v) = k_{2} }$ if $k_{1}\neq k_{2}$.

        \item \label{asmp:twins} [No-twins]: Distinct latent variables have distinct neighbors $\text{ne}(v_{1}) \neq \text{ne}(v_{2})$, if $v_{1} \neq v_{2} \in \mV$.
        
        \item \label{asmp:maximality} [Maximality]: There is no DAG $ \tilde{\cG}^{*} := ( \tilde{\mV}^{*}, \tilde{\mE}^{*}) $ resulting from splitting a latent variable in $\cG^{*}$ (i.e., turning $z_{i}$ into $\tilde{z}_{i,1}$ and $\tilde{z}_{i,2}$ with identical neighbors and cardinality $ |\Omega^{z}_{i} | = | \tilde{\Omega}^{z}_{i,1} | + | \tilde{\Omega}^{z}_{i,2} | $ ), such that $\Pb{ \tilde{\mV}^{*} }$ is Markov w.r.t. $\tilde{\cG}^{*}$ and $\tilde{\cG}^{*}$ satisfies \ref{asmp:twins}.
    \end{enumerate}
\end{restatable}
\vspace{-0.15cm}
\paragraph{Discussion.}
Condition~\ref{cond:basic_model} is a necessary set of conditions for identifying latent discrete models, which is employed and discussed extensively~\citep{kivva2021learning,kivva2022identifiability}.
Intuitively, Condition~\ref{cond:basic_model}-\ref{asmp:nondegeneracy} excludes dummy discrete states and graph edges that exert no influence on the observed variables $\xx$. 
Condition~\ref{cond:basic_model}-\ref{asmp:twins},\ref{asmp:maximality} constrain the latent model to be the most informative graph without introducing redundant latent variables, thus forbidding arbitrary merging and splitting over latent variables.

\subsection{Discrete Component Identification} \label{subsec:discrete_component}

We show with access to only the observed data $\xx$, we can identify each discrete component $d_{i}$ up to permutation indeterminacy (Definition~\ref{def:permutation}) and a corresponding bipartite graph equivalent to $\Gamma$.

\begin{definition}[Component-wise Identifiability] \label{def:permutation}
    Variables $\dd \in \N^{n_{d}}$ and $\hat{\dd} \in \N^{n_{d}}$ are identified component-wise if there exists a permutation $\pi$, such that $ \hat{d}_{i} = h_{i}(d_{\pi(i)})$ with invertible function $ h_{i} $. 
\end{definition}
\vspace{-0.15cm}
That is, our estimation $\hat{d}_{i}$ captures full information of $d_{\pi(i)}$ and no information from $ d_{j} $ such that $j\neq \pi(i)$. \footnote{
    We use ``components'' to refer to individual discrete variables $d_{i}$ in the vector $\dd$. 
}
The permutation is a fundamental indeterminacy for disentanglement~\citep{hyvarinen2016unsupervised,hyvarinen2019nonlinear,khemakhem2020icebeem,kivva2021learning}.

\paragraph{Remarks on the problem.}
A large body of prior work~\citep{hyvarinen2016unsupervised,khemakhem2020variational,vonkugelgen2021selfsupervised} requires continuous or even differentiable density function over all latent variables and domain/class labels or counterfactual counterparts to generate variation.
Thus, their techniques do not transfer naturally to our latent space with both continuous and discrete parts $[\cc, \dd]$ and no supervision of any form.
With a similar goal, \citet{kivva2021learning} assumes access to an oracle (Definition~\ref{def:oracle}) to the mixture distribution over $p(\xx)$, which is not directly available in the general case here.
\citet{kivva2022identifiability} assumes a specific parametric generating process, whereas we focus on a generic non-parametric generative model (Equation~\ref{eq:discrete_generating}).


\paragraph{High-level description of our proposed approach.}
We decompose the problem into two tractable subproblems: 1) extracting the global discrete state $\dd$ from the mixing with the continuous variable $\cc$; 2) further identifying each discrete component $d_{i}$ from the mixing with other discrete components $d_{j}$ ($i\neq j$) and the causal graph $\Gamma$.
For 1), we show that, perhaps surprisingly, minimal conditions on the generating function $g$ suffice to remove the information of $\cc$ and thus identify the global state of $\dd$.
For 2), we observe that the identification results in 1) can be viewed as a mixture oracle over $p(\xx)$, which enables us to employ techniques from \citet{kivva2021learning} to solve the problem.

We introduce key conditions and formal theoretical statements as follows.

\setlength{\abovedisplayskip}{0pt}
\setlength{\belowdisplayskip}{0pt}
\begin{restatable}[Discrete Components Identification]{condition}{discreteconditions} \label{cond:discrete_component_conditions} {\ } 
    \begin{enumerate}[label=\roman*,leftmargin=2em,itemsep=0em]
        \item \label{asmp:openspace} [Connected Spaces] The continuous support $\cC \subset \R^{n_{c}}$ is closed and connected.
        \item \label{asmp:invertibility} [Invertibility \& Continuity]: 
        The generating function $g$ in equation~\ref{eq:discrete_generating} is invertible, and for any fixed $\dd$, $g(\dd, \cdot)$ and its inverse are continuous.
        \item \label{asmp:subset_influences} [Non-Subset Observed Children]:
        For any pair $ d_{i} $ and $ d_{j}$, one's observed children are not the subset of the other's, $ \text{Ch}_{\Gamma} ( d_{i} ) \not\subset \text{Ch}_{\Gamma} ( d_{j} ) $.
     
    \end{enumerate}
\end{restatable}

\paragraph{Discussion on the conditions.}
Condition~\ref{cond:discrete_component_conditions}-\ref{asmp:openspace} requires the continuous support $\cC$ to be regular in contrast with the discrete variable's support.
Intuitively, the continuous variable $\cc$ often controls the extents/degrees of specific attributes (e.g., sizes, lighting, and angles) and takes values from connected spaces. For instance, ``lightning'' ranges from the lowest to the highest intensity continuously.
Condition~\ref{cond:discrete_component_conditions}-\ref{asmp:invertibility} ensures the generating process preserves latent variables' information~\citep{hyvarinen2016unsupervised,khemakhem2020variational,vonkugelgen2021selfsupervised,kong2022partial}.
Thanks to the high dimensionality, images often have adequate capacity to meet this condition. 
For instance, the image of a dog contains a detailed description of the dog’s breed, shape, color, lighting intensity, and angles, all of which are decodable from the image.
Condition~\ref{cond:discrete_component_conditions}-\ref{asmp:subset_influences} ensures that each latent component should exhibit sufficiently distinguishable influences on the observed variable $\xx$.
Practically, this condition indicates that the lowest-level concepts influence diverse parts of the image. 
These concepts are often atomic, such as a dog’s ear, eyes, or even finer, which often don’t overlap.
This condition is adopted in prior work~\citep{kivva2021learning,kivva2022identifiability} and related to the notation of sparsity.
Along this line, prior work~\citep{arora2012learning,arora2012practical,moran2021identifiable} assumes pure observed children for each discrete variable, which is strictly stronger. 
Recent work~\citep{zheng2022identifiability} assumes each latent variable is connected to a unique set of observed variables. 
This condition implies Condition~\ref{cond:discrete_component_conditions}-\ref{asmp:subset_influences} because if $z_{0}$’s children form a subset of $z_{1}$’s children, then one cannot find a subset of observed variables whose parent is $z_{0}$ alone.

\begin{restatable}[Discrete Component Identification]{theorem}{discretetheorem}  \label{thm:discrete_identification}
    Under the generating process in Equation~\ref{eq:discrete_generating} and Condition~\ref{cond:discrete_component_conditions}-\ref{asmp:invertibility}, the estimated discrete variable $\hat{\dd}$ and the true discrete variable $ \dd $ are equivalent up to an invertible function, i.e., $\hat{\dd} = h(\dd)$ with $h(\cdot)$ invertible.
    Moreover, if Condition~\ref{cond:basic_model} and Condition~\ref{cond:discrete_component_conditions}-\ref{asmp:subset_influences} further hold, we attain component-wise identifiability (Definition~\ref{def:permutation}) and the bipartite graph $\Gamma$ up to permutation of component indices.
\end{restatable}

\paragraph{Proof sketch.}
Intuitively, each state of the discrete subspace $\dd$ indexes a manifold $ g(\dd, \cdot): \cc \mapsto \xx $ that maps the continuous subspace $\cc$ to the observed variable $ \xx $.
These manifolds do not intersect in the observed variable space $\cX$ regardless of however close they may be to each other, thanks to the invertibility of the generating function $g$ (Condition~\ref{cond:discrete_component_conditions}-\ref{asmp:invertibility}).
This leaves a sufficient footprint in $\xx$ for us to uniquely identify the manifold it resides in, giving rise to the identifiability of $\dd$.
This reveals the discrete state of each realization of $\xx$ and equivalently the joint distribution $p(\tilde{d}, \xx)$ where we merge all components in $\dd$ into a discrete variable $\tilde{d}$.
Identifying this joint distribution enables the application of tensor decomposition techniques~\citep{kivva2021learning} to disentangle the global state $\hat{d}$ into individual discrete components $d_{i}$ and the causal graph $\Gamma$, under Condition~\ref{cond:basic_model} and Condition~\ref{cond:discrete_component_conditions}-\ref{asmp:subset_influences}.


\subsection{Hierarchical Model Identification} \label{subsec:hierarchy_identification}

We show that we can identify the underlying hierarchical causal structure $\cG$ that explains the dependence among low-level discrete components $d_{i}$ that we identify in Theorem~\ref{thm:discrete_identification}.

\begin{wrapfigure}{r}{10cm}
    \centering
    \setlength{\abovecaptionskip}{2pt}
    \setlength{\belowcaptionskip}{-4pt}
    \vspace{-0.55cm}
    \begin{subfigure}{0.2\textwidth}
        \centering
       \begin{tikzpicture}[scale=.38, line width=0.6pt, inner sep=0.6mm, shorten >=.1pt, shorten <=.1pt]
      \tikzset{
        znode/.style={text=brown},
        dnode/.style={text=blue},
        xnode/.style={text=black},
        every node/.style={align=center},
        edge from parent/.style={draw,->}
      }
      
      \node[znode] (z1) at (0,1) {$z_1$};
      \node[znode] (z2) at (-1.5,-1) {$z_2$};
      \node[znode] (z3) at (1.5,-1) {$z_3$};
      
      \node[znode] (d1) at (-3,-3) {$z_4$};
      \node[znode] (d2) at (-1,-3) {$z_5$};
      \node[znode] (d3) at (1,-3) {$z_6$};
      \node[znode] (d4) at (3,-3) {$z_7$}; 
    
      \node[dnode] (x1) at (-3.5,-5) {$d_1$};
      \node[dnode] (x2) at (-2.5,-5) {$d_2$};
      \node[dnode] (x3) at (-1.5,-5) {$d_3$};
      \node[dnode] (x4) at (-0.5,-5) {$d_4$};
      \node[dnode] (x5) at (0.5,-5) {$d_5$};
      \node[dnode] (x6) at (1.5,-5) {$d_6$};
      \node[dnode] (x7) at (2.5,-5) {$d_7$};
      \node[dnode] (x8) at (3.5,-5) {$d_8$};
      
        \draw[-latex] (z1) -- (z2);
      \draw[-latex] (z1) -- (z3);
      
      \draw[-latex] (z2) -- (d1);
      \draw[-latex] (z2) -- (d2);

    \draw[-latex] (z3) -- (d3);
      \draw[-latex] (z3) -- (d4);

      \draw[-latex] (d1) -- (x1);
      \draw[-latex] (d1) -- (x2);
      
      \draw[-latex] (d2) -- (x3);
      \draw[-latex] (d2) -- (x4);
    
      \draw[-latex] (d3) -- (x5);
      \draw[-latex] (d3) -- (x6);
    
      \draw[-latex] (d4) -- (x7);
      \draw[-latex] (d4) -- (x8);

    \end{tikzpicture}
        \caption{\small Trees.}
        \label{subfig:trees}
    \end{subfigure}
    \hfill
    \begin{subfigure}[b]{0.2\textwidth}
    \centering
    \begin{tikzpicture}[scale=.38, line width=0.6pt, inner sep=0.6mm, shorten >=.1pt, shorten <=.1pt]
      \tikzset{
        znode/.style={text=brown},
        dnode/.style={text=blue},
        xnode/.style={text=black},
        every node/.style={align=center},
        edge from parent/.style={draw,->}
      }
      
      \node[znode] (z1) at (0,1) {$z_1$};
      \node[znode] (z2) at (-1.5,-1) {$z_2$};
      \node[znode] (z3) at (1.5,-1) {$z_3$};
      
      \node[znode] (d1) at (-3,-3) {$z_4$};
      \node[znode] (d2) at (-1,-3) {$z_5$};
      \node[znode] (d3) at (1,-3) {$z_6$};
      \node[znode] (d4) at (3,-3) {$z_7$}; 
    
      \node[dnode] (x1) at (-3.5,-5) {$d_1$};
      \node[dnode] (x2) at (-2.5,-5) {$d_2$};
      \node[dnode] (x3) at (-1.5,-5) {$d_3$};
      \node[dnode] (x4) at (-0.5,-5) {$d_4$};
      \node[dnode] (x5) at (0.5,-5) {$d_5$};
      \node[dnode] (x6) at (1.5,-5) {$d_6$};
      \node[dnode] (x7) at (2.5,-5) {$d_7$};
      \node[dnode] (x8) at (3.5,-5) {$d_8$};
      
        \draw[-latex] (z1) -- (z2);
      \draw[-latex] (z1) -- (z3);
      
      \draw[-latex] (z2) -- (d1);
      \draw[-latex] (z2) -- (d2);
      \draw[-latex] (z2) -- (d3);

    \draw[-latex] (z3) -- (d3);
      \draw[-latex] (z3) -- (d4);
      \draw[-latex] (z3) -- (d2);

      \draw[-latex] (d1) -- (x1);
      
      \draw[-latex] (d2) -- (x2);
      
      \draw[-latex] (d2) -- (x3);
    
      \draw[-latex] (d3) -- (x4);
    
      \draw[-latex] (d2) -- (x5);
      \draw[-latex] (d4) -- (x5);
    
      \draw[-latex] (d3) -- (x6);
      \draw[-latex] (d4) -- (x6);
    
      \draw[-latex] (d1) -- (x7);
      \draw[-latex] (d4) -- (x7);

      \draw[-latex] (d4) -- (x8);
    
    \end{tikzpicture}
    \caption{Multi-level DAGs.}
    \label{subfig:multileveldags}
    \end{subfigure}
    \hfill
    \begin{subfigure}[b]{0.2\textwidth}
    \centering
    \begin{tikzpicture}[scale=.38, line width=0.6pt, inner sep=0.6mm, shorten >=.1pt, shorten <=.1pt]
      \tikzset{
        znode/.style={text=brown},
        dnode/.style={text=blue},
        xnode/.style={text=black},
        every node/.style={align=center},
        edge from parent/.style={draw,->}
      }
      
      \node[znode] (z1) at (0,1) {$z_1$};
      \node[znode] (z2) at (-1.5,-1) {$z_2$};
      \node[dnode] (z3) at (1.5,-1) {$d_1$};
      
      \node[dnode] (d1) at (-3,-3) {$d_2$};
      \node[znode] (d2) at (-1,-3) {$z_3$};
      \node[znode] (d3) at (1,-3) {$z_4$};
      \node[dnode] (d4) at (3,-3) {$d_3$}; 
    
      \node[dnode] (x1) at (-3.5,-5) {$d_4$};
      \node[dnode] (x2) at (-2.5,-5) {$d_5$};
      \node[dnode] (x3) at (-1.5,-5) {$d_6$};
      \node[dnode] (x4) at (-0.5,-5) {$d_7$};
      \node[dnode] (x5) at (0.5,-5) {$d_8$};
      \node[dnode] (x6) at (1.5,-5) {$d_9$};
      
        \draw[-latex] (z1) -- (z2);
      \draw[-latex] (z1) -- (z3);
      
      \draw[-latex] (z2) -- (d1);
      \draw[-latex] (z2) -- (d2);
      \draw[-latex] (z2) -- (d3);

    \draw[-latex] (z3) -- (d3);
      \draw[-latex] (z3) -- (d4);
      \draw[-latex] (z3) -- (d2);

      \draw[-latex] (d1) -- (x1);
      \draw[-latex] (d1) -- (x2);

      \draw[-latex] (d2) -- (x2);
      
      \draw[-latex] (d2) -- (x3);

      \draw[-latex] (d1) -- (x4);
      \draw[-latex] (z3) -- (x4);

      \draw[-latex] (d3) -- (x5);
    
      \draw[-latex] (d2) -- (x5);
    
      \draw[-latex] (d3) -- (x6);


    
    \end{tikzpicture}
    \caption{Ours.}
    \label{subfig:ours}
    \end{subfigure}
    \vspace{0.15cm}
    \caption{\footnotesize
        \textbf{Graphical comparison.}
        Tree Structures permit one undirected path between any two variables. Multi-level DAGs require partitioning variables into levels with edges only between adjacent levels. Our conditions allow multiple paths between variables across levels and include non-leaf observed variables.
    }
    \label{fig:comparison_graphs}
    \vspace{-0.2cm}
\end{wrapfigure}
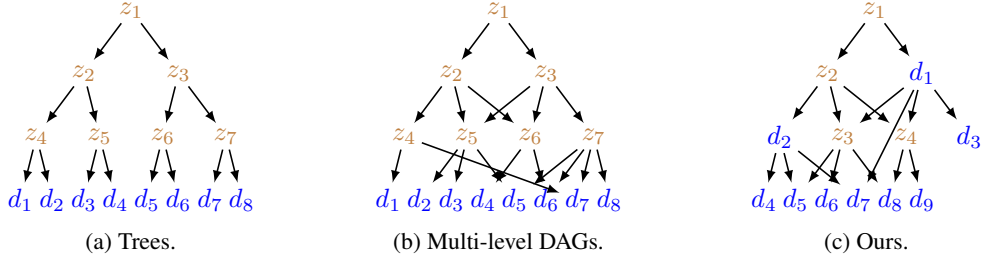
\paragraph{Remarks on the problem.}
Benefiting from the identified discrete components in Theorem~\ref{thm:discrete_identification}, we employ $\dd$ as observed variables to identify the discrete latent hierarchical model $\cG$. 
Although discrete latent hierarchical models have been under investigation for an extensive period, existing results mostly assume relatively strong graphical conditions -- the causal structures are either trees~\citep{zhang2004hierarchical,Pearl88,choi2011learning} or multi-level DAGs~\citep{gu2023bayesian,gu2023blessing}, which can be restrictive in capturing the complex interactions among latent variables among different hierarchical levels.
Separately, recent work~\citep{huang2022latent,dong2023versatile} has exhibited more flexible graphical conditions for linear, continuous latent hierarchical models.
For instance, prior work~\citep{dong2023versatile} allows for multiple directed paths of disparate edge numbers within 
a variable pair and potential non-leaf observed variables.  
Unfortunately, their techniques hinge on linearity and cannot directly apply to discrete models of high nonlinearity.

\paragraph{High-level description of our approach.}
The central machinery in prior work~\citep{huang2022latent,dong2023versatile} is Theorem~\ref{thm:continuous_rank_graph}~\citep{Sullivant_2010}, which builds a connection between easily computable statistical quantities (i.e., sub-covariance matrix ranks) and local latent graph information.
\citet{dong2023versatile} utilize a graph search algorithm to piece together these local latent graph structures to identify the entire hierarchical model. 
Ideally, if we can access these local latent structures in the discrete model, we can apply the same graph search procedure and theorems to identify the discrete model.
Nevertheless, Theorem~\ref{thm:continuous_rank_graph} relies on linearity (i.e., each causal edge represents a linear function), which doesn't hold in the discrete case.
We show that interestingly, Theorem~\ref{thm:continuous_rank_graph} can find a counterpart in the discrete case (Theorem~\ref{thm:rank_graph}), despite the absence of linearity.
Since given the graphical information from Theorem~\ref{thm:continuous_rank_graph}, the theory in \citet{dong2023versatile} is independent of statistical properties, we can utilize flexible conditions and algorithm therein by obtaining the same graphical information with Theorem~\ref{thm:rank_graph}.

To present Theorem~\ref{thm:rank_graph}, we introduce non-negative rank $\text{rank}_{+} (\cdot)$~\citep{cohen1993nonnegative} (Definition~\ref{def:nonnegtive_rank}), and t-separation~\citep{spirtes2001causation} (Definition~\ref{def:t_sep}) as follows.
\begin{restatable}[Non-negative Rank]{definition}{nonnegativeranks} \label{def:nonnegtive_rank}
    The non-negative rank of non-negative $A \in \R_{+}^{m\times n} $ is equal to the smallest $p$ for which there exist $\mB \in \R^{m\times p}_{+}$ and $\mC \in \R_{+}^{p\times n}$ such that $\mA = \mB \mC$. 
\end{restatable}
 
\begin{restatable}[Treks]{definition}{treks} \label{def:treks}
    A trek $T_{i,j}$ in a DAG from vertex $i$ to $j$ consists of a directed path $P_{ki}$ from $k$ to $i$ and a direct path $P_{kj}$ from $k$ to $j$, where we refer to $ P_{ki} $ as the $i$ side and $P_{kj}$ as the $ j $ side.
\end{restatable}

\begin{restatable}[T-separation]{definition}{tsep} \label{def:t_sep}
    Let $\mA$, $\mB$, $\mC_{\mA}$, and $ \mC_{\mB} $ be subsets (not necessarily disjoint) of vertices in a DAG. Then $ (\mC_{\mA}, \mC_{\mB}) $ t-separates $\mA$ and $\mB$ if every trek from $\mA$ to $\mB$ passes through either a vertex in $ \mC_{\mA} $ on the $\mA$ side of the trek or a vertex $ \mC_{\mB} $ on the $\mB$ side of the trek.
\end{restatable}

Intuitively, a trek is a path containing at most one fork structure and no collider structures.
It is known that one can formulate d-separation as a special form of t-separation (see Theorem~\ref{thm:d_t_equivalence}).
Thus, t-separation is at least as informative as d-separation.
As detailed in \citet{dong2023versatile}, t-separation can provide more information when latent variables are involved, benefiting from Theorem~\ref{thm:continuous_rank_graph}~\citep{Sullivant_2010}.

\begin{restatable}[Implication of Rank Information on Latent Discrete Graphs]{theorem}{rankgraph} \label{thm:rank_graph}
    Given two sets of variables $\set{A}$ and $\set{B}$ from a non-degenerate, faithful (Condition~\ref{cond:basic_model}-\ref{asmp:nondegeneracy}, Condition~\ref{cond:hierarchy}-\ref{asmp:faithfulness}) discrete model $\graph$, it follows that $\text{rank}_{+}(\mP_{\set{A}, {\set{B}}}) = \min \{ \abs{ \text{Supp} (\mL) }: \text{a partition }(\mL_{1}, \mL_{2}) \text{ t-separates}~\set{A}~\text{and}~\set{B}~\text{in}~\graph\}$.
\end{restatable}

\paragraph{Example.}
Suppose every variable in Figure~\ref{fig:comparison_graphs}(a) is binary, then for $\mA = \{ d_{1}, d_{2}, d_{3} \}$, $\mB = \{ d_{3}, \dots, d_{8} \}$, $ \text{rank}_{+} (\mP_{\mA, \mB}) = 4 $ since $\mA$ and $\mB$ are t-separated by $\{ d_{3}, z_{4}\}$ with $4$ states.

\paragraph{Discussion.}
Parallel to Theorem~\ref{thm:continuous_rank_graph}~\citep{Sullivant_2010} for linear models, Theorem~\ref{thm:rank_graph} acts as an oracle to reveal the minimal t-separation set's cardinality between any two variable sets in \textit{discrete} models beyond linearity.
This enables us to infer the latent graph structure from only the observed variables' statistical information.
To the best of our knowledge, Theorem~\ref{thm:rank_graph} is the first to establish this connection and can be of independent interest for learning latent discrete models in future work.
Although the computation of non-negative ranks can be expensive~\citep{cohen1993nonnegative}, existing work~\citep{anandkumar2012learning,mazaheri2023causal} demonstrates that regular rank tests are decent substitutes, we observe in our synthetic data experiments (Section~\ref{sec:synthetic}).

We present the identification conditions for discrete models as follows (Condition~\ref{cond:hierarchy}).

\begin{definition} [Atomic Covers] \label{def:discrete_cover}
  Let $\set{A} \subset \mV $ be a set of variables in $\graph$ with $|\text{Supp}(\set{A})| = k$, where $t$ of the $k$ states belong to observed variables $d_{i}$, and the remaining $k-t$ are from latent variables $z_{j}$. 
  $\set{A}$ is an atomic cover if $\set{A}$ contains a single observed variable, or if the following conditions hold:
  \begin{itemize}[leftmargin=20pt,itemsep=-3pt,topsep=-3pt]
      \item [(i)] There exists a set of atomic covers $\setset{C}$, with $ \abs{ \text{Supp}(\setset{C}) } \geq k+1-t$, such that $\cup_{\set{C} \in \setset{C}} \set{C}\subseteq \purechildren(\set{A})$ and $\forall \set{C_1}, \set{C_2} \in \setset{C}, \set{C_1}\cap\set{C_2}=\emptyset$.

     \item [(ii)] There exists a set of covers $\setset{N}$, with $\abs{ \text{Supp}(\setset{N}) }\geq k+1-t$,
     such that every element in $\cup_{\set{N} \in \setset{N}} \set{N}$ is a neighbour of $\set{V}$  and 
     $ (\cup_{\set{N} \in \setset{N}} \set{N}) \cap (\cup_{\set{C} \in \setset{C}} \set{C})=\emptyset$.

    \item [(iii)] There does not exist a partition of $\set{A}= \set{A_1} \cup \set{A_2}$ such that both $\set{A_1}$ and $\set{A_2}$ are atomic covers.
  \end{itemize}
\end{definition}

\paragraph{Example.}
In Figure~\ref{fig:comparison_graphs} (c), $\{ z_{2} \}$ is an atomic cover if its pure child $\{ d_{2} \}$ and its neighbors $ \{ z_{1}, z_{3}, z_{4} \} $ possess more than $ \text{Supp} (z_{2}) + 1 $ states separately.
Otherwise, $ \{ z_{2}, d_{1}\}$ can be an atomic cover if (some of) pure children $ \{ z_{3}, z_{4} \} $ and neighbors $\{ z_{1}, d_{2}, d_{3}\}$ possess $ \text{Supp} (z_{2}) + 1 $ states separately.

\begin{restatable}[Discrete Hierarchical Model Conditions]{condition}{hierarchicalconditions} \label{cond:hierarchy} {\ } 
    \begin{enumerate}[label=\roman*,leftmargin=2em]
        \item \label{asmp:faithfulness} [Faithfulness] All the conditional independence relations are entailed by the DAG.
        
        \item \label{asmp:basic_graph} [Basic Graphical Conditions] Each latent variable $z \in \mZ$ corresponds to a unique atomic cover in $\graph$ and no $z$ is involved in any triangle structure (i.e., three mutually adjacent variables).

        \item \label{asmp:vstructure} [Graphical Condition on Colliders] 
        In a latent graph $\graph$, if (i) there exists a set of variables $\set{C}$ such that every variable in $\set{C}$ is a collider of two atomic covers $\set{L_1}$, $\set{L_2}$, and denote by $\set{A}$ the minimal set of variables that d-separates $\set{L_1}$ from $\set{L_2}$, (ii) there is a latent variable in $\set{L_1}, \set{L_2}, \set{C}$ or $\set{A}$, then we must have $\abs{\text{Supp}(\set{C})} + \abs{\text{Supp}(\set{A})} \geq \abs{\text{Supp}(\set{L_1})}+\abs{\text{Supp}(\set{L_2})}$.
    \end{enumerate}
\end{restatable}

\paragraph{Discussion on the conditions.}
Condition~\ref{cond:hierarchy}-\ref{asmp:faithfulness} is known as the faithfulness condition widely adopted for causal discovery~\citep{spirtes2001causation,kivva2021learning,xie2020generalized,xie2022identification}, which attributes statistical independence to graph structures rather than unlikely coincidence~\citep{lemeire2013replacing,spirtes2001causation}.
In linear models, \citet{dong2023versatile} introduce atomic covers (Definition~\ref{def:discrete_cover}) to represent a group of indistinguishable variables.
In the discrete case, an atomic cover consists of indistinguishable latent states, which we merge into a single latent discrete variable (Condition~\ref{cond:basic_model}-\ref{asmp:twins}).
Intuitively, we treat each state as a separate variable and merge those belonging to the same atomic cover at the end of the identification procedure.
This handles discrete variables of arbitrary state numbers, in contrast with the binary or identical support assumptions~\citep{choi2011learning,gu2023bayesian}, which we use as an alternative condition in Theorem~\ref{thm:identical_support}.
Condition~\ref{cond:hierarchy}-\ref{asmp:basic_graph} requires each atomic cover to possess sufficiently many children and neighbors to preserve its influence while avoiding problematic triangle structures to ensure the uniqueness of its influence. 
In contrast, existing work~\citep{gu2023bayesian} assumes at least three pure children for each latent variable, amounting to six times more states.
Condition~\ref{cond:hierarchy}-\ref{asmp:vstructure} ensures adequate side information (large $|\mA|$) to discover latent colliders $ \mC $, admitting graphs more general than tree structures~\citep{Pearl88,zhang2004hierarchical,choi2011learning} (i.e., no colliders).
Overall, our model encompasses a rich class of latent structures more complex than tree structures and multi-level DAGs~\citep{gu2023bayesian} (Figure~\ref{fig:comparison_graphs}).


Following \citet{dong2023versatile}, we introduce the minimal-graph operator $\cO_{\min}$ (Definition~\ref{def:minimal_operator} and Figure~\ref{fig:minimal_operator}), which merges certain redundancy structures that rank information cannot distinguish.

\begin{restatable}[Minimal-graph Operator~\citep{huang2022latent,dong2023versatile}]{definition}{minimalgraph} \label{def:minimal_operator}
We can merge atomic covers $\mL$ into $\mP$ in $\cG$ if (i) $\mL$ is a pure child of $\mP$, (ii) all elements of $L$ and $P$ are latent and $|\text{Supp}(\mL)| = |\text{Supp}(\mP)|$, and (iii) the pure children of $\mL$ form a single atomic cover, or the siblings of $\mL$ form a single atomic cover. We denote such an operator as the minimal-graph operator $\cO_{\min}(\cG)$. 
\end{restatable}

\begin{restatable}[Discrete Hierarchical Identification]{theorem}{hierarchyidentification} \label{thm:hierarchical_model_identification}
    Suppose the causal model $\cG$ satisfies Condition~\ref{cond:basic_model} and Condition~\ref{cond:hierarchy}
    We can identify $\cG$ up to the Markov equivalence class of $\cO_{\min}(\cG)$.
\end{restatable}

\paragraph{Proof sketch.}
As discussed above, Theorem~\ref{thm:rank_graph} gives a graph structure oracle equivalent to Theorem~\ref{thm:continuous_rank_graph}, which we leverage to prove Theorem~\ref{thm:hierarchical_model_identification}.
Besides the rank test, the major distinction between Theorem~\ref{thm:continuous_rank_graph} and Theorem~\ref{thm:rank_graph} is that the former returns the number of variables in the minimal t-separation set whereas the latter returns the number of states.
Applying the search algorithm from \citet{dong2023versatile} alongside our rank test from Theorem~\ref{thm:rank_graph} to a discrete model $\cG$ results in a graph $\tilde{\cG}$.
In $\tilde{\cG}$, each latent variable $z$ in $\cG$ is split into a set of variables $\tilde{z}^{(1)}, \dots, \tilde{z}^{(|\text{Supp}(z)|)}$ as an atomic cover, with the set size equal to the state number of $z$. 
We can then reconstruct the original graph $\cG$ from $\tilde{\cG}$ by merging these atomic covers into discrete variables.
We present our algorithm in Algorithm~\ref{alg:all} and highlight the differences from that in \citet{dong2023versatile}. 

Our techniques can also utilize the identical support condition (e.g., binary latent variables)~\citep{gu2023bayesian,choi2011learning} for identification under slightly different conditions. We present the results in Theorem~\ref{thm:identical_support}.

%% file: sections/synthetic_data_experiments.tex
\section{Synthetic Data Experiments} \label{sec:synthetic}

\begin{table}[t]
\centering 
\caption{
    \footnotesize
\textbf{F1 scores for our method and the baseline \citet{kivva2021learning} }. Figure~\ref{fig:synthetic_with_x} exhibits the graphs. 
}
\vspace{0.1cm}
\resizebox{\columnwidth}{!}{%
\begin{tabular}{|c|c|c|c|c|c|c|c|c|c|}
\hline
         & Graph 1 & 
           Graph 2 & 
           Graph 3 & 
           Graph 4 & 
           Graph 5 & 
           Graph 6 & 
           Graph 7 & 
           Graph 8 & 
           Graph 9 \\ \hline
Baseline  & 0.67 $\pm$ 0.0 & 0.69 $\pm$ 0.1 & 0.67 $\pm$ 0.0 & 0.67 $\pm$ 0.2 & 0.63 $\pm$ 0.0 & 0.65 $\pm$ 0.0 & 0.67 $\pm$ 0.0 & 0.65 $\pm$ 0.0 & 0.63 $\pm$ 0.0 \\ \hline
Ours       & 0.94 $\pm$ 0.1   & 0.98 $\pm$ 0.1 & 0.94 $\pm$ 0.0 
& 0.98 $\pm$ 0.2 & 0.94 $\pm$ 0.1 & 0.93 $\pm$ 0.0 & 0.93 $\pm$ 0.1 &
0.96 $\pm$ 0.0
 & 0.93 $\pm$ 0.1\\ \hline
\end{tabular}
}
\vspace{-0.5cm}
\label{tab:synthetic_with_x}
\end{table}

\begin{table}[t]
\centering
\caption{
    \footnotesize
    \textbf{F1 scores for our method and the baseline~\citet{dong2023versatile}}. Figure~\ref{fig:synthetic_without_x} exhibits the graphs.}
\vspace{0.1cm}
\resizebox{0.8\columnwidth}{!}{%
\begin{tabular}{|c|c|c|c|c|c|c|c|}
\hline
           & Graph 1 & 
           Graph 2 & 
           Graph 3 & 
           Graph 4 & 
           Graph 5 & 
           Graph 6 & 
           Graph 7     \\ \hline
Baseline & 0.24 $\pm$ 0.3 & 0.48 $\pm$ 0.0 & 0.33 $\pm$ 0.2 & 0.63 $\pm$ 0.1 & 0.0 $\pm$ 0.0 & 0.55 $\pm$ 0.1 & 0.0 $\pm$ 0.0 \\ \hline
Ours       & 1.0 $\pm$ 0.0   & 1.0 $\pm$ 0.0   & 0.73 $\pm$ 0.0  & 0.73 $\pm$ 0.0 & 0.75 $\pm$ 0.0 & 0.95 $\pm$ 0.0 & 1.0 $\pm$ 0.0 \\ \hline
\end{tabular}
}
\label{tab:synthetic_without_x}
\vspace{-0.5cm}
\end{table}

\textbf{Experimental setup.} We generate the hierarchical model $\cG$ with randomly sampled parameters, and follow \citep{kivva2021learning} to build the generating process from $\dd$ to the observed variables $\mathbf{x}$ (i.e., graph $\Gamma$) by a Gaussian mixture model. The graphs are exhibited in Figure~\ref{fig:synthetic_with_x} and Figure~\ref{fig:synthetic_without_x} in Appendix~\ref{app:synthetic}. We follow \citet{dong2023versatile} to use F1 score for evaluation. More details can be found in Appendix~\ref{app:synthetic}.

\textbf{Results and discussion.}
We choose \citet{kivva2021learning} as our baseline because it is the only method we know designed to learn a non-parametric, discrete latent model from continuous observations.
We evaluate both methods on graphs in Figure~\ref{fig:synthetic_with_x}.
As shown in Table~\ref{tab:synthetic_with_x} and Table~\ref{tab:synthetic_without_x}, our method consistently achieves near-perfect scores, while the baseline, despite correctly identifying $\Gamma$ and directing edges among $\dd$ components, cannot handle higher-level latent variables.

To verify Theorem~\ref{thm:rank_graph}, we evaluate Algorithm~\ref{alg:all} and a baseline~\citep{dong2023versatile} on graphs satisfying the conditions on $\cG$ (i.e., purely discrete models in Figure~\ref{fig:synthetic_without_x}). Our method performs well on graphs that meet conditions of Theorem~\ref{thm:identical_support} and achieves decent scores on graphs that do not (Figure~\ref{fig:synthetic_without_x} (c) and (e)). The significant margins over the baseline validate Theorem~\ref{thm:rank_graph} and Theorem~\ref{thm:identical_support}.

%% file: sections/connection.tex
\section{Interpretations of Latent Diffusion}  \label{sec:connection}


In this section, we present a novel interpretation of latent diffusion (LD)~\citep{rombach2021highresolution} from the perspective of our hierarchical concept learning framework.
Concretely, the diffusion training objective can be viewed as performing denoising autoencoding at different noise levels~\citep{vincent2011connection,song2019generative}. 
Denoising autoencoders~\citep{vincent2008extracting,vincent2010stacked} and variants~\citep{pathak2016context,he2021masked} have shown the capability of extracting high-level, semantic representations as their encoder output. 
In the following, we adopt this perspective to interpret the diffusion model’s representation (i.e., the UNet encoder output) through our hierarchical model, which connects the noise level and the hierarchical level of the latent representation in our causal model.
For brevity, we refer to the diffusion model encoder's output as diffusion representation.

\paragraph{Discrete variables and representation embeddings.}
In practice, discrete variables are often modeled as embedding vectors from a finite dictionary (e.g., wording embeddings).
Therefore, although diffusion representation is not discrete, we can interpret it as an ensemble of embeddings of involved discrete variables.
\citet{park2023understanding} empirically demonstrates that one can indeed decompose the diffusion representation into a finite set of basis vectors that carry distinct semantic information, which can be viewed as the concept embedding vectors.


\textbf{Vector-quantization.}
Given an image $\xx$, LD first discretizes it with a vector-quantization generative adversarial network (VQ-GAN)~\citep{esser2021taming}:$\dd = f_{ \text{VQ} } (\xx).$
Through the lens of our framework, VQ-GAN represents the image with a rich but finite set of embeddings of bottom-level concepts $\dd$ and discards nuances in the continuous representation $\cc$, inverting the generation process in Equation~\ref{eq:discrete_generating}.

\begin{wrapfigure}{r}{7cm}
    \centering
    \vspace{-0.7cm}
    \includegraphics[width=7cm]{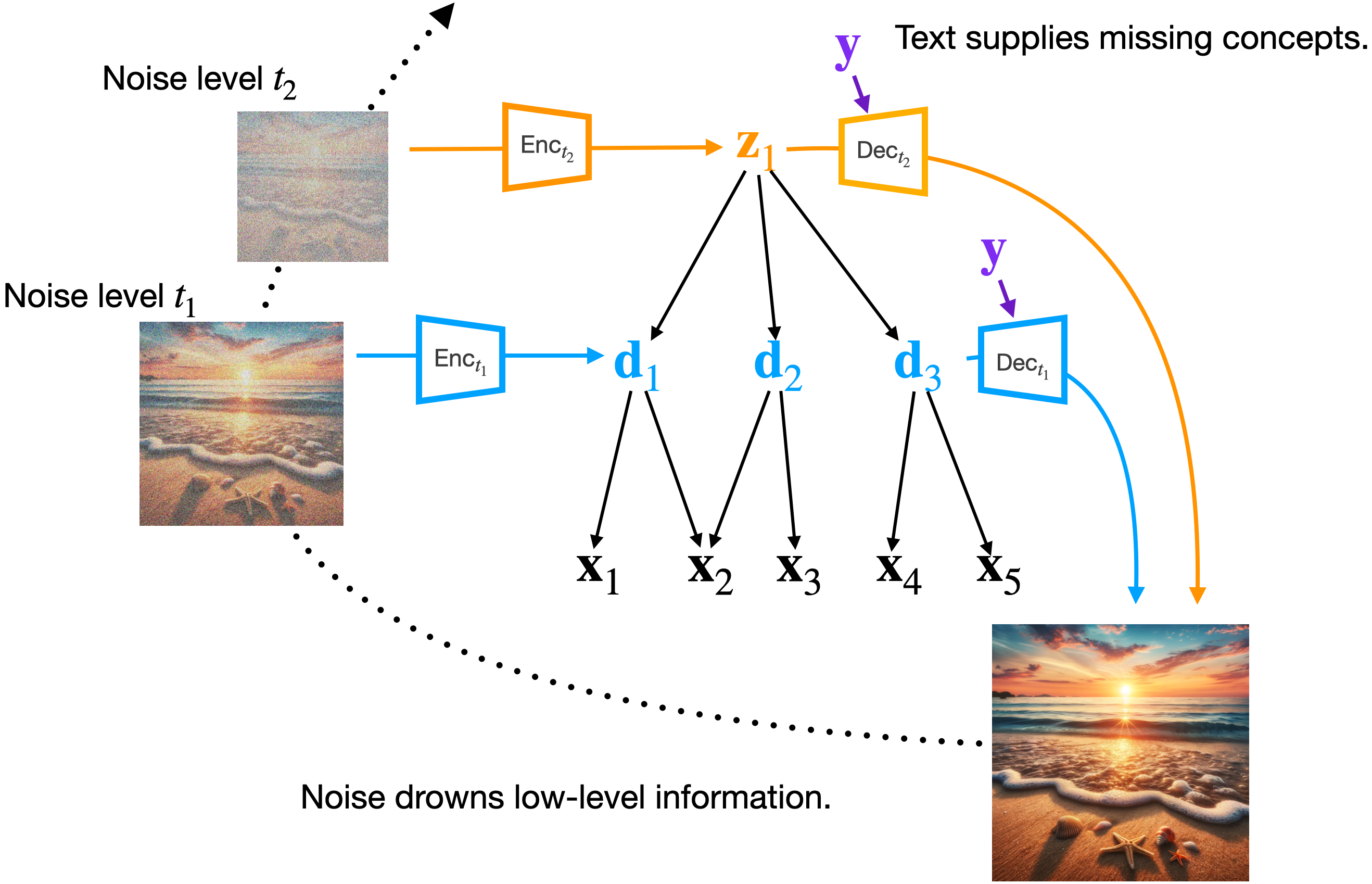}
    \vspace{-0.5cm}
    \caption{
        \small
        \textbf{Diffusion models estimate the latent hierarchical model.}
        Different noise levels correspond to different concept levels.
        To avoid cluttering, we leave out vector quantization.
    }
    \vspace{-0.3cm}
    \label{fig:thesis_figure}
\end{wrapfigure}
\paragraph{Denoising objectives.}
As discussed, diffusion training can be viewed as denoising the corrupted embedding $\tilde{\dd}$ to restore noiseless $\dd$~\citep{song2019generative,vincent2008extracting,vincent2010stacked,gu2022vector} for a designated denoising model $ f_{t} $ at noise level $t$:
\begin{align}
    \argmax_{f_{t}} \EEb{ \dd, \tilde{\dd}_{t} \sim \mathbb{Q}_{t}( \tilde{\dd}_{t} | \dd ) }{ \mathbb{P}_{ f_{t} }( \dd | \tilde{\dd}_{t}, \yy ) },
\end{align}
where $\yy$ denotes the text prompt.
Under this objective, the model is supposed to compress the noisy view $\tilde{\dd}_{t}$ to extract a clean, high-level representation, together with additional information from the text $\yy$, to reconstruct the original embedding $\dd$.
Formally, the denoising model $f_{t}:= f_{\text{dec}, t} \circ f_{\text{enc}, t} $ performs auto-encoding $ \zz_{\cS(t)} = f_{\text{enc}, t} ( \tilde{\dd}_{t} ) $ and $ \hat{\dd} = f_{ \text{dec}, t} ( \zz_{\cS(t)}, \yy ) $, where
we use $\cS(t)$ to indicate the dependence on the noise level $t$.
We can view the compressed representation $\zz_{\cS(t)}$ as a set of high-level latent variables in the hierarchical model: \textit{the encoder $ f_{\text{enc}, t} $ maps the noisy view $\tilde{\dd}_{t}$ to high-level latent variables $ \zz_{\cS_{t}} $ and the decoder $ f_{ \text{dec}, t } $ assimilates the text information $\yy$ and reconstructs the original view $\dd $}.
In practice, $f_{t}$ is implemented as a single model (e.g., UNets) paired with time embeddings.
We visualize this process in Figure~\ref{fig:thesis_figure}.


\paragraph{Noise levels and hierarchical levels.}
Intuitively, the noise level controls the amount of semantic information remaining in $\tilde{\dd}_{t}$.
For instance, a high noise level $t$ drowns the bulk of the low-level concepts in $\dd$, leaving only sparse high-level concepts in $\tilde{\dd}_{t}$.
In this case, the diffusion representation $ \zz_{\cS(t)} $ estimates a high concept level in the hierarchical model.
In Figure~\ref{fig:thesis_figure}, a high noise level may destroy low-level concepts, such as the sand texture and the waveforms, while preserving high-level concepts, such as the beach and the sunrise.
In Section~\ref{subsec:hierarchical_ordering}, we follow \citet{park2023understanding} to demonstrate diffusion representation's semantic levels under different noise levels.

\paragraph{Theory and practice.}
We connect LD training and estimating latent variables in the hierarchical model in an intuitive sense.
Our theory focuses on the fundamental conditions of the data-generating process and does not directly translate to guarantees for LD.
That said, our conditions naturally have implications on the algorithm design.
For instance, a sparsity constraint on the decoding model may facilitate the identification condition that variables influence each other sparsely (e.g., pure children in Condition~\ref{cond:hierarchy}).
In Section~\ref{subsec:concept_slider}, we show such constraints are beneficial for concept extraction.
We hope that our new perspective can provide more novel insights into advancing practical algorithms.

%% file: sections/experiments.tex
\section{Real-world Experiments} \label{sec:exp}

\subsection{Discovering Hierarchical Concept Structures from Diffusion Models} \label{subsec:hierarchical_ordering}

\begin{figure*}[t]
    \centering
    \includegraphics[width=.8\textwidth]{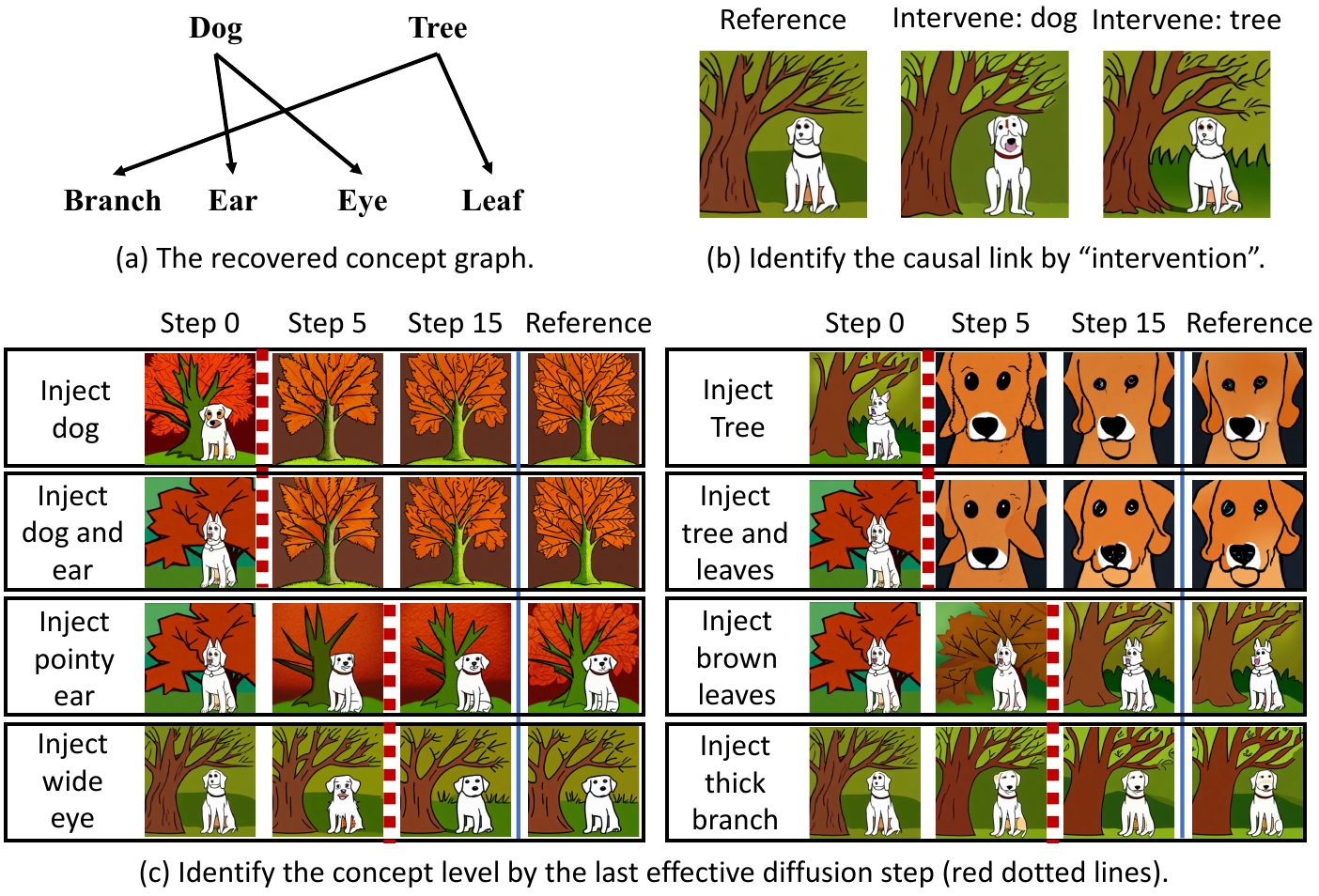}
    \caption{
    \small
    \textbf{Recovering concepts and their relationships from LD.} 
    \textbf{(a)} The final recovered concept graph among concepts ``dog'', ``tree'', ``eyes'', ``ears'', ``branch'', and ``leaf''.
    \textbf{(b)} Identifying causal links through ``interventions''. For example, we compare two prompts that vary in ``dog'': ``a dog with wide eyes and a wilting tree with short branches, in a cartoon style'' and ``a big dog with wide eyes and a wilting tree with short branches, in a cartoon style''. We observe significant changes in ``eyes'' but not in ``branch'', indicating a causal link between ``dog'' and ``eyes'' but not between ``dog'' and ``branch''.
    \textbf{(c)} Identifying concept levels by the last effective diffusion step. For example, we use the base prompt ``a tree with long branches, in a cartoon style'' and prepend ``dog'' at steps 0, 5, and 15. Only injecting ``dog'' at step 0 works. Similarly, injecting ``wide eyes'' works at both steps 0 and 5, indicating that ``dog'' is a higher-level concept than ``eyes''.
}
    \label{fig:recovered_graph}
    \vspace{-0.6cm}
\end{figure*}

In Figure~\ref{fig:recovered_graph}, we extract concepts and their relationships from LD through our hierarchical model interpretation.
Our recovery involves two stages: determining the concept level and identifying causal links. 
We add a textual concept, like ``dog'', into the prompt and identify the latest diffusion step that would render this concept properly. 
If ``dog'' appears in the image only when added at step 0 and ``eye'' appears when added from step 5, it indicates that ``dog'' is a higher-level concept than ``eyes''. 
After determining the levels of concepts, we intervene on a high-level concept and observe changes in low-level ones. No significant changes indicate no direct causal relationship.
We explore the relationships among the concepts ``dog'', ``tree'', ``eyes'', ``ears'', ``branch'', and ``leaf''. 
Figure~\ref{fig:recovered_graph} presents the final recovered graph and intermediate results. 
See Section~\ref{app:causal_order_exps} for more investigation.

\subsection{Diffusion Representation as Concept Embeddings}

We support our interpretations in Section~\ref{sec:connection} that diffusion representation can be viewed as concept embeddings, and it corresponds to high-level concepts for high noise levels.
Following \citet{park2023understanding}, we modify the diffusion representation along certain directions found unsupervisedly.
We can observe that this manipulation gives rise to semantic concept changes rather than entangled corruption Figure~\ref{fig:semantic_unet}.
Editing the latent representation at early steps corresponds to shifting global concepts.
\begin{wrapfigure}{r}{5cm}
    \centering
    \vspace{-0.5cm}
    \includegraphics[width=5cm]{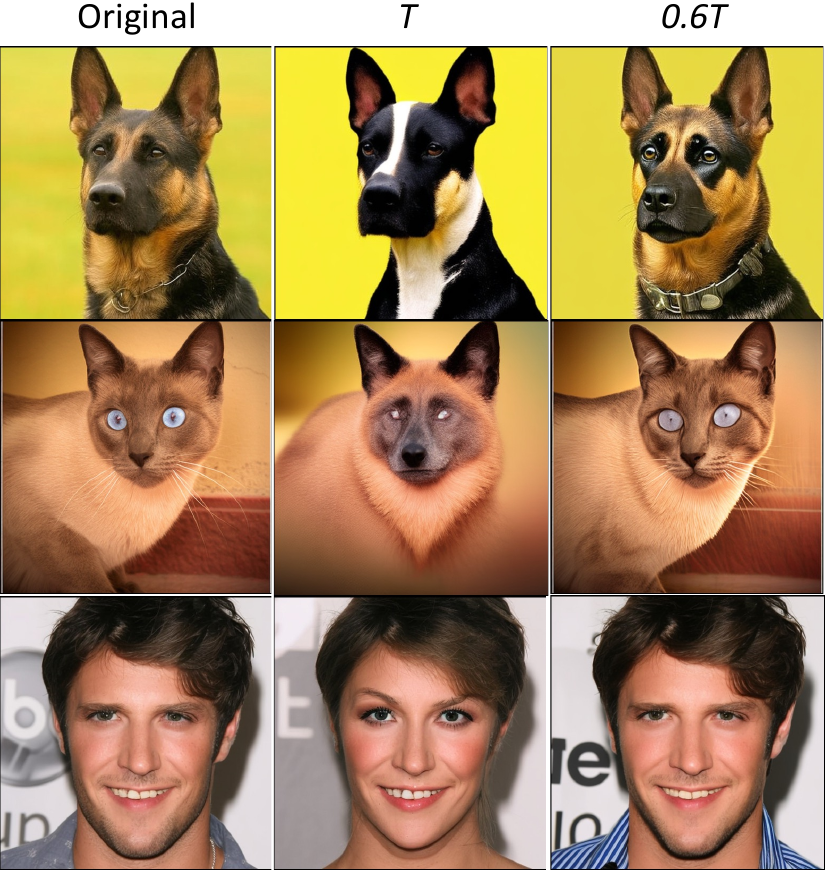}
    \caption{
        \small 
        \textbf{Semantic latent space.}
        We modify the diffusion model's representation (UNet encoder's output) along principal directions at steps $T$ and $0.6T$.
        Structure changes indicate the semantics of the representation and manipulation at the early time $T$ induces global shifts. 
        See more examples in Figure~\ref{fig:semantic_unet_more}.
    }
    \label{fig:semantic_unet}
    \vspace{-1.1cm}
\end{wrapfigure}
In Figure~\ref{fig:semantic_unet}, the latent representation in earlier steps (step $T$) determines breeds (the top row), species (the middle row), and gender (the bottom row).
In contrast, the latent representation in later steps (step $0.6T$) correlates with the dog collar, cat eyes, and shirt patterns. Implementation details and additional results are provided in Appendix~\ref{app:implementation}.

\subsection{Causal Sparsity for Concept Extraction} 
\label{subsec:concept_slider}

Recent work~\citep{gandikota2023sliders} shows that concepts can be extracted as low-rank parameter subspaces of LD models via LoRA~\citep{hu2021lora}. 
This low-rankness limits the complexity of text-induced changes, resembling sparse influences from latent concepts to their descendants. 
Our theory suggests that different levels of concepts may require varying sparsity levels to capture.
We present empirical evidence in Section~\ref{app:sparsity_results}.
Motivated by this, we design an adaptive sparsity selection mechanism for capturing concepts at different levels.
Inspired by \citet{ding2023sparse}, we implement a sparsity constraint on the LoRA dimensionality for the model to select the LoRA rank at each module automatically, benefiting concept extraction (see Appendix~\ref{app:sparsity_results}).

%% file: sections/conclusion.tex
\section{Conclusion} \label{sec:conclusion}
In this work, we cast the task of learning concepts as the identification problem of a discrete latent hierarchical model. 
Our theory provides conditions to guarantee the recoverability of discrete concepts. 
\textbf{Limitations:} Although our theoretical framework provides a lens for interpretation, our conditions do not directly guarantee diffusion's success, which would require nontrivial assumptions.
Also, Algorithm~\ref{alg:all} can be expensive for large graphs due to the dependency on the state count. 
We leave giving guarantees to diffusion models and efficient graph learning algorithms as future work.

%% file: sections/acknowledgement.tex
\paragraph{Acknowledgments.}
We thank the anonymous reviewers for their valuable insights and recommendations, which have greatly improved our work.
The work of L. Kong and Y. Chi is supported in part by NSF DMS-2134080. 
This material is based upon work supported by NSF Award No. 2229881, AI Institute for Societal Decision Making (AI-SDM), the National Institutes of Health (NIH) under Contract R01HL159805, and grants from Salesforce, Apple Inc., Quris AI, and Florin Court Capital.
This research has been graciously funded by the National Science Foundation (NSF) CNS2414087, NSF BCS2040381, NSF IIS2123952, NSF IIS1955532, NSF IIS2123952; NSF IIS2311990; the National Institutes of Health (NIH) R01GM140467; the National Geospatial Intelligence Agency (NGA) HM04762010002; the Semiconductor Research Corporation (SRC) AIHW award 2024AH3210; the National Institute of General Medical Sciences (NIGMS) R01GM140467; and the Defense Advanced Research Projects Agency (DARPA) ECOLE HR00112390063. 
Any opinions, findings, and conclusions or recommendations expressed in this publication are those of the author(s) and do not necessarily reflect the views of the National Science Foundation, the National Institutes of Health, the National Geospatial Intelligence Agency, the Semiconductor Research Corporation, the National Institute of General Medical Sciences, and the Defense Advanced Research Projects Agency.

%% file: sections/appendix.tex
\textit{\large Appendix for}\\ \ \\
      {\large \bf ``Learning Discrete Concepts in Latent Hierarchical
Causal Models''}\

\newcommand{\beginsupplement}{%
	\setcounter{table}{0}
	\renewcommand{\thetable}{A\arabic{table}}%
	
	\setcounter{figure}{0}
	\renewcommand{\thefigure}{A\arabic{figure}}%

	\setcounter{section}{0}
	\renewcommand{\thesection}{A\arabic{section}}%

    \setcounter{theorem}{0}
    \renewcommand{\thetheorem}{A\arabic{theorem}}%

    \setcounter{corollary}{0}
    \renewcommand{\thecorollary}{A\arabic{corollary}}%

        \setcounter{lemma}{0}
    \renewcommand{\thelemma}{A\arabic{lemma}}%

}

\vspace{.5cm}
\beginsupplement

{\large Table of Contents}

\DoToC 

\clearpage

\input{sections/related_work_full}

\section{Proof for Theorem~\ref{thm:discrete_identification}} \label{app:discrete_proof}

\discreteconditions*
\discretetheorem*

\begin{proof}[Proof of Theorem~\ref{thm:discrete_identification} Part 1]

    The estimate $ \hat{\dd} $ and the true variable $ \dd $ are related through the map $ [\hat{\dd}, \hat{\cc}] = \hat{g}^{-1} \circ g (\dd, \cc) $.
    In the following, we show that the induced relation between $\dd$ and $\hat{\dd}$ is invertible under Condition~\ref{cond:discrete_component_conditions}-\ref{asmp:invertibility}.
    The estimated generating process respects the conditions on the true generating process.

    We denote that support of the estimate $\hat{\dd}$ as $ \hat{\Omega}^{(d)} $.
    First, we show by contradiction that for each state $k \in \Omega^{(d)}$, $ k $ corresponds to at most one state $ \hat{k} \in \hat{\Omega}^{(d)} $ of the estimate $\hat{\dd}$.
    
    Suppose that $ k $ corresponds to two distinct states $ \hat{k}_{1} $ and $ \hat{k}_{2} $.
    That is, there exist $ \cc_{1}, \cc_{2} \in \cC $ and $ \hat{\cc}_{1}, \hat{\cc}_{2} \in \hat{\cC}$, such that $ \hat{g}^{-1} \circ g ( k, \cc_{1} ) =  [\hat{k}_{1}, \hat{\cc}_{1}] $ and $ \hat{g}^{-1} \circ g( k, \cc_{2} ) = [\hat{k}_{2}, \hat{\cc}_{2}] $.
    On one hand, As $ g( k, \cdot ) $ is a continuous function and $ \cC $ is connected, the image $ \cI ( k ):= g( k, \cC ) $ is a connected set.
    On the other hand, $ \hat{\cI} ( k_{1} ):= \hat{g}( \hat{k}_{1}, \hat{\cC}) $ and $ \hat{\cI} ( \hat{k}_{2} ):= \hat{g}( \hat{k}_{2}, \hat{\cC}) $ are two separate sets due to the invertibility and continuity of $ \hat{g} $ and the closed-ness of $ \hat{\cC}$. 
    To see this, invertibility implies that $ \hat{\cI} ( k_{1} )$ and $ \hat{\cI} ( k_{2} )$ are disjoint. 
    The fact that $\hat{g}$ is continuous over $\hat{\cc}$ and has a continuous inverse over $\hat{\cc}$ implies that $ \hat{\cI} ( k_{1} )$ and $ \hat{\cI} ( k_{2} )$ preserve the closed-ness of $\cC$.
    The space formed by two disjoint closed subspaces is disconnected.
    Since $ k $ corresponds to $ \hat{k}_{1} $ and $ \hat{k}_{2} $, it follows that $ \cI ( k ) = \hat{\cI}_{1} \cup \hat{\cI}_{2} $ where $ \hat{\cI}_{1} $ and $ \hat{\cI}_{2} $ are nonempty subsets of $  \hat{\cI}_{ \hat{k}_{1} } $ and $ \hat{\cI}_{ \hat{k}_{2} } $ respectively and inherit their separability.
    As $ \cI (k ) $ is a union of two nonempty separated sets, it is disconnected.
    This contradicts the connectedness of $ \cI (k) $.
    Therefore, for each state $k \in \Omega^{(d)}$, $ k $ corresponds to at most one state $ \hat{k} \in \hat{\Omega}^{(d)}$.

    Having established that each state of $\dd$ corresponds to at most one state of $\hat{\dd}$, we now show that states $\hat{k}_{1}$, $\hat{k}_{2}$ of $\hat{\dd}$ corresponding to distinct states $k_{1} \neq k_{2}$ of $\dd$ must also be distinct, i.e., $ \hat{k}_{1} \neq \hat{k}_{2} $ if $ k_{1} \neq k_{2} $.
    Suppose that $ \exists k_{1} \neq k_{2} $, such that the corresponding states $ \hat{k}_{1} = \hat{k}_{2} $.
    We denote $ \hat{k} := \hat{k}_{1} = \hat{k}_{2}$ and two arbitrary points $ \xx_{1}: = g( k_{1}, \cc_{1} ) $ and $ \xx_{2}:= g( k_{2}, \cc_{2} ) $ from modes $ k_{1} $ and $ k_{2} $ respectively.
    As the two estimated discrete states collapse at $\hat{k}$, it follows that
    \begin{align}
        \xx_{1} = g ( k_{1}, \cc_{1} ) &= \hat{g}(\hat{k}, \hat{\cc}_{1}) \\
        \xx_{2} = g ( k_{2}, \cc_{2} ) &= \hat{g}(\hat{k}, \hat{\cc}_{2}).
    \end{align}
    Since $\hat{g}( \hat{k}, \cdot )$ is continuous and $\hat{\cC} $ is a connected set, the image $ \hat{g}( \hat{k}, \hat{\cC} ) $ is path-connected.
    Thus, we can find a path $ f: [0, 1] \to \cX $ such that $ f(0) = \xx_{1} $ and $ f(1) = \xx_{2} $.
    Also, each point on the path $f$ has a positive probability density due to positive $ \hat{p}( \hat{\cc} ) $ and $ \Pb{ \hat{\dd} = \hat{k} } $.
    However, the two images $g ( k_{1}, \cC )$ and $ g ( k_{2}, \cC ) $ are disconnected due to the invertibility of $g$ and $ k_{1} \neq k_{2} $.
    On any path from $\xx_{1}$ to $ \xx_{2} $, there exists points $ \xx_{0} $ such that the density is strictly zero due to the discrete structure of $ \dd $.
    Thus, we have arrived at a contradiction.
    We have shown that if $ k_{1} \neq k_{2} $, the corresponding estimated states are distinct $ \hat{k}_{1} \neq \hat{k}_{2}$.

    Since for for each $k\in \Omega^{(d)}$, $k$ corresponds to at most one state $\hat{k}$ and distinct states $k_{1}$, $k_{2}$ give rise to distinct states $\hat{k}_{1}$, $\hat{k}_{2}$, we have proven that for each $k\in\Omega^{(d)}$, $ k $ corresponds to exactly one estimated state $ \hat{k} \in \hat{\Omega}^{(d)} $.
    
\end{proof}

\begin{definition}[Mixture Oracles] \label{def:oracle}
    Let $\xx$ be a set of observed variables and $\dd \in \Omega^{(d)}$ be a discrete latent variable. The mixture model is defined as $ \Pb{\xx} = \sum_{k\in \Omega^{(d)}} \Pb{\dd=k} \Pb{\xx | \dd=k} $.
    A mixture oracle $\text{MixOracle}(\xx)$ takes $\Pb{\xx}$ as input and returns the number of components $ \abs{\Omega^{(d)}} $, the weights $ \Pb{\dd = k} $ and the component $ \Pb{\xx | \dd=k} $ for $ k \in \Omega^{(d)}$. \footnote{
        We abuse the notation $\Pb{\cdot}$ to denote probability density functions for continuous variables and mass functions for discrete variables.
    }
\end{definition}

\begin{theorem}[\citet{kivva2021learning}] \label{thm:mixture_oracle_identification}
    Under Condition~\ref{cond:basic_model} and Condition~\ref{cond:discrete_component_conditions}-\ref{asmp:subset_influences}, on can reconstruct the bipartite graph $\Gamma$ between $\dd$ and $\xx$, and the joint distribution $\Pb{d_{1}=k_{1}, \dots, d_{n_{d}} = k_{n_{d}}}$ from $ \Pb{\xx} $ and $ \text{MixOracle}(\xx) $.
\end{theorem}


\begin{proof}[Proof of Theorem~\ref{thm:discrete_identification} Part 2]


\paragraph{Step 1}: Given the first result in Theorem~\ref{thm:discrete_identification}, we can identify the discrete state index $k$ for each realization of $\mathbf{x}$ (up to permutations). Since we can do this to all realizations of $\xx$ and we are given $ \Pb{\xx}  $, we can compute the cardinality of the discrete subspace $ | \Omega^{(d)} |  $, the marginal distribution of each latent state $ \Pb{ \dd = k } $, and the conditional distribution $ \Pb{ \xx | \dd = k } $ for $k \in \Omega^{(d)}$.

\paragraph{Step 2}: Step 1 shows the availability of the mixture oracle MixOracle (i.e., $ | \Omega |  $, $ \Pb{ \dd = k } $, and $\Pb{ \xx | \dd = k } $ ) as defined in Definition~\ref{def:oracle}.
Now, all conditions employed in Theorem~\ref{thm:mixture_oracle_identification} are ready, namely Condition~\ref{cond:basic_model}, Condition~\ref{cond:discrete_component_conditions}~\ref{asmp:subset_influences}, and MixOracle (the consequence of step 1).
The derivation in \citet{kivva2021learning} entails identifying a map from the discrete subspace state index $\dd=k$ where $ k \in \Omega^{(d)} $ to all discrete components’ state indices $ [d_{1}, \dots, d_{n_{d}} ] = [ k_{1}, \dots, k_{n_{d}} ]  $ where $ k_{i} \in \Omega^{(d)}_{i}  $ is the state index of the $i$-th component $ d_{i} $. Thus, we can utilize this map to identify the state index for each individual discrete variable $d_{i}$ from the global index $k$.


\paragraph{Step 3}: As stated in Step 2, all conditions in Theorem~\ref{thm:mixture_oracle_identification} hold in our problem. Since Theorem Theorem~\ref{thm:mixture_oracle_identification} additionally identifies the bipartite graph $\Gamma$ between $ \{ x_{1}, x_{2}, x_{3}, \dots \} $ and$ \{ d_{1}, d_{2}, d_{3}, \dots \}$, the same follows in our case.

\end{proof}

\section{Proof for Theorem~\ref{thm:hierarchical_model_identification}} \label{app:hierarchical_proof}

In this section, we present a proof for Theorem~\ref{thm:hierarchical_model_identification}.
Since all variables are discrete for this proof, for a set of variables $\mA$, we adopt the notation $ \mA = i $ to indicate the joint state of all variables in $\mA$.

As outlined in Section~\ref{sec:theory}, we will derive Theorem~\ref{thm:rank_graph} which serves as the bridge between the distributional information and the graphical information, equivalent to the role of Theorem~\ref{thm:continuous_rank_graph}~\citet{Sullivant_2010} in \citet{dong2023versatile,huang2022latent}.

To familiarize the reader with the context, we introduce Theorem~\ref{thm:continuous_rank_graph} and the involved graphical definitions treks~\ref{def:treks}, t-separation~\ref{def:t_sep}, and its connection between d-separation~\citep{pearl2009causality}.

\treks*

Intuitively, a trek is a path containing at most one fork structure and no collider structures.
Given this definition, a notion of t-separation is introduced~\citep{spirtes2001causation}, reminiscent of the classic d-separation. 


\tsep*

\begin{theorem}[Equivalence between d-separation and t-separation~\citep{di2009t}] \label{thm:d_t_equivalence}
    Suppose we have disjoint vertex sets $ \mA $, $ \mB $, and $\mC$ in a DAG.
    Set $\mC$ d-separates set $ \mA $ and set $\mB$ if and only if there exists a partition $ \mC := \mC_{\mA} \cup \mC_{\mB} $ such that $ (\mC_{\mA}, \mC_{\mB}) $ t-separates $ \mA \cup \mC $ and $ \mB \cup \mC $.
\end{theorem}

Theorem~\ref{thm:d_t_equivalence} shows that one can reformulate d-separation with a special form of t-separation.
Thus, t-separation is at least as informative as d-separation.
Further, as detailed in \citet{dong2023versatile}, t-separation can provide more information when latent variables are involved, benefiting from Theorem~\ref{thm:continuous_rank_graph}~\citep{Sullivant_2010}.

\begin{theorem} [Covariance Matrices and Graph Structures~\citep{Sullivant_2010}] \label{thm:continuous_rank_graph}
    Given two sets of variables $\set{A}$ and $\set{B}$ from a linear model with graph $\graph$, it follows that $\text{rank} (\Sigma_{\set{A}, {\set{B}}}) = \min \{ \abs{  \mL }: \mL~\text{t-separates}~\set{A}~\text{from}~\set{B}~\text{in}~\graph\}$, where $\Sigma_{\set{A}, {\set{B}}}$ denotes the generic covariance matrix between $\mA$ and $\mB$.
\end{theorem}
Theorem~\ref{thm:continuous_rank_graph} reveals that one can access local latent graph structures, i.e., the cardinality of the minimal separation set between two subsets of observed variables, through computable statistical quantities, e.g., covariance matrix ranks.
\citet{dong2023versatile} utilize these local latent graph structures, together with graphical conditions, to develop their identification theory for linear hierarchical models. 
Ideally, if we can access such local latent structures in the discrete hierarchical model, we can apply the same graph search procedure and theorems in \citet{dong2023versatile} to identify the discrete model.
Nevertheless, Theorem~\ref{thm:continuous_rank_graph} relies on the linearity of the causal model (i.e., each causal edge represents a linear function), which doesn't hold in the discrete case.
This motivates us to derive a counterpart of Theorem~\ref{thm:continuous_rank_graph} for discrete causal models.

To this end, we introduce a classic theorem (Theorem~\ref{thm:nonnegative_rank}) that connects the non-negative rank of a joint probability table with latent variable states.
\nonnegativeranks*

\begin{theorem}[Non-negative Rank and Probability Matrix Decomposition~\citep{cohen1993nonnegative}] \label{thm:nonnegative_rank}
    Let $\mP \in \R^{m \times n} $ be a bi-variate probability matrix. Then its non-negative rank $\text{rank}_{+} (\mP)$ is the smallest non-negative integer $p$ such that $\mP$ can be expressed as a convex combination of $p$ rank-one bi-variate probability matrices.
\end{theorem}

Given this machinery, we now derive Theorem~\ref{thm:rank_graph} which provides equivalent information in discrete models as Theorem~\ref{thm:continuous_rank_graph} in linear models.
\rankgraph*

\begin{proof}
    We express the joint distribution table $\mP_{\mA, \mB} $ as 
    \begin{align}\label{eq:conditional_ind_form}
        \mP (\mA = i, \mB = j) = \sum_{ r \in [R] } \mP(\mA=i | \mL = r) \mP( \mB=j | \mL = r) \mP( \mL = r ),
    \end{align}
    where $ R \in \N^{+} $ is the smallest possible value.
    This is always possible since we can assign $\mL$ as either $\mA$ or $\mB$ and obtain a trivial expression.
    
    We note that $ \mA \setminus \mL $, $ \mB \setminus \mL $, and $\mL$ are disjoint because if $\mA \cap \mB $ is nonempty, it must be a subset of $\mL$.
    Since the graph $\cG$ is non-degenerate (Condition~\ref{cond:basic_model}-\ref{asmp:nondegeneracy}) and faithful (Condition~\ref{cond:hierarchy}-\ref{asmp:faithfulness}), Equation~\ref{eq:conditional_ind_form} implies the graphical condition 
    that $ \mA \setminus \mL $ and $ \mB \setminus \mL $ are d-separate given $ \mL $.
    
    The equivalence relation in Theorem~\ref{thm:d_t_equivalence} implies that a partition of $ \mL $ t-separates $ \mA $ and $ \mB $.
    Thus, the minimal cardinality $R$ is equal to the smallest number of discrete states of $\mL$ that t-separates $\mA$ and $\mB$.
    Moreover, Theorem~\ref{thm:nonnegative_rank} implies that the minimal number of states is equal to the non-negative rank of $\mP_{\mA, \mB}$, i.e., $R = \text{rank}^{+}(\mP_{\mA, \mB}) $, which concludes our proof.
\end{proof}

With Theorem~\ref{thm:rank_graph} in hand, we leverage existing structural identification results on \textit{linear} hierarchical models (Theorem~\ref{thm:hierarchical_model_identification_linear_original}) to obtain the identification results of desire (Theorem~\ref{thm:hierarchical_model_identification}). 

We introduce formal definitions of linear models, pure children, and the minimal graph operator, which we refer to in the main text.

\begin{definition}[Linear Causal Models~\citep{dong2023versatile,huang2022latent}] \label{def:linear_models}
    A linear causal model is a DAG with variable set $\mV$ and an edge set $\mE$, where each causal variables $v$ is generated by its parents $ \text{Pa}(v) $ through a linear function:
    \begin{align}
        v_{i} := \sum_{v_{j} \in \text{Pa}(v)} a_{i,j} v_{j} + \epsilon_{i},
    \end{align}
    where $a_{i, j}$ is the causal strength and $\epsilon_{i}$ is the exogenous variable associated with $v_{i}$.
\end{definition}

\begin{definition}[Pure Children] \label{def:pch}
  A variable set $\set{Y}$ are pure children of variables $\set{X}$ in graph $\graph$, iff $\parents(\set{Y}) = \cup_{\node{Y_i} \in \set{Y}} \parents(\node{Y_i}) = \set{X}$ and $\set{X} \cap \set{Y}=\emptyset$. 
  We denote the pure children of $\set{X}$ in $\graph$ by $\purechildren(\set{X})$. 
\end{definition}
Basically, the definition dictates that variable $ \set{Y} $ has no other parents than $ \set{X} $.


\minimalgraph*
This operator merges certain structural redundancies not detectable from rank information~\citep{huang2022latent,dong2023versatile} (Lemma~\ref{lemma:rank_invariances}).
Please refer to Figure~\ref{fig:minimal_operator} for an example.

\begin{definition} [Atomic Covers (Linear Models)] \label{def:linear_cover}
  Let $\set{A}$ be a set of variables in $\graph$ with $|\set{A}| = k$, where $t$ of the $k$ variables are observed variables, and the remaining $k-t$ are latent variables. 
  $\set{A}$ is an atomic cover if $\set{A}$ contains a single observed variable, or if the following conditions hold:
  \begin{itemize}[leftmargin=20pt,itemsep=-3pt,topsep=-3pt]
      \item [(i)] There exists a set of atomic covers $\setset{C}$, with $ \abs{ \setset{C} } \geq k+1-t$, such that $\cup_{\set{C} \in \setset{C}} \set{C}\subseteq \purechildren(\set{V})$ and $\forall \set{C_1}, \set{C_2} \in \setset{C}, \set{C_1}\cap\set{C_2}=\emptyset$.

     \item [(ii)] There exists a set of covers $\setset{N}$, with $\abs{\setset{N} }\geq k+1-t$,
     such that every element in $\cup_{\set{N} \in \setset{N}} \set{N}$ is a neighbour of $\set{V}$  and 
     $ (\cup_{\set{N} \in \setset{N}} \set{N}) \cap (\cup_{\set{C} \in \setset{C}} \set{C})=\emptyset$.

    \item [(iii)] There does not exist a partition of $\set{A}= \set{A_1} \cup \set{A_2}$ such that both $\set{A_1}$ and $\set{A_2}$ are atomic covers.
  \end{itemize}
\end{definition}

\begin{theorem} [Linear Hierarchical Model Conditions] \label{cond:linear_hierarchy} {\ }
    \begin{enumerate}[label=\roman*,leftmargin=2em]
        \item \label{asmp:rank_faithfulness} [Rank Faithfulness]:
        All the rank constraints on the covariance matrices are entailed by the DAG.
        
        \item \label{asmp:basic_graph_linear} [Basic Graphical Conditions] For any $ \node{L} \in \set{V}$,  $\node{L}$ belongs to at least one atomic cover (Definition~\ref{def:linear_cover}) in the linear model $\graph$ (Definition~\ref{def:linear_models}) and no latent variable is involved in any triangle structure (i.e., three mutually adjacent variables).

        \item \label{asmp:vstructure_linear} [Graphical Condition on Colliders] 
        In a latent graph $\graph$, if (i) there exists a set of variables $\set{C}$ such that every variable in $\set{C}$ is a collider of two atomic covers $\set{L_1}$, $\set{L_2}$, and denote by $\set{A}$ the minimal set of variables that d-separates $\set{L_1}$ from $\set{L_2}$, (ii) there is a latent variable in $\set{L_1}, \set{L_2}, \set{C}$ or $\set{A}$, then we must have $\abs{\set{C}} + \abs{\set{A}} \geq \abs{\set{L_1}}+\abs{\set{L_2}}$.
    \end{enumerate}
\end{theorem}

\begin{definition}[Skeleton Operator~\cite{huang2022latent,dong2023versatile}] \label{def:skeleton_operator}
    Given an atomic cover $\mA$ in a graph $\cG$, for all $a \in \mA$, $a$ is latent, and all $c \in \text{PCh} (\mA)$, such that $a$ and $c$ are not adjacent, we can draw an edge from $a$ to $c$. We denote such an operator as skeleton operator $\cO_{s}(\cG)$.
\end{definition}
The skeleton operator introduces additional edges to fully connect atomic clusters~\citep{huang2022latent,dong2023versatile}, which are indistinguishable from the rank information (Lemma~\ref{lemma:rank_invariances}).
Please refer to Figure~\ref{fig:minimal_operator} for an example.

\begin{lemma}[Rank Invariance~\citet{huang2022latent}] \label{lemma:rank_invariances}
    The rank constraints are invariant with the minimal-graph operator and the skeleton operator; that is, $\cG$ and $\cO_{s}(\cO_{\min}(\cG))$ are rank equivalent.
\end{lemma}

\begin{theorem}[Linear Hierarchical Model Identification~\citep{dong2023versatile}] \label{thm:hierarchical_model_identification_linear_original}
    Suppose the $\cG$ is a linear latent causal model (Definition~\ref{def:linear_models}) that satisfies Condition~\ref{cond:linear_hierarchy}. 
    Then the hierarchical causal model $\cG$ is identifiable up to the Markov equivalent class of $\cO_{s}(\cO_{\min}(\cG))$.
\end{theorem}

We note that linear model conditions (Condition~\ref{cond:linear_hierarchy}) and discrete model conditions (Condition~\ref{cond:hierarchy}) differ mainly in the substitutes of variables in the linear models with states in the discrete models.
This originates from the local graph structures we can access, i.e., states in Theorem~\ref{thm:rank_graph} and variables in Theorem~\ref{thm:continuous_rank_graph}.
The skeleton operator $\cO_{\min}$ (Definition~\ref{def:skeleton_operator} is not necessary under Condition~\ref{cond:hierarchy} since each cover represents a discrete variable whose states must all be connected to its neighbors.

We now present Theorem~\ref{thm:hierarchical_model_identification} and its proof.
\hierarchicalconditions*
\hierarchyidentification*

\begin{proof}
    
     We observe that the linearity condition (Definition~\ref{def:linear_models}) in Theorem~\ref{thm:hierarchical_model_identification_linear_original} is only utilized to invoke Theorem~\ref{thm:continuous_rank_graph} to access the cardinality of the smallest t-separation set between any two sets of observed variables in the linear model.
     Through this, the graph identification results in Theorem~\ref{thm:hierarchical_model_identification_linear_original} are derived based on a graph search algorithm repeatedly querying partial graph structures under Condition~\ref{cond:linear_hierarchy}.

    For discrete models (Condition~\ref{cond:basic_model}), Theorem~\ref{thm:rank_graph} supplies partial graph structures equivalent to Theorem~\ref{thm:continuous_rank_graph}.
    The difference is that Theorem~\ref{thm:continuous_rank_graph} returns the number of variables in the smallest t-separation set while Theorem~\ref{thm:rank_graph} returns the number of states in the smallest t-separation set. 
    Thus, running Algorithm~\ref{alg:all} up to Step~\ref{step:convert_1} (i.e., the original search algorithm~\citet{dong2023versatile} with a different rank oracle in Theorem~\ref{thm:rank_graph} highlighted in blue) will return a graph with latent nodes representing discrete states.
    Algorithm~\ref{alg:all} is guaranteed to correctly discover all the atomic covers (Theorem~\ref{thm:hierarchical_model_identification_linear_original}) and each atomic cover corresponds to a latent discrete variable (Condition~\ref{cond:hierarchy}-\ref{asmp:basic_graph}). 
    Thus, we can obtain each true latent variable by merging all the latent nodes $\mA_{L}$ in each atomic cover $\mA$ into a discrete latent variable $z$ whose support cardinality $\abs{\text{Supp}(z)}$ equals to the number of latent nodes $ \abs{\mA_{L}} $.
    We highlight this procedure (Step~\ref{step:convert_1} in Algorithm~\ref{alg:all}).
    Moreover, as all latent nodes (i.e., latent states) in an atomic cover belong to one discrete variable, these latent nodes in adjacent atomic covers must be fully connected. Thus, we do not need the skeleton operator $\cO_{s}$ as for linear models (Theorem~\ref{thm:hierarchical_model_identification_linear_original}).
    This concludes our proof for Theorem~\ref{thm:hierarchical_model_identification}.

    
\end{proof}

Theorem~\ref{thm:identical_support} follows the same reasoning as in Theorem~\ref{thm:hierarchical_model_identification}, with the main difference in organizing latent nodes/states into latent discrete variables.
\begin{restatable}[Discrete Hierarchical Model Conditions for Identical Supports]{condition}{hierarchicalconditionsidenticalsupports} \label{cond:hierarchy_identical_supports} {\ } 
    \begin{enumerate}[label=\roman*,leftmargin=2em]
        \item \label{asmp:faithfulness_identical_supports} [Faithfulness]: All the conditional independence relations are entailed by the DAG.
        \item \label{asmp:basic_graph_identical_supports} [Basic Graphical Conditions]: Each latent variable $z \in \mZ$ belongs to at least one atomic cover in $\graph$ and no $z$ is involved in any triangle structure (i.e., three mutually adjacent variables).
        \item \label{asmp:v_structures_identical_supports} [Graphical Condition on Colliders]: In a latent graph $\graph$, if (i) there exists a set of variables $\set{C}$ such that every variable in $\set{C}$ is a collider of two atomic covers $\set{L_1}$, $\set{L_2}$, and denote by $\set{A}$ the minimal set of variables that d-separates $\set{L_1}$ from $\set{L_2}$, (ii) there is a latent variable in $\set{L_1}, \set{L_2}, \set{C}$ or $\set{A}$, then we must have $\abs{\text{Supp}(\set{C})} + \abs{\text{Supp}(\set{A})} \geq \abs{\text{Supp}(\set{L_1})}+\abs{\text{Supp}(\set{L_2})}$.
    \end{enumerate}
\end{restatable}

We introduce the skeleton operator $\cO_{s}$~\citep{huang2022latent,dong2023versatile} (Definition~\ref{def:skeleton_operator}) that include edges between adjacent covers indistinguishable to rank information.
\begin{restatable}[Discrete Hierarchical Identification on Identical Supports]{theorem}{identicalsupport}
\label{thm:identical_support}
Suppose the causal model $\cG$ satisfies Condition~\ref{cond:basic_model}-\ref{asmp:nondegeneracy}, Condition~\ref{cond:hierarchy_identical_supports}, and $\abs{\text{Supp}(z)} = K \ge 2$ for all $z \in \mZ$.
We can identify $\cG$ up to the Markov equivalence class of $\cO_{s}(\cO_{\min}(\cG))$.
\end{restatable}

\begin{proof}
        The bulk of the proof overlaps with the proof of Theorem~\ref{thm:hierarchical_model_identification}.
        Following the same reasoning of the proof of Theorem~\ref{thm:hierarchical_model_identification}, we can obtain a graph with latent nodes representing discrete states before Step~\ref{step:convert_1} and Step~\ref{step:convert_2} in Algorithm~\ref{alg:all}.        
        Under the identical support condition in Theorem~\ref{thm:identical_support}, we can directly group $K$ states in an atomic cover into a latent variable as in Algorithm~\ref{alg:all}-Step~\ref{step:convert_2}.
        Since the true latent variable cardinality is known to be identical, we don't need Condition~\ref{cond:basic_model}-\ref{asmp:twins},~\ref{asmp:maximality} to ensure the structure is well defined.
        Under Condition~\ref{cond:hierarchy_identical_supports}, each atomic cover may contain multiple discrete latent variables, depending on the cover size.
        It could be possible that one latent variable is not connected to all latent variables in an adjacent atomic cover, as in the linear model case.
        However, this difference cannot be detected from the rank information (Lemma~\ref{lemma:rank_invariances}).
        Thus, we need to retain the skeleton operator $\cO_{s}$ inherited from Theorem~\ref{thm:hierarchical_model_identification_linear_original}
        This concludes our proof for Theorem~\ref{thm:identical_support}. 
\end{proof}

\begin{figure}[t]
    \centering
    \setlength{\belowcaptionskip}{-2pt}
    \begin{subfigure}{0.31\textwidth}
        \centering
       \begin{tikzpicture}[scale=.58, line width=0.4pt, inner sep=0.6mm, shorten >=.1pt, shorten <=.1pt]
      \tikzset{
        znode/.style={text=brown},
        dnode/.style={text=blue},
        xnode/.style={text=black},
        every node/.style={align=center},
        edge from parent/.style={draw,->}
      }

        \node[znode] (L5) at (-0.5, 4) {{\footnotesize\,$z_5$\,}};
        \node[znode] (L4) at (-2, 2.5) {{\footnotesize\,$z_4$\,}};
        \node[dnode] (X8) at (-1, 2.5) {{\footnotesize\,$d_8$\,}};
        \node[dnode] (X9) at (0, 2.5) {{\footnotesize\,$d_9$\,}};
        \node[dnode] (X10) at (1, 2.5) {{\footnotesize\,$d_{10}$\,}};
        
        \node[znode] (L1) at (-3, 1) {{\footnotesize\,$z_1$\,}};
        \node[znode] (L2) at (-2, 1) {{\footnotesize\,$z_2$\,}};
        \node[znode] (L3) at (-1, 1) {{\footnotesize\,$z_3$\,}};
        
        \node[dnode] (X1) at (-4.5, -0.5) {{\footnotesize\,$d_1$\,}};
        \node[dnode] (X2) at (-3.7, -0.5) {{\footnotesize\,$d_2$\,}};
        \node[dnode] (X3) at (-2.9, -0.5) {{\footnotesize\,$d_3$\,}};
        \node[dnode] (X4) at (-2.1, -0.5) {{\footnotesize\,$d_4$\,}};
        \node[dnode] (X5) at (-1.3, -0.5) {{\footnotesize\,$d_5$\,}};
        \node[dnode] (X6) at (-0.5, -0.5) {{\footnotesize\,$d_6$\,}};
        \node[dnode] (X7) at (0.3, -0.5) {{\footnotesize\,$d_7$\,}};

	   \draw[-latex] (L5) -- (L4);
	   \draw[-latex] (L5) -- (X8);
	   \draw[-latex] (L5) -- (X9);
	   \draw[-latex] (L5) -- (X10);
	   \draw[-latex] (L4) -- (L1);
	   \draw[-latex] (L4) -- (L2);
	   \draw[-latex] (L4) -- (L3);
	   
	   \draw[-latex] (L1) -- (X1);
	   \draw[-latex] (L1) -- (X2);
	   \draw[-latex] (L1) -- (X3);
	   \draw[-latex] (L1) -- (X4);
	   \draw[-latex] (L1) -- (X5);
	   \draw[-latex] (L1) -- (X6);
	   
	   \draw[-latex] (L2) -- (X1);
	   \draw[-latex] (L2) -- (X2);
	   \draw[-latex] (L2) -- (X3);
	   \draw[-latex] (L2) -- (X4);
	   \draw[-latex] (L2) -- (X5);
	   \draw[-latex] (L2) -- (X6);
	   \draw[-latex] (L2) -- (X7);
	   
	   \draw[-latex] (L3) -- (X2);
	   \draw[-latex] (L3) -- (X3);
	   \draw[-latex] (L3) -- (X4);
	   \draw[-latex] (L3) -- (X5);
	   \draw[-latex] (L3) -- (X6);
	   \draw[-latex] (L3) -- (X7);
    \end{tikzpicture}
        \caption{ $\cG$ }
    \end{subfigure}
    \hfill
    \begin{subfigure}[b]{0.31\textwidth}
    \centering
    \begin{tikzpicture}[scale=.58, line width=0.4pt, inner sep=0.6mm, shorten >=.1pt, shorten <=.1pt]
      \tikzset{
        znode/.style={text=brown},
        dnode/.style={text=blue},
        xnode/.style={text=black},
        every node/.style={align=center},
        edge from parent/.style={draw,->}
      }
\node[znode] (L5) at (-0.5, 4) {{\footnotesize\,$z_5$\,}};
\node[dnode] (X8) at (-1, 2.5) {{\footnotesize\,$d_8$\,}};
\node[dnode] (X9) at (0, 2.5) {{\footnotesize\,$d_9$\,}};
\node[dnode] (X10) at (1, 2.5) {{\footnotesize\,$d_{10}$\,}};

\node[znode] (L1) at (-4, 2.5) {{\footnotesize\,$z_1$\,}};
\node[znode] (L2) at (-3, 2.5) {{\footnotesize\,$z_2$\,}};
\node[znode] (L3) at (-2, 2.5) {{\footnotesize\,$z_3$\,}};

\node[dnode] (X1) at (-4.5, 1) {{\footnotesize\,$d_1$\,}};
\node[dnode] (X2) at (-3.7, 1) {{\footnotesize\,$d_2$\,}};
\node[dnode] (X3) at (-2.9, 1) {{\footnotesize\,$d_3$\,}};
\node[dnode] (X4) at (-2.1, 1) {{\footnotesize\,$d_4$\,}};
\node[dnode] (X5) at (-1.3, 1) {{\footnotesize\,$d_5$\,}};
\node[dnode] (X6) at (-0.5, 1) {{\footnotesize\,$d_6$\,}};
\node[dnode] (X7) at (0.3, 1) {{\footnotesize\,$d_7$\,}};

	   \draw[-latex] (L5) -- (X8);
	   \draw[-latex] (L5) -- (X9);
	   \draw[-latex] (L5) -- (X10);
	   \draw[-latex] (L5) -- (L1);
	   \draw[-latex] (L5) -- (L2);
	   \draw[-latex] (L5) -- (L3);
	   
	   \draw[-latex] (L1) -- (X1);
	   \draw[-latex] (L1) -- (X2);
	   \draw[-latex] (L1) -- (X3);
	   \draw[-latex] (L1) -- (X4);
	   \draw[-latex] (L1) -- (X5);
	   \draw[-latex] (L1) -- (X6);
	   
	   \draw[-latex] (L2) -- (X1);
	   \draw[-latex] (L2) -- (X2);
	   \draw[-latex] (L2) -- (X3);
	   \draw[-latex] (L2) -- (X4);
	   \draw[-latex] (L2) -- (X5);
	   \draw[-latex] (L2) -- (X6);
	   \draw[-latex] (L2) -- (X7);
	   
	   \draw[-latex] (L3) -- (X2);
	   \draw[-latex] (L3) -- (X3);
	   \draw[-latex] (L3) -- (X4);
	   \draw[-latex] (L3) -- (X5);
	   \draw[-latex] (L3) -- (X6);
	   \draw[-latex] (L3) -- (X7);

    \end{tikzpicture}
    \caption{ $\cO_{\min}(\cG)$ }
    \end{subfigure}    
    \hfill
    \begin{subfigure}[b]{0.31\textwidth}
    \centering
    \begin{tikzpicture}[scale=.58, line width=0.4pt, inner sep=0.6mm, shorten >=.1pt, shorten <=.1pt]
      \tikzset{
        znode/.style={text=brown},
        dnode/.style={text=blue},
        xnode/.style={text=black},
        every node/.style={align=center},
        edge from parent/.style={draw,->}
      }

        \node[znode] (L5) at (-0.5, 4) {{\footnotesize\,$z_5$\,}};
\node[dnode] (X8) at (-1, 2.5) {{\footnotesize\,$d_8$\,}};
\node[dnode] (X9) at (0, 2.5) {{\footnotesize\,$d_9$\,}};
\node[dnode] (X10) at (1, 2.5) {{\footnotesize\,$d_{10}$\,}};

\node[znode] (L1) at (-4, 2.5) {{\footnotesize\,$z_1$\,}};
\node[znode] (L2) at (-3, 2.5) {{\footnotesize\,$z_2$\,}};
\node[znode] (L3) at (-2, 2.5) {{\footnotesize\,$z_3$\,}};

\node[dnode] (X1) at (-4.5, 1) {{\footnotesize\,$d_1$\,}};
\node[dnode] (X2) at (-3.7, 1) {{\footnotesize\,$d_2$\,}};
\node[dnode] (X3) at (-2.9, 1) {{\footnotesize\,$d_3$\,}};
\node[dnode] (X4) at (-2.1, 1) {{\footnotesize\,$d_4$\,}};
\node[dnode] (X5) at (-1.3, 1) {{\footnotesize\,$d_5$\,}};
\node[dnode] (X6) at (-0.5, 1) {{\footnotesize\,$d_6$\,}};
\node[dnode] (X7) at (0.3, 1) {{\footnotesize\,$d_7$\,}};

	   \draw[-latex] (L5) -- (X8);
	   \draw[-latex] (L5) -- (X9);
	   \draw[-latex] (L5) -- (X10);
	   \draw[-latex] (L5) -- (L1);
	   \draw[-latex] (L5) -- (L2);
	   \draw[-latex] (L5) -- (L3);
	   
	   \draw[-latex] (L1) -- (X1);
	   \draw[-latex] (L1) -- (X2);
	   \draw[-latex] (L1) -- (X3);
	   \draw[-latex] (L1) -- (X4);
	   \draw[-latex] (L1) -- (X5);
	   \draw[-latex] (L1) -- (X6);
	   \draw[-latex] (L1) -- (X7);
	   
	   \draw[-latex] (L2) -- (X1);
	   \draw[-latex] (L2) -- (X2);
	   \draw[-latex] (L2) -- (X3);
	   \draw[-latex] (L2) -- (X4);
	   \draw[-latex] (L2) -- (X5);
	   \draw[-latex] (L2) -- (X6);
	   \draw[-latex] (L2) -- (X7);
	   
	   \draw[-latex] (L3) -- (X1);
	   \draw[-latex] (L3) -- (X2);
	   \draw[-latex] (L3) -- (X3);
	   \draw[-latex] (L3) -- (X4);
	   \draw[-latex] (L3) -- (X5);
	   \draw[-latex] (L3) -- (X6);
	   \draw[-latex] (L3) -- (X7);
      
    \end{tikzpicture}
    \caption{ $\cO_{s}(\cO_{\min}(\cG))$ }
    \end{subfigure}
    \caption{
        \textbf{
        The discrete graph $\cG$ satisfies conditions in Theorem~\ref{thm:identical_support} (i.e., identical supports).}
        After applying the minimal-graph operator to the graph $\cG$, $z_4$ is merged to its parent $z_5$, and the rank constraints do not change.
        After applying the skeleton operator to the graph in (b), $z_1$ has an edge to $d_7$ and $z_3$ has an edge to $d_1$.
        We adopt this example from \citet{huang2022latent}.
    }
    \label{fig:minimal_operator}
\end{figure}

\input{sections/algorithms/algorithms0}

\input{sections/algorithms/algorithms1}
\input{sections/algorithms/algorithms2}
\input{sections/algorithms/algorithms3}
\input{sections/algorithms/algorithms4}

\section{Synthetic Data Experiments} \label{app:synthetic}

\paragraph{Data-generating processes.} 
For the hierarchical model $\cG$, we randomly sample the parameters for each causal module, i.e., conditional distributions $p( z_{i} | \text{Parents} (z_{i}) )$, according to a Dirichlet distribution over the states of $z_{i}$ with coefficient $1$.
For simplicity, we follow conditions in Theorem~\ref{thm:identical_support} and set the support size of latent variables to $2$.
Like \citet{kivva2021learning}, we build the generating process from $\dd$ to the observed variables $\mathbf{x}$ (i.e., graph $\Gamma$) by a Gaussian mixture model where each state of the discrete subspace corresponds to one component/mode in the mixture model. We truncate the support of each component to improve the invertibility (Condition~\ref{cond:discrete_component_conditions}-\ref{asmp:invertibility}).
The graphs are exhibited in Figure~\ref{fig:synthetic_with_x} and Figure~\ref{fig:synthetic_without_x}.

\paragraph{Metrics.}
We adopt F1 score (i.e., $ \frac{ 2 \text{Precision} \cdot \text{Recall} }{ \text{Precision} + \text{Recall} } $ ) to assess the graph learning results~\citep{dong2023versatile}. 
We compute recall and precision by checking whether the estimated model correctly retrieves edges in the true causal graph.  
Ranging between $0$ to $1$, high F1 scores indicate the search algorithm can recover ground-truth causal graphs.
We repeat each experiment over at least $5$ random seeds.

\paragraph{Implementation details.} 
Our method comprises two stages: 1) learning the bottom-level discrete variable $\dd$ and the bipartite graph $\Gamma$ from the observed variable $\xx$; 2) learning the latent hierarchical model $\cG$ given the bottom-level discrete variable $\dd$ discovered in 1).
For stage 1), we follow the clustering implementation in \citet{kivva2021learning} under the same hyper-parameter setup as in the original implementation.
For stage 2), we apply Algorithm~\ref{alg:all} to learn the hierarchical model $\cG$. 
We opt for Step~\ref{step:convert_2} in Algorithm~\ref{alg:all} because we evaluate graphs with binary latent variables that meet the conditions of Theorem~\ref{thm:identical_support}. 
Following \citet{anandkumar2012learning,mazaheri2023causal}, we perform conventional rank computation rather than non-negative rank computation and find this replacement satisfactory.
We conduct our experiments on a cluster of 64 CPUs.
All experiments can be finished within half an hour. 
The search algorithm implementation is adapted from \citet{dong2023versatile}.

\paragraph{Graphical structures.} 
Table~\ref{tab:synthetic_with_x} and Table~\ref{tab:synthetic_without_x} correspond to  Figure~\ref{fig:synthetic_with_x} and Figure~\ref{fig:synthetic_without_x} respectively.
As mentioned above, the graphs meet the conditions of Theorem~\ref{thm:identical_support} with the latent variable cardinality equal to two (binary variables).

\begin{figure}
\centering
 \includegraphics[width=1.0\columnwidth]{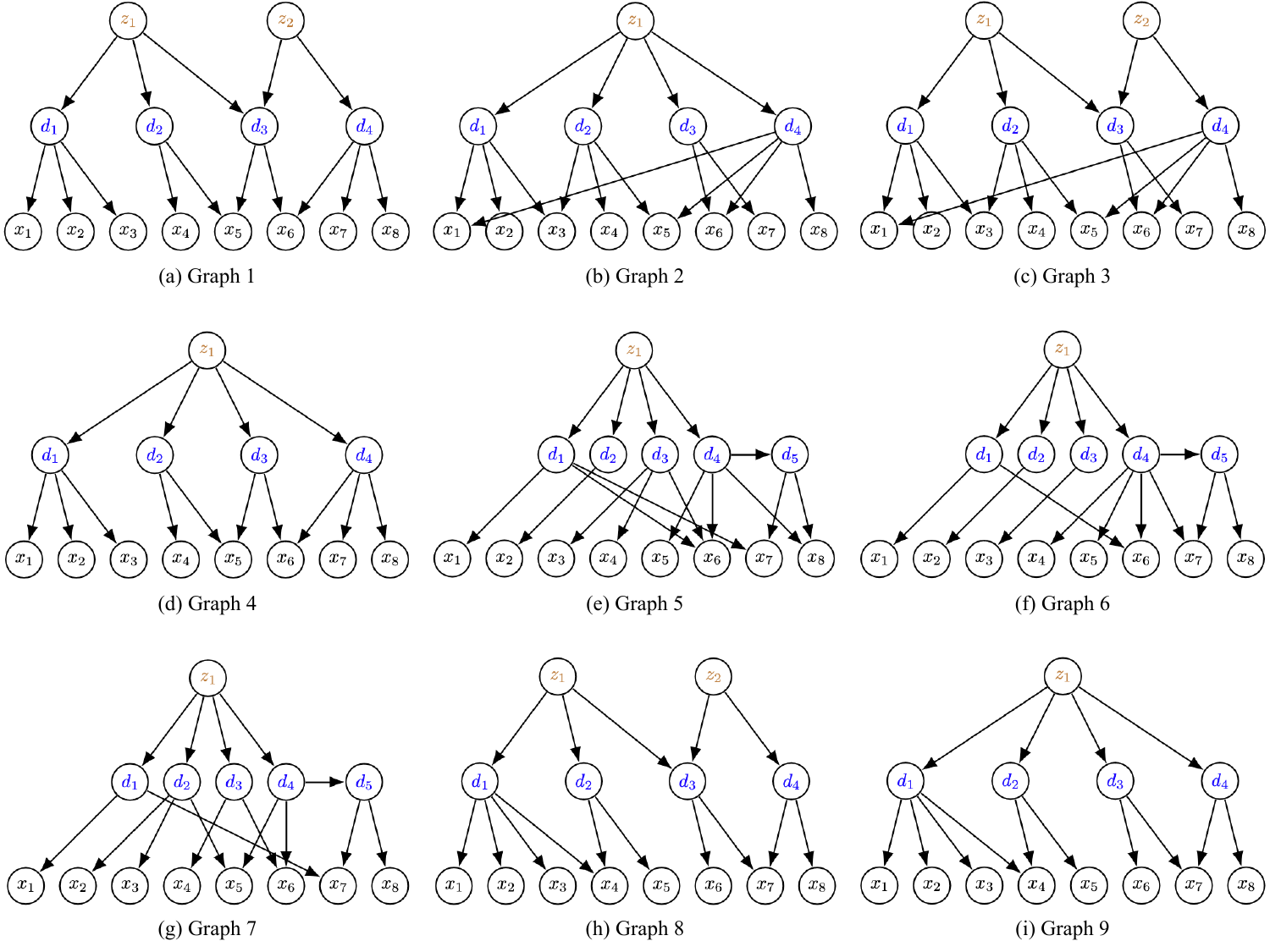}
\caption{
    \textbf{Causal graphs evaluated in Table~\ref{tab:synthetic_with_x}.}
    We denote the observed variables with $x$, the bottom-level latent discrete variables with $d$, and the high-level latent discrete variables with $z$.
}
\label{fig:synthetic_with_x}
\end{figure}

\begin{figure}
\centering
 \includegraphics[width=1.0\columnwidth]{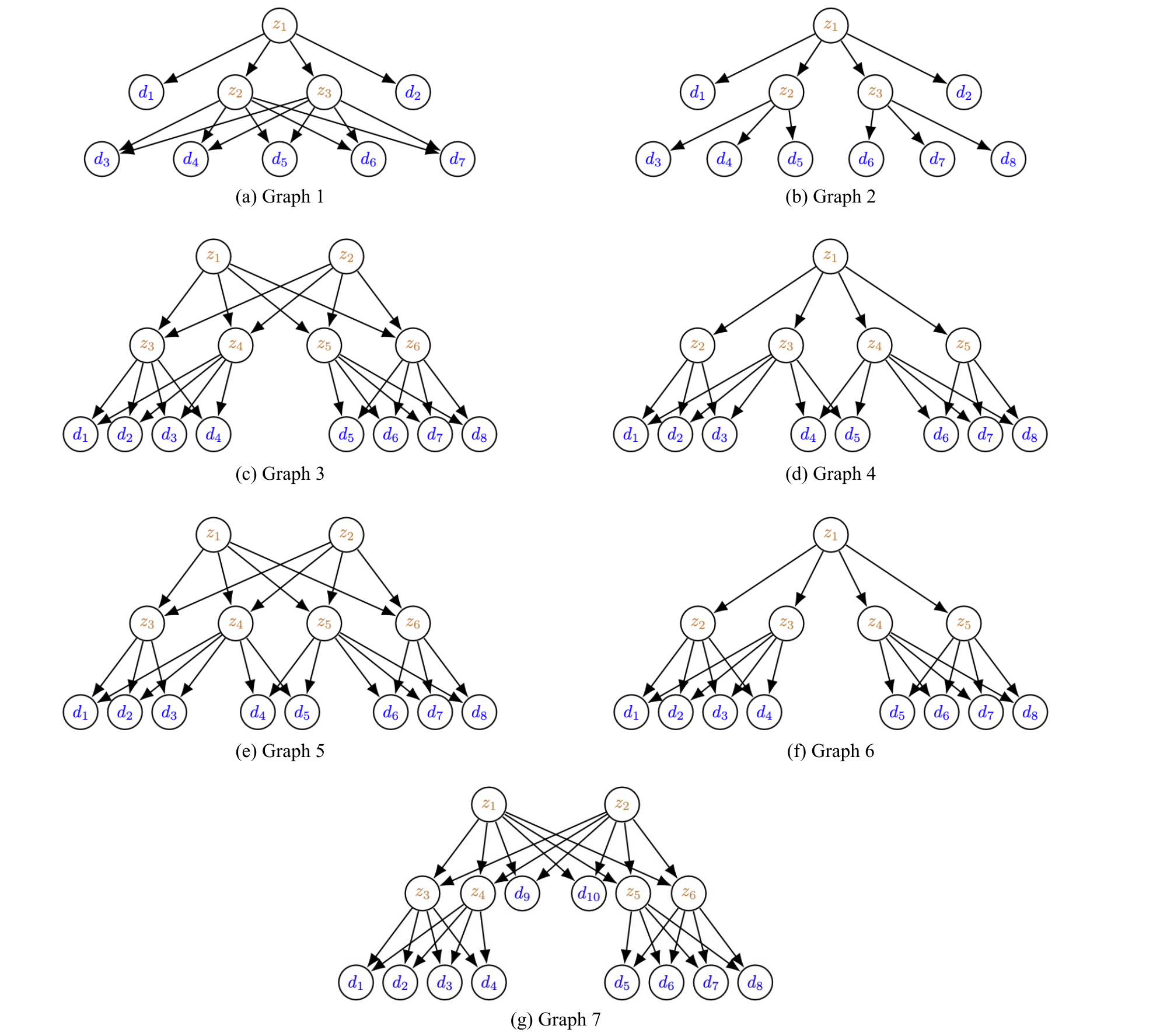}

\caption{
    \textbf{Causal graphs evaluated in Table~\ref{tab:synthetic_without_x}.}
    We denote the bottom-level latent discrete variables with $d$, and high-level latent discrete variables with $z$. Since baselines cannot extract discrete subspaces, we directly feed the algorithms the bottom-level discrete variables and test their structure learning performances.
}
\label{fig:synthetic_without_x}
\end{figure}

\section{Real-world Experiments} \label{app:implementation}

\subsection{Implementation Details}
We employ the pre-trained latent diffusion model~\citep{rombach2021highresolution} SD v1.4 across all our experiments.
The inference process consists of $50$ steps.

For experiments in Section~\ref{subsec:hierarchical_ordering}, we inject concepts by appending keywords to the original prompt.
For instance, we inject the concept pair (``sketch'', ``wide eyes'') in Figure~\ref{fig:hierarchical_orders} as follows.
For the inference steps $0 - 10$, we feed the text prompt ``A picture of a person'', for steps $10-20$, ``a photo of a person, in a sketch style'', and for steps $ 20 - 50 $, ``a photo of a person, in a sketch style, with wide eyes''.
For the reverse injection order (injecting ``wide eyes'' before ``sketch''), we inject the following prompts at the three-step stages:
``A picture of a person'', ``a photo of a person, with wide eyes'', and  ``a photo of a person, with wide eyes, in a sketch style".


For experiments understanding the UNet's latent presentation (Figure~\ref{fig:semantic_unet}), we adopt the open-sourced code of \citet{park2023understanding}.

For the attention sparsity experiment (Figure~\ref{fig:attention_sparisty}), we randomly generate images with the pre-train latent diffusion model and record their attention score across layers.
To compute the relative sparsity, we select the threshold as $ 1 / 4096 $ and compute the proportion of the attention scores over this threshold.
For the attention visualization, we randomly select a head from the last attention module in the UNet architecture.

We follow the implementation of \citet{gandikota2023sliders} to train concept sliders of various ranks.
We adopt their evaluation protocol to obtain CLIP and LPIPS scores over $20$ randomly sampled images for each rank, concept, and scale combination.
We evaluate ranks in $ \{ 2,4, 8 \} $ and scales $ \{ 1, 2, 3, 4, 5 \} $.
The rank selection technique, inspired by \citet{ding2023sparse}, involves multiplying each LoRA's inner dimension with a scalar parameter and imposing $\ell_{0}$ penalty on these scalar parameters.
The weight on the $\ell_{0}$ penalty is selected from $ \{1e-1, 1e-2, 1e-3, 1e-4, 1e-5\} $. 
We repeat each run for at least three random seeds.
The code can be found \href{https://github.com/Lingjing-Kong/sparsity_diffusion_editing.git}{here}.

We conduct all our experiments on 2 Nvidia L40 GPUs.
Each image inference takes the same time as in standard SD v1.4 (i.e., within two minutes).
Each concept slider in Figure~\ref{fig:method_evaluation} takes around half an hour to train.

\subsection{Sparsity in the Hierarchical Model} \label{subsec:attention_sparsity}
To verify the sparse structure condition in Theorem~\ref{thm:hierarchical_model_identification}, we view the attention sparsity in the LD model as an indicator of the connectivity between a specific hierarchical level and the bottom concept level.
Figure~\ref{fig:attention_sparisty} visualizes the attention sparsity of an LD model over diffusion steps and specific attention patterns in the model.
We observe that the sparsity increases as the generative process progresses, which reflects that the connectivity between the hierarchical level ($\zdd{\cS_{t}}$) and the bottom level variable ($\dd$) becomes sparse and more local as we march down the hierarchical structure, which indicates a gradual localization of the concept. 

\begin{figure*}
     \centering
     \begin{subfigure}[b]{0.21\textwidth}
         \centering
         \includegraphics[width=\textwidth]{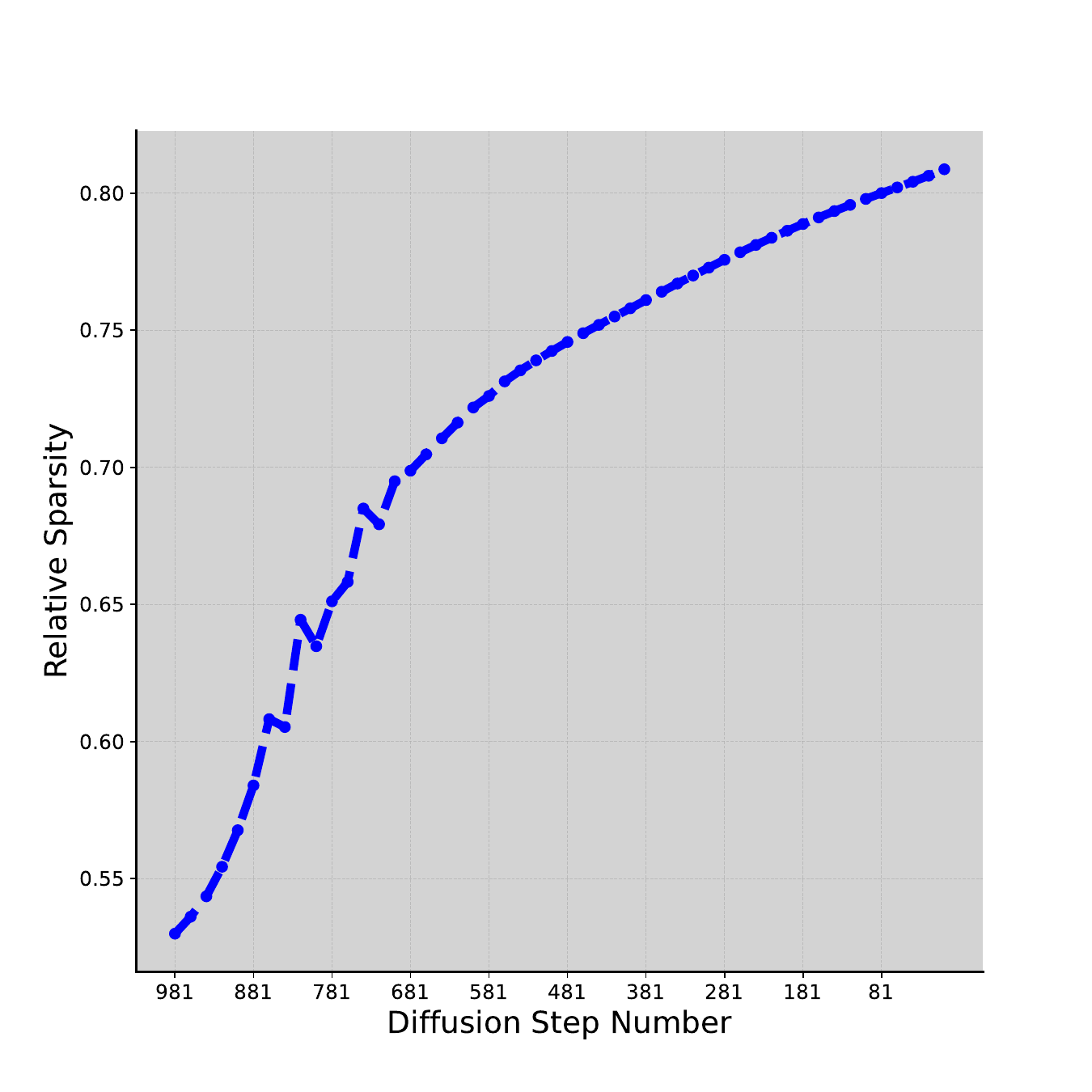}
         \caption{\small Attention sparsity.}
     \end{subfigure}
     \hfill
     \begin{subfigure}[b]{0.25\textwidth}
         \centering
         \includegraphics[width=\textwidth]{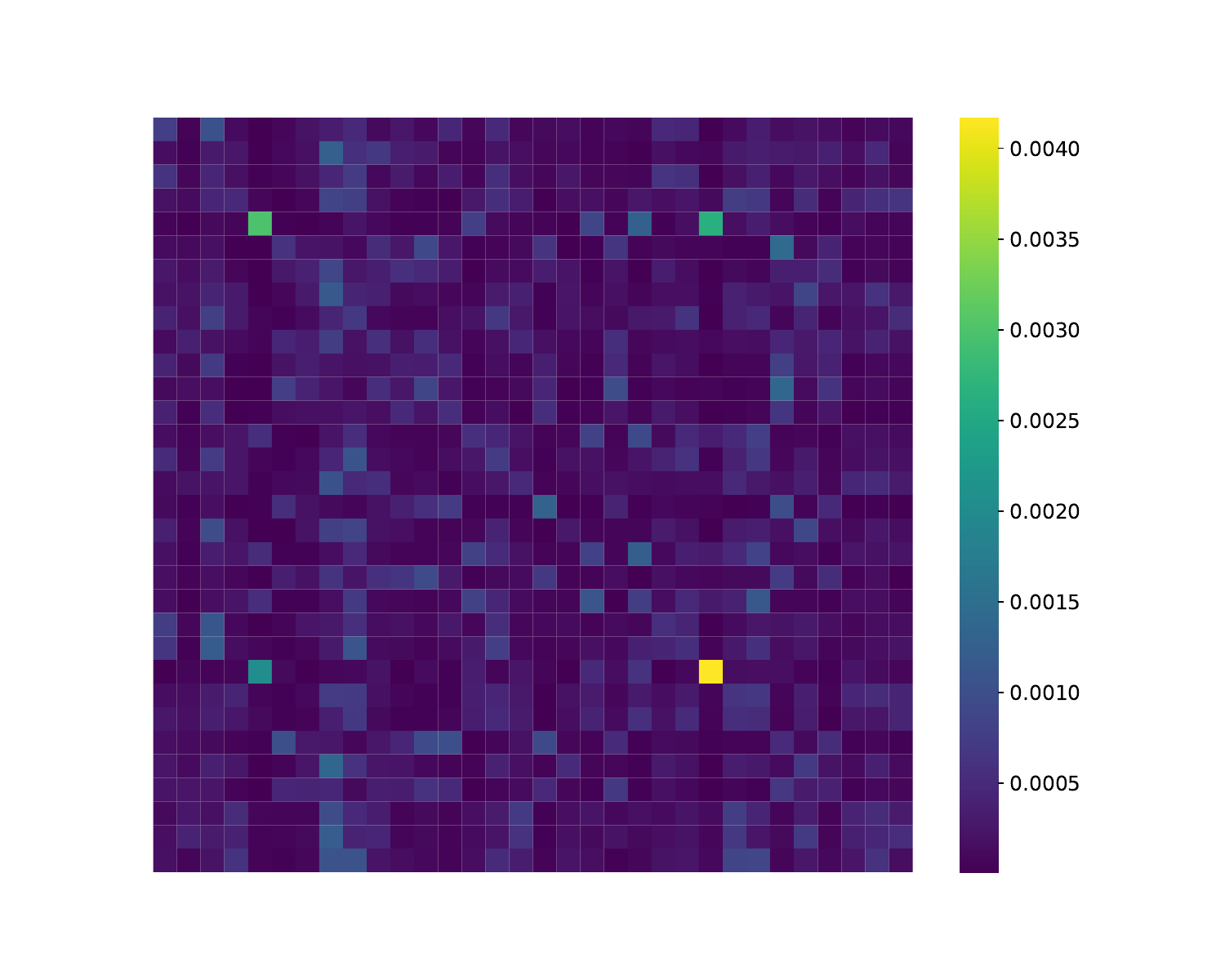}
         \caption{\small Early step (index 1).}
     \end{subfigure}
     \hfill
     \begin{subfigure}[b]{0.25\textwidth}
         \centering
         \includegraphics[width=\textwidth]{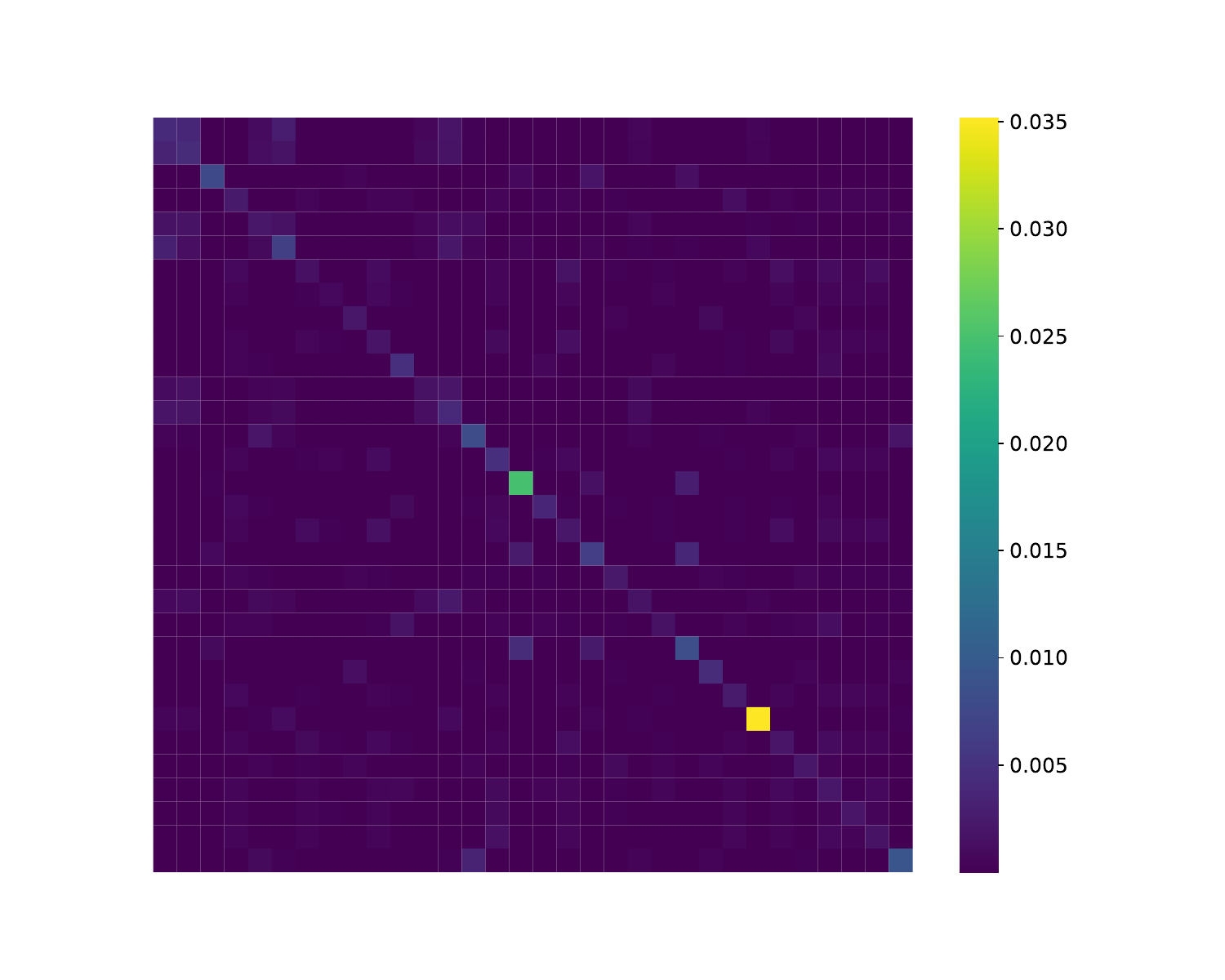}
         \caption{\small Middle step (index 481).}
     \end{subfigure}
     \hfill
     \begin{subfigure}[b]{0.25\textwidth}
         \centering
         \includegraphics[width=\textwidth]{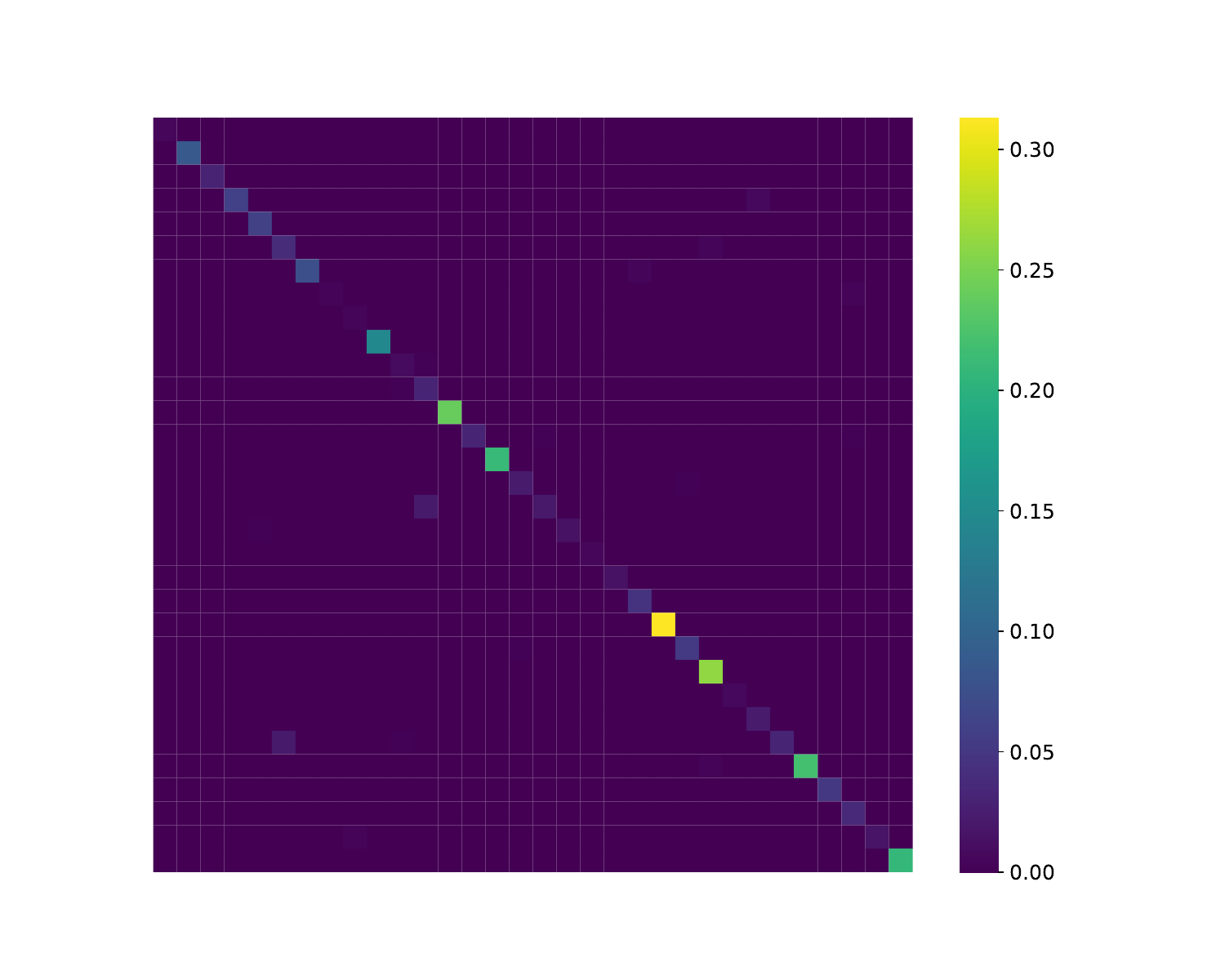}
         \caption{\small Late step (index 981).}
     \end{subfigure}
    \caption{
         \textbf{Sparsity patterns in latent diffusion models' attention.}
        We compute the proportions of the attention scores lower than a fixed threshold over the entire model.
        We can observe that the sparsity increases greatly towards small timesteps, i.e., the lower levels of the hierarchical model, which verifies our theory.
    }
    \label{fig:attention_sparisty}
\end{figure*}

\subsection{Discovering Hierarchical Orders from Diffusion Models} \label{app:causal_order_exps}

We provide further evidence that latent representations at different diffusion steps correspond to different levels of the hierarchical causal model.
We select concept pairs, each with higher-level and lower-level concepts. For example, in (``sketch,'' ``wide eyes''), ``sketch'' is more global, while ``wide eyes'' is more local. We alter the text prompt during diffusion generation for concept injection, appending ``in a sketch style'' to inject ``sketch'' (see Appendix~\ref{app:implementation} for prompts).
In Figure~\ref{fig:hierarchical_orders}, global concepts are successfully injected at early diffusion steps and local ones at late steps (top row). Reversing this order fails, as shown in the bottom row. 
For example, injecting ``sketch'' early and ``wide eyes'' late renders both correctly, but the global concept ``sketch'' is absent under the reverse injection order. 
This supports our theory that concepts are hierarchically organized, with higher-level concepts related to earlier diffusion steps.

\begin{figure*}[t]
    \centering
    \includegraphics[width=\textwidth]{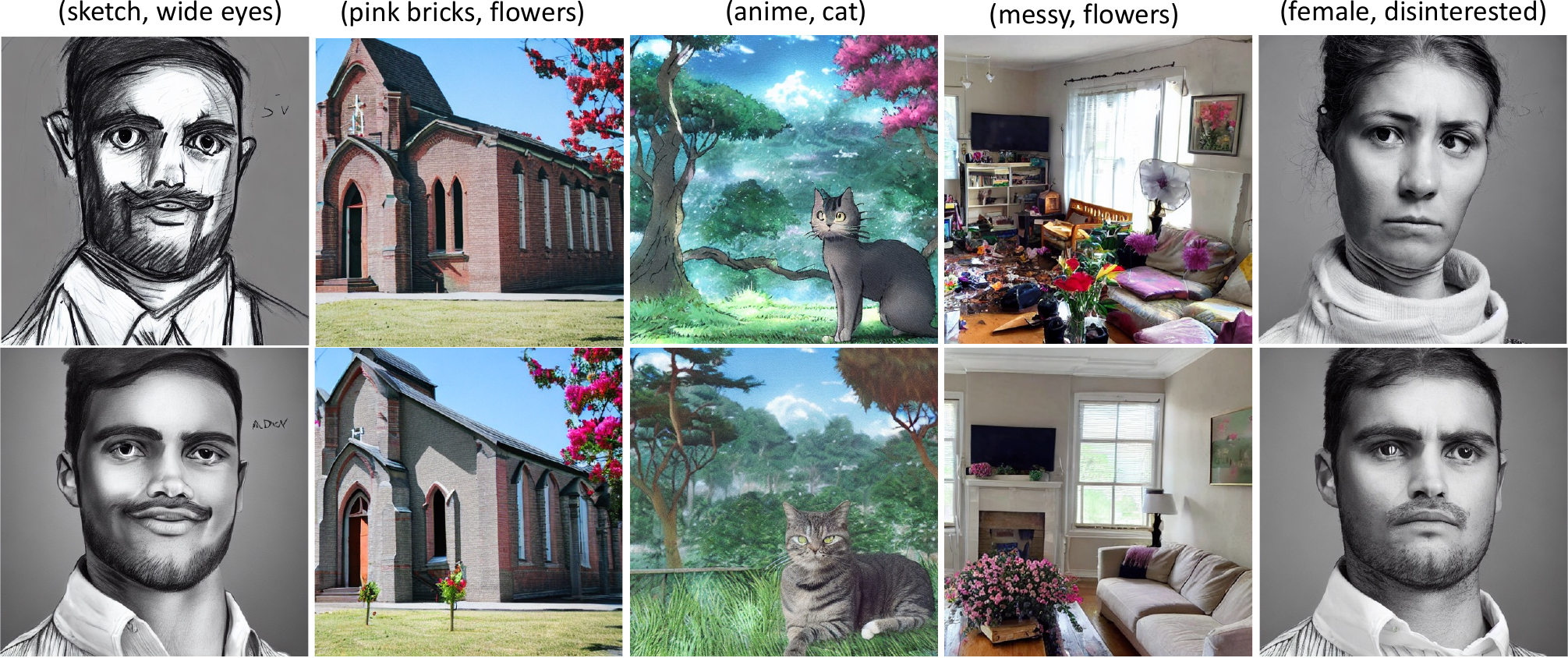}
    \caption{
        \textbf{Hierarchical Concept Ordering.} 
        We inject concepts of distinct abstraction levels into the generating process at different time steps.
        In the top row, the concept injection follows the hierarchical order, which renders injected concepts faithfully.
        The bottom row reverses the hierarchical order and cannot incorporate concepts properly. More examples in Figure~\ref{fig:hierarchical_order_more}.
    }
    \label{fig:hierarchical_orders}
\end{figure*}

\subsection{Causal Sparsity for Concept Extraction} \label{app:sparsity_results}

Figure~\ref{fig:sparse_concepts_visualization} shows that indeed concepts at different abstraction levels have desirable representations at different ranks.
For instance, the concept of bright weather is appropriately conveyed by a rank-$2$ LoRA and higher-rank LoRAs alter the background.
The same observation occurs to other concepts, where inadequate ranks fail to capture the concept faithfully and unnecessary ranks entangle the target concept with other attributes. 

Figure~\ref{fig:method_evaluation} presents the CLIP and LPIPS evaluation for the baseline and our approach, where the CLIP score evaluates the alignment between the image and the target description and the LPIPS score measures the structure change between the edited image and the original image. 
We can observe that under the sparsity constraint, our approach attains the highest CLIP score and the lowest LPIPS score when compared with the baselines of several ranks, indicating a higher level of alignment and a lower level of undesirable entanglement.

\begin{figure}[t]
    \centering
    \includegraphics[width=\columnwidth]{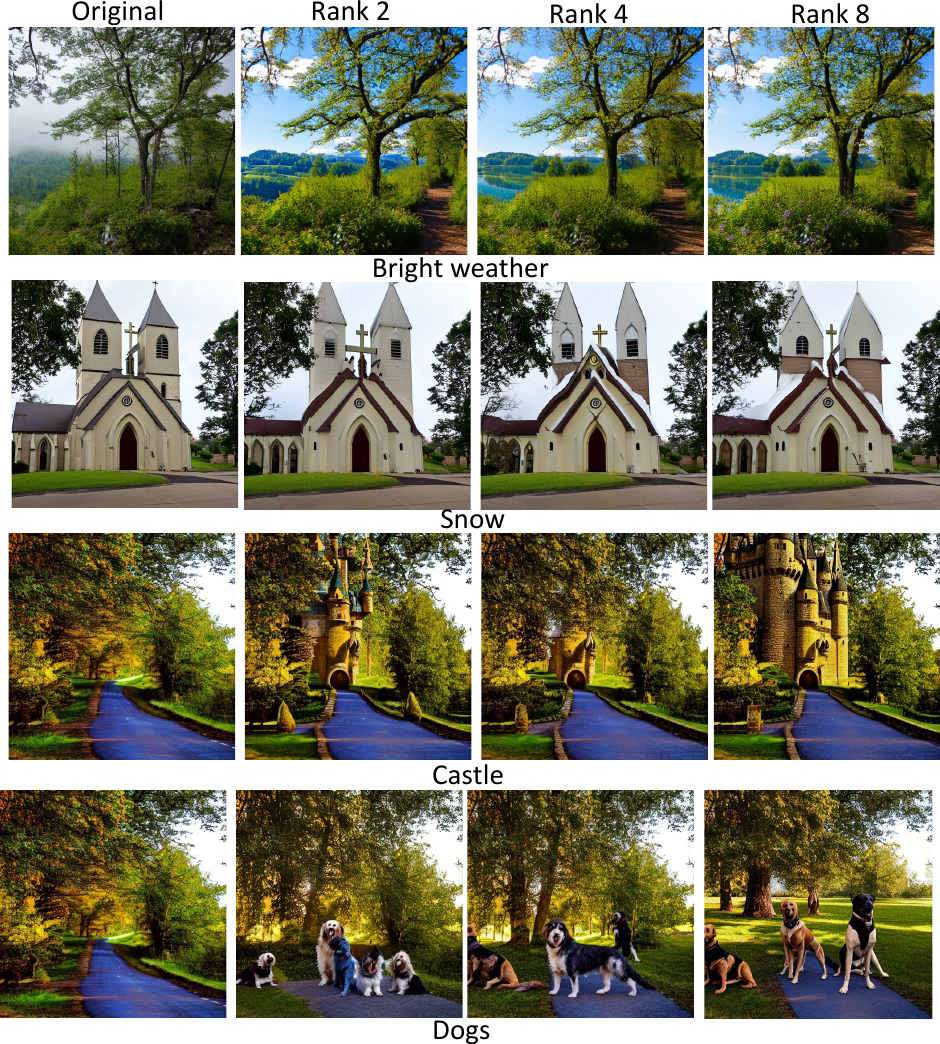}
    \caption{
        \textbf{Concepts have varying levels of sparsity.}
        We show that concepts of various abstraction levels correspond to different sparsity levels.
        For instance, bright weather is appropriately conveyed by a rank-$2$ LoRA and higher-rank LoRAs alter the background.
        Inadequate ranks fail to capture the concept faithfully and unnecessary ranks entangle the target concept with other attributes.
    }
    \label{fig:sparse_concepts_visualization}
\end{figure}

\begin{figure}[t]
     \centering
     \begin{subfigure}[b]{.8\columnwidth}
         \centering
         \includegraphics[width=\textwidth]{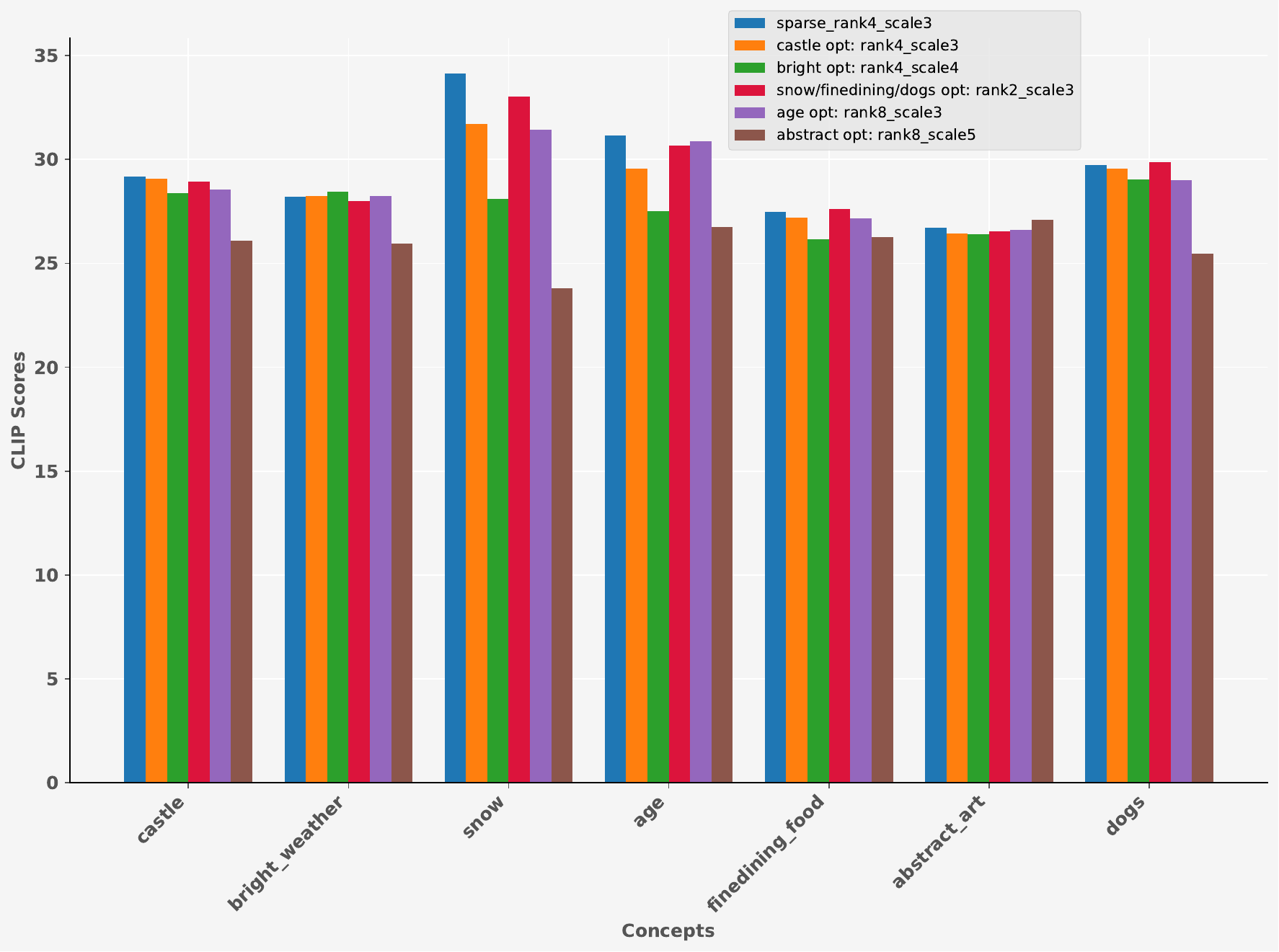}
         \caption{CLIP scores.}
     \end{subfigure}
     \hfill
     \begin{subfigure}[b]{.8\columnwidth}
         \centering
         \includegraphics[width=\textwidth]{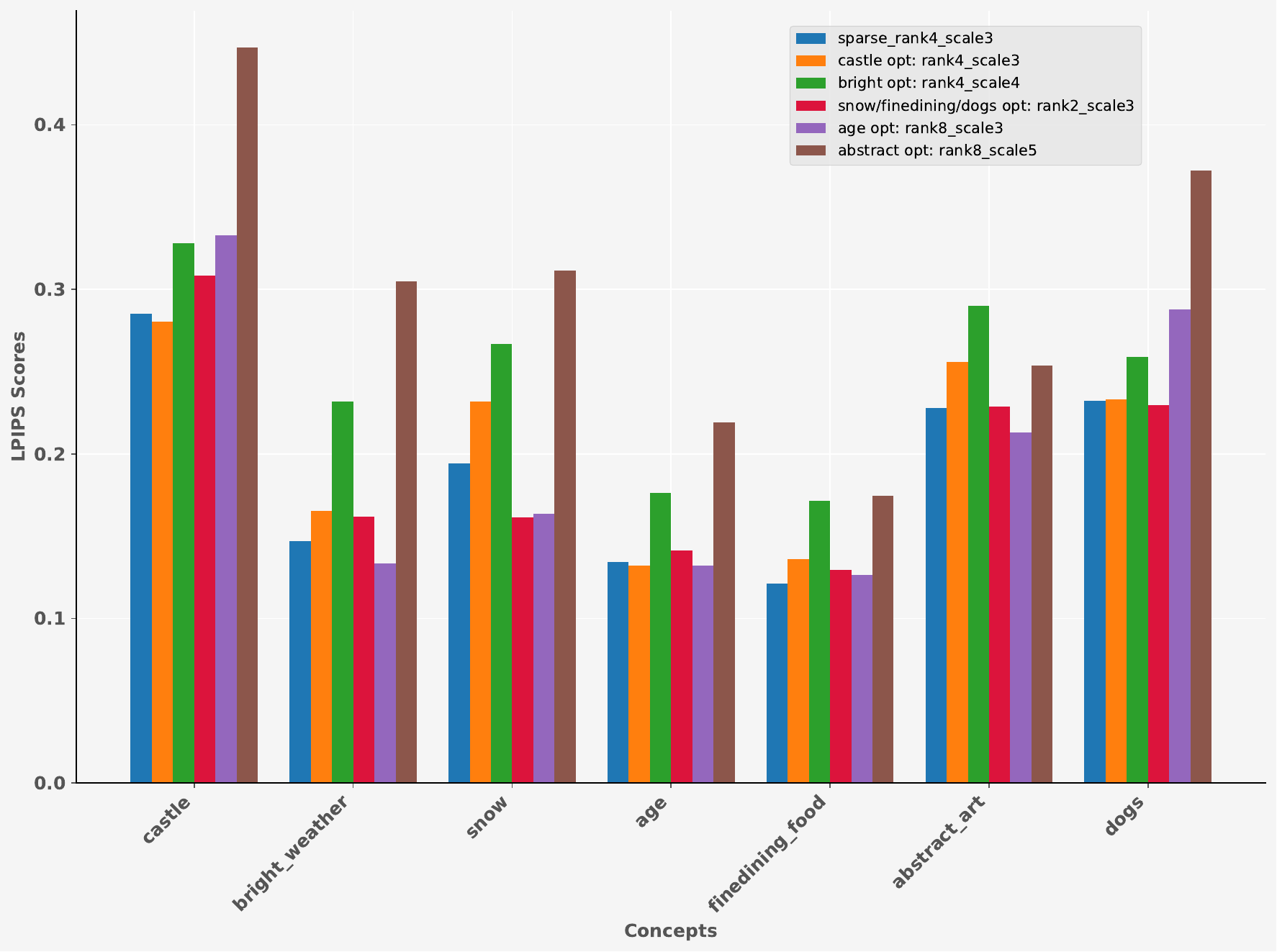}
         \caption{LPIPS scores.}
     \end{subfigure}
    \caption{
        \textbf{CLIP/LPIPS evaluation.}
        We evaluate our approach and baselines at individual rank constraints.
        A high CLIP score is favorable as it indicates semantic alignment.
        A low LPIPS score is more favorable as it indicates minimal excessive changes.
        We compare our method ``sparse'' with the optimal fixed rank setting on each concept. For instance, ``castle opt: rank4\_scale3'' indicates that the optimal setting for the concept ``castle'' is the LoRA of rank $4$ and scale $3$.
        With a adaptive rank selection, our approach outperforms or keeps up with the optimal fixed setting across different concepts.
        We repeat each training over three random seeds.
    }
    \label{fig:method_evaluation}
\end{figure}

\subsection{More Examples}

\begin{figure}[t]
    \centering
    \includegraphics[width=\textwidth]{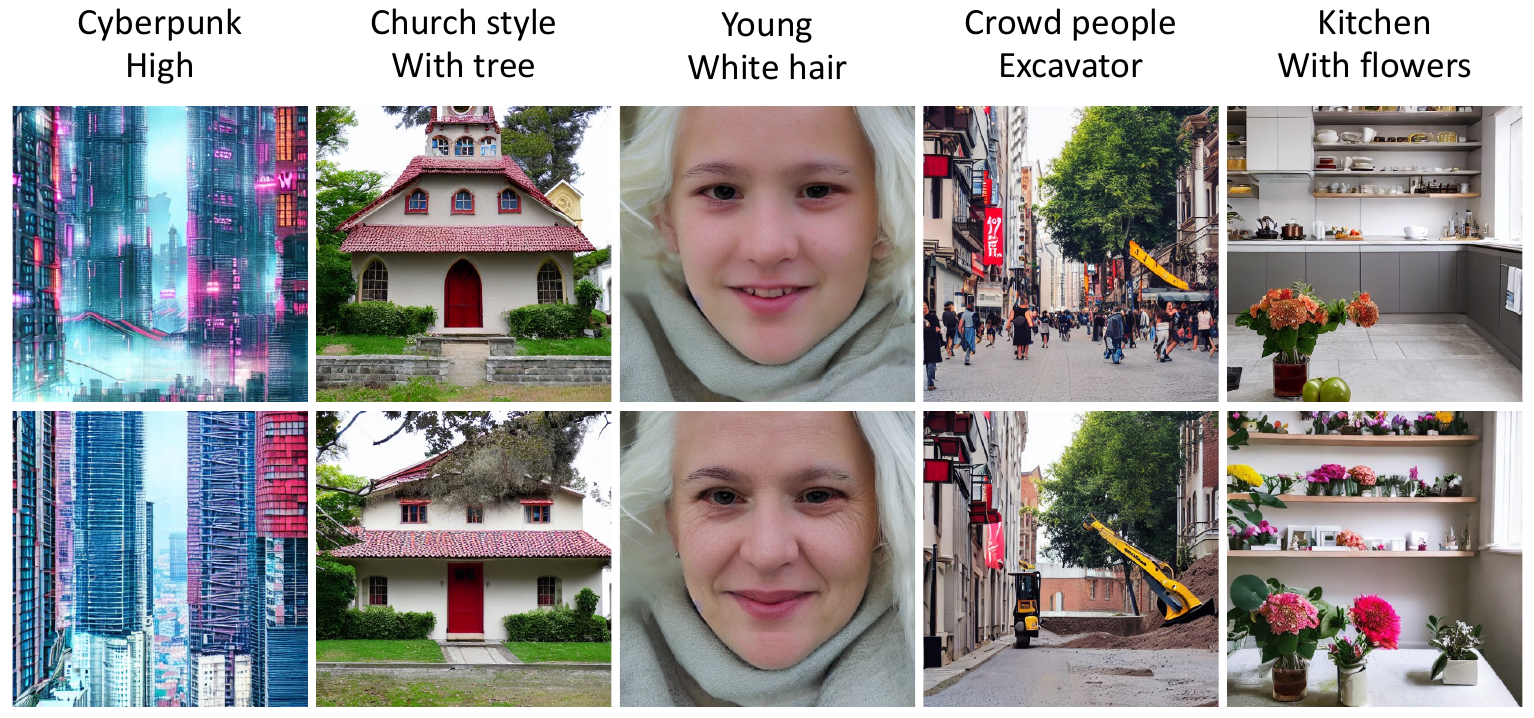}
    \caption{\small \textbf{More examples for Figure~\ref{fig:hierarchical_orders}}.}
    \label{fig:hierarchical_order_more}
\end{figure}

\begin{figure}[t]
    \centering
    \includegraphics[width=\textwidth]{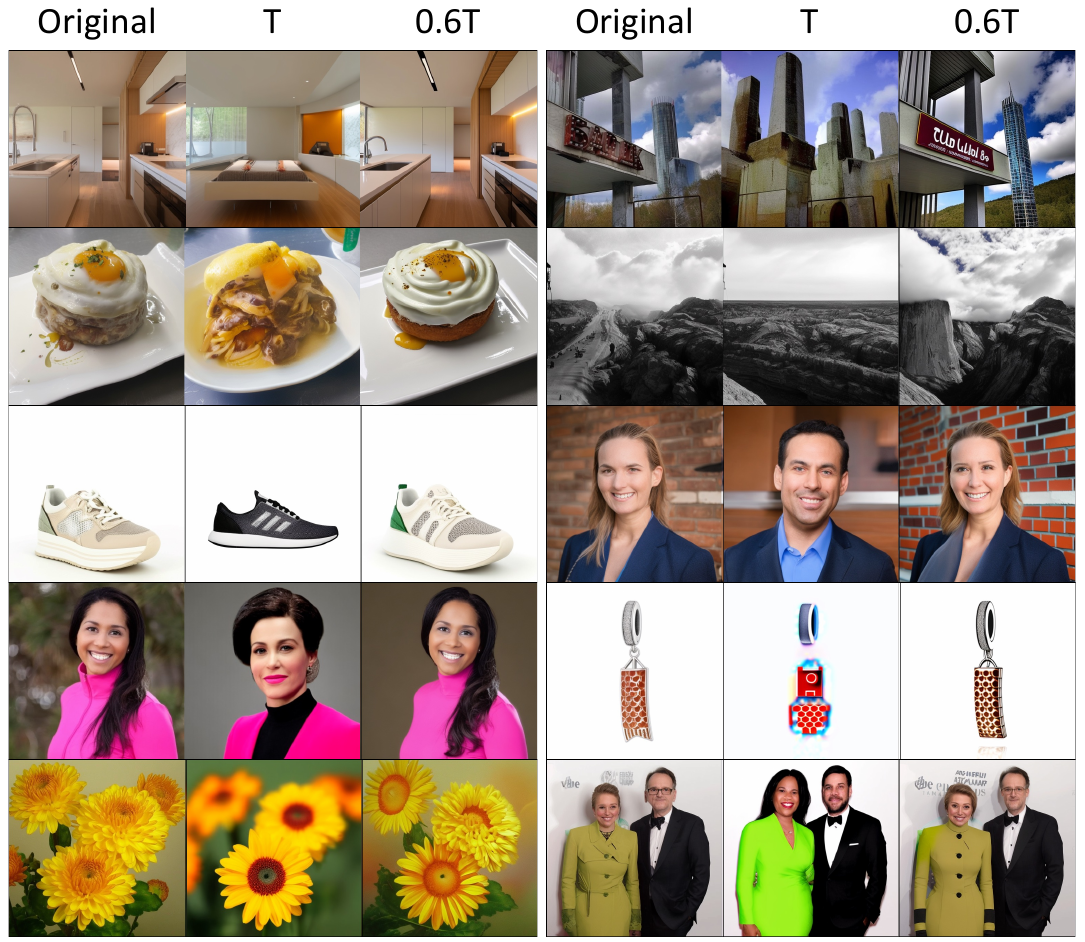}
    \caption{\small \textbf{More examples for Figure~\ref{fig:semantic_unet}}.}
    \label{fig:semantic_unet_more}
\end{figure}

%% file: sections/related_work_full.tex
\section{Related Work} \label{app:related_work}

\paragraph{Concept learning.}
In recent years, a significant strand of research has focused on employing labeled data to learn concepts in generative models' latent space for image editing and manipulation~\citep{gal2022image,jahanian2020steerability,härkönen2020ganspace,shen2020interpreting,wu2020stylespace,ruiz2023dreambooth}.
Concurrently, another independent research trajectory has been exploring unsupervised concept discovery and its potential to learn more compositional and transferable models, as shown in~\citet{burgess2019monet, locatello2020objectcentric,du2021unsupervised,du2021unsupervised3d,liu2023unsupervised}.
These prior works focus on the empirical methodological development of concept learning by proposing novel neural network architectures and training objectives, with limited discussion on the theoretical aspect. In contrast, our work investigates the theoretical foundation of concept learning. Specifically, we formulate concept learning as an identification problem for a discrete latent hierarchical model and provide conditions under which extracting concepts is possible. Thus, the existing work and our work can be viewed as two complementary lines of research for concept learning.
Concurrently, a plethora of work has been dedicated to extracting interpretable concepts from high-dimensional data such as images. 
Concept-bottleneck~\citep{koh2020concept} first predicts a set of human-annotated concepts as an intermediate stage and then predicts the task labels from these intermediate concepts. This paradigm has attracted a large amount of follow-up work~\citep{zarlenga2022concept,yuksekgonulpost,kim2023probabilistic,havasi2022addressing,shang2024incremental,chauhan2023interactive}. A recent surge of pre-trained multimodal models (e.g., CLIP~\citep{radford2021learning}) can explain the image concepts through text directly~\citep{oikarinen2022clip,moayeri2023text,moayeri2023text2concept}. 
In contrast with these successes, our work focuses on the formulation of concept learning and theoretical guarantees.

\paragraph{Latent hierarchical models.}
Complex real-world data distributions often possess a hierarchical structure among their underlying latent variables.
On the theoretical front, \citet{xie2022identification,huang2022latent,dong2023versatile} investigate identification conditions of latent hierarchical structures under the assumption that the latent variables are continuous and influence each other through linear functions.
\citet{kong2023identification} extends the functional class to the nonlinear case over continuous variables.
\citet{Pearl88,zhang2004hierarchical,choi2011learning,gu2023bayesian} study fully discrete cases and thus fall short of modeling the continuous observed variables like images.
Specifically, \citet{Pearl88,zhang2004hierarchical,choi2011learning} focus on the latent trees in which every pair of variables is connected through exactly one undirected path.
\citet{gu2023bayesian} assume a multi-level DAG~\citep{anandkumar2013learning} in which variables can be partitioned into disjoint groups (i.e., levels), such that all edges are between adjacent levels, with the observed variables as the bottom level (i.e., leaf nodes).
In contrast, we show that we can not only extract discrete components from continuous observed variables but also uncover higher-level concepts and their interactions.
Our graphical conditions admit multiple paths within each pair of latent variables, flexible hierarchical structures that are not necessarily multi-level, and flat structures in which all latent variables are adjacent to observed variables~\citep{kivva2021learning}.
On the empirical side, prior work~\citep{sonderby2016ladder} improves the inference model of vanilla VAEs by combining bottom-up data-dependent likelihood terms with prior generative distribution parameters. 
\citet{zhao2017learning} assign more expressive (deeper) neural modules to higher-level variables to learn a more disentangled generative model. 
\citet{li2019learning} present a VAE/clustering approach to empirically estimating latent tree structures.
\citet{leeb2024structure} propose to feed latent variable partitions into different decoder neural network layers and remove the prior regularization term to enable high-quality generation. Like our work, \citet{ross2022benchmarks} consider discrete latent variables. However, their focus is on empirical evaluation benchmarks and metrics, without touching on the theoretical formulation of this task. Unlike these efforts, our work concentrates on the formalization of the data-generating process and the theoretical understanding. Thus, these two lines complement each other.

\paragraph{Latent variable identification.}
Identifying latent variables under nonlinear transformations is central to representation learning on complex unstructured data. 
\citet{khemakhem2020variational,khemakhem2020icebeem,hyvarinen2016unsupervised,hyvarinen2019nonlinear} assume the availability of auxiliary information (e.g., domain/class labels) and that the latent variables' probability density functions have sufficiently different derivatives over domains/classes.
However, many important concepts (e.g., object classes) are inherently discrete. 
Since latent variables are not equipped with differentiable density functions, identifying these concepts necessitates novel techniques.
Our theory requires neither domain/class labels nor differentiable density functions and can accommodate discrete variables readily.
Another line of studies~\citep{brady2023provably,lachapelle2023additive} refrains from the auxiliary information by making sparsity and mechanistic independence assumptions over latent variables, disregarding causal structures among the latent variables.
Moreover, images may comprise abstract concepts and convey sophisticated interplay among concepts at various levels of abstraction.
In this work, we address these limitations by formulating the concept space as a discrete hierarchical causal model, capturing concepts at distinct levels and their causal relations.

\paragraph{Latent diffusion understanding.}
Diffusion probabilistic models~\citep{sohldickstein2015deep,ho2020denoising,rombach2021highresolution,song2022denoising,dhariwal2021diffusion,nichol2021improved} have recently become the workhorse for state-of-the-art image generation.
Diffusion models' empirical success sparked a plethora of efforts to probe into their empirical properties.
\citet{kwon2023diffusion,park2023understanding} discover that the UNet bottleneck representation exhibits highly structured semantic properties, traversing over which manipulates the generated image in a meaningful manner.
\citet{choi2022perception,daras2022multiresolution,wu2022uncovering,sclocchi2024phase} realize that early/late diffusion steps at the inference correlate with coarse/fine features in the output. 
Recently, \citet{gandikota2023sliders} showcase that concepts are encoded by low-rank influences in latent diffusion models.     
The theoretical insights in our work consolidate these apparently separate strands of empirical observations and also lead to new understandings that could enhance empirical methodologies.

%% file: sections/algorithms/algorithms0.tex

\begin{figure}[h]
\begin{algorithm}[H]
\small
  \caption{
    \textbf{The overall procedure for Rank-based Discrete Latent Causal Model Discovery.}~\citep{dong2023versatile}
    We denote the latent nodes in an atomic cover $\mA$ as $\mA_{L}$ and all observed nodes in the model $\cG$ as $\mX_{\cG}$.
    We use blue color to highlight our modifications needed for Theorem~\ref{thm:hierarchical_model_identification} and Theorem~\ref{thm:identical_support} respectively.
  }
  \label{alg:all}
  \SetAlgoLined
  \SetKwInOut{Input}{Input}
  \SetKwInOut{Output}{Output}
  \Input{Samples from all $n$ observed variables $\mathbf{X}_{\mathcal{G}}$}
  \Output{Markov equivalence class $\graphp$}
  \SetKwProg{Def}{def}{:}{}
  \Def{LatentVariableCausalDiscovery($\set{X}_{\graph}$)}
  {
    {Phase 1: $\graphp$ = FindCISkeleton($\set{X}_{\graph}$)} (Algorithm~\ref{alg:phase1})\;
    \For{ Each $\setset{Q}$, a group of overlapping maximal cliques, in $\graphp$} 
    {
      Set an empty graph $\graphpp$,      $\set{X}_\setset{Q}=\cup_{\set{Q}\in\setset{Q}}\set{Q}$,
      $\set{N}_\setset{Q}=\{\node{N}: \exists \node{X}\in \set{X}_\setset{Q}~\text{s.t.}~ 
      \node{N},\node{X}~\text{are adjacent in}~\graphp \}$\;
      {Phase 2: $\graphpp$ = FindCausalClusters($\graphpp$, $\set{X}_\setset{Q}\cup\set{N}_\setset{Q}$) (Algorithm~\ref{alg:phase2})\;}
      {Phase 3: $\graphpp$ = RefineCausalClusters($\graphpp$, $\set{X}_\setset{Q}\cup\set{N}_\setset{Q}$) (Algorithm~\ref{alg:phase3})\;}
      Transfer the estimated DAG $\graphpp$ to the Markov equivalence class and update $\graphp$ by $\graphpp$\;
    }
    Orient remaining causal directions that can be inferred from v structures\;
    \ours{Theorem~\ref{thm:hierarchical_model_identification}: replace each cover $\mA$ to a discrete variable $z$ with $\abs{\text{Supp}(z)} := \abs{\mA_{L}}$ \& update $\cG'$} \label{step:convert_1}\;
    \ours{Theorem~\ref{thm:identical_support}: convert each cover $\mA$ to $\log_{K} (\abs{\mA_{L}})$ discrete variables of cardinality $K$ \& update $\cG'$} \label{step:convert_2}\;
    \Return{$\graphp$}
  }
\end{algorithm}
\end{figure}

%% file: sections/algorithms/algorithms1.tex
\begin{algorithm}[h]
    \small
     \caption{
     \textbf{Phase1: FindCISkeleton~\citep{dong2023versatile}} (Stage 1 of PC~\citep{spirtes2001causation}). We denote the joint probability table between two sets $\mA$ and $\mB$ as $\mP_{\mA, \mB}$, the adjacent nodes as $\text{Adj}, $the non-negative rank with $\text{rank}^{+}$, and the collection of d-separation sets as $\text{Sepset}$.
     We use blue color to highlight our modifications.
     }
     \label{alg:phase1}
     \SetAlgoLined
     \SetKwInOut{Input}{Input}
     \SetKwInOut{Output}{Output}
     \Input{Samples from observed variables $\mX_{\cG}$}
     \Output{CI skeleton $\graphp$}
     \SetKwProg{Def}{def}{:}{}
     \Def{\text{Stage1PC}($\set{X}_{\graph}$)}
     {
        Initialize a complete undirected graph $\graphp$ on $\set{X}_{\graph}$\;
        \Repeat{no adjacent $\node{X},\node{Y}$ s.t., $|\text{Adj}_{\graphp}(\node{X})\backslash\{\node{Y}\}|<n$}
        {
          \Repeat{all $\node{X},\node{Y}$ s.t., 
          $|\text{Adj}_{\graphp}(\node{X})\backslash\{\node{Y}\}|\geq n$ and all 
          $\set{S}\subseteq \text{Adj}_{\graphp}(\node{X})\backslash\{\node{Y}\}$, $|\set{S}|=n$, tested.}
          {
          Select an ordered pair $\node{X},\node{Y}$ that are adjacent in $\graphp$, s.t., 
          $|\text{Adj}_{\graphp}(\node{X})\backslash\{\node{Y}\}|\geq n$\;
          Select a subset $\set{S}\subseteq \text{Adj}_{\graphp}(\node{X})\backslash\{\node{Y}\}$
          s.t., $|\set{S}|=n$\;
          If $\ours{\text{rank}^{+}(\mP_{\{\node{X}\}\cup\set{S}, \{\node{Y}\}\cup\set{S}}) }=|\set{S}|$,
          delete the edge between $\node{X}$ and $\node{Y}$ from $\graphp$ 
          and record $\set{S}$ in
          $\text{Sepset}(\node{X},\node{Y})$ and $\text{Sepset}(\node{Y},\node{X})$.\;
          }
          n:=n+1\;
        }
       \Return $\graphp$
     }
   \end{algorithm}

%% file: sections/algorithms/algorithms2.tex
\begin{algorithm}[h]
 \small
  \caption{
    \textbf{Phase2: FindCausalClusters}~\citep{dong2023versatile}.
    We use $ ||\cdot|| $ to denote the number of all elements in a set of sets.
    We use the term "nodes" to refer to dummy variables that represent states rather than causal variables in the intermediate graph. 
    We use blue color to highlight our modifications.
  }
  \label{alg:phase2}
  \SetAlgoLined
  \SetKwInOut{Input}{Input}
  \SetKwInOut{Output}{Output}
  \Input{Samples from $n$ observed variables $\set{X}_{\graph}$}
  \Output{Graph $\graphp$}
  \SetKwProg{Def}{def}{:}{}
  \Def{\text{FindCausalClusters}($\graphp$, $\set{X}_{\graph}$)}
  {
    Active set $\setset{S} \gets \setset{X}_\graph=\{\{\node{X_1}\},...,\{\node{X_n}\}\}$, $k \gets 1$ \tcp*{$\setset{S}$ is a set of covers}\
    \Repeat{$k$ is sufficiently large}
    {
      $\graphp$, $\setset{S}$, $\text{found}$ = Search($\graphp$, $\setset{S}$, $\set{X}_{\graph}$, $k$) \tcp*{Only when nothing can be found }\
      If $\text{found}=1$ then $k\gets 1$ else $k\gets k+1$ \tcp*{udner current $k$ do we add $k$ by 1}
    }
    \Return $\graphp$\;
  }
  \SetKwProg{Def}{def}{:}{}
  \Def{\text{Search}($\graphp$, $\setset{S}$, $\set{X}_{\graph}$, $k$)}
  {
    Rank deficiency set $\setsetset{D}=\{\}$ \tcp*{To store rank deficient combinations}\
    \For{$\setset{T} \in \text{PowerSet}(\setset{S})$ (from $\setset{S}$ to $\emptyset$)} 
    {
      $\setset{S'} \gets (\setset{S} \backslash \setset{T}) \cup (\cup_{\set{T} \in \setset{T}} \purechildrenp(\set{T}))$ \tcp*{\footnotesize Unfold $\setset{S}$ to get $\setset{S'}$}
      \For{$t=k$ to $0$} 
      {
        \Repeat{all $\setset{X}$  exhausted}
        {
          Draw a set of $t$ observed covers $\setset{X} \subset \mathcal{S'}\cap \setset{X}_\graph$\;
          \Repeat{all $\setset{C}$ exhausted}
          {
            Draw a set of covers $\setset{C} \subset \setset{S'}\backslash \setset{X}$, s.t.,
            $||\setset{C}||=k-t+1$
            and get $\setset{N} \gets \setset{S'}\backslash (\setset{X} \cup \setset{C})$\;
            
            \lIf{$\ours{\text{rank}^{+}(\mP_{\setset{C}\cup\setset{X},\setset{N}\cup\setset{X})}} = k$ and  NoCollider($\setset{C}$, $\setset{X}$, $\setset{N})$}
            {
              Add $\setset{C}$ to $\setsetset{D}$
            }
          }
          \If{$\setsetset{D}\neq \emptyset$}
          {
            \For{$\setset{D}_i \in \setsetset{D}$}
            {
              \lIf {$|\parentsp(\setset{D}_i)\cup\mathbf{X}|=k$}
              {
                $\mathbf{P}\gets \parentsp(\setset{D}_i)\cup\mathbf{X}$
              }
              \lElse
              {
                Create new latent \ours{nodes} $\set{L}$, s.t., $\mathbf{P} \gets \mathbf{L}\cup\parentsp(\setset{D}_i)\cup\mathbf{X}$ 
                 and $|\set{L}|=k-|\parentsp(\setset{D}_i)\cup\mathbf{X}|$
              }
              Update $\graphp$ by taking elements of $\setset{D}_i$ as the pure children of $\set{P}$\; 
              \lIf {$\set{P}$ is atomic}
              {
                Update $\setset{S} \gets (\setset{S} \backslash \setset{D}_i) \cup \set{P}$
              }
            }
            \Return $\graphp$, $\setset{S}$, $\text{True}$ \tcp*{Return to search with $k=1$}
          }
        }
      }
    } 
    \Return $\graphp$, $\setset{S}$, $\text{False}$ \tcp*{Return to search with $k\gets k+1$}
  }
 \vspace{-1mm} 
\end{algorithm}
\setlength{\textfloatsep}{7pt}

%% file: sections/algorithms/algorithms3.tex

\begin{algorithm}[h]    
\small
  \caption{\textbf{NoCollider}~\citep{dong2023versatile}. We use blue color to highlight our modifications.}
  \SetAlgoLined
  \SetKwInOut{Input}{Input}
  \SetKwInOut{Output}{Output}
  \Input{$\setset{C}$, $\setset{X}$, $\setset{N}$}
  \Output{Whether  there exists $\set{O}\in\setset{C}$ s.t., $\set{O}$ is a collider of $\setset{C}\backslash\{\set{O}\}$ and $\setset{N}$}
  \SetKwProg{Def}{def}{:}{}
    \label{alg:checkcollider}
  \Def{\text{NoCollider}($\setset{C}$, $\setset{X}$, $\setset{N}$)}
  {
    \For{$c=1$ to $|\setset{C}|-1$} 
    {
      Draw $\setset{C'} \subset \setset{C}$ s.t., $|\setset{C'}|=c$\;
      \Repeat{all $\setset{C'}$ exhausted}
      {
        \lIf{$\ours{\text{rank}^{+}(\mP_{\setset{C'}\cup\setset{X},\setset{N}\cup\setset{X}})} <  ||\setset{C'}\cup\setset{X}||$}
        {
          \Return False
        }
      } 
    }
  \Return True
  }
  \vspace{-1mm}
\end{algorithm}
\setlength{\textfloatsep}{7pt}

%% file: sections/algorithms/algorithms4.tex

\begin{algorithm}[h]
\small
  \caption{Phase3: RefineCausalClusters}
  \label{alg:phase3}
  \SetKwInOut{Input}{Input}
  \SetKwInOut{Output}{Output}
  \Input{Graph $\graph'$}
  \Output{Refined graph $\graph'$}
  \SetKwProg{Def}{def}{:}{}
  \Def{RefineCausalCLusters($\graphp$, $\set{X}_\graph$)}
  {
    \Repeat{No more $\set{V}$ found and all $\set{V}$ exhausted}
    {
      Draw an atomic cover $\set{V}$ from $\graphp$\;
      Delete $\set{V}$, neighbours of $\set{V}$ that are latent, and all relating edges from $\graphp$ to get $\hat{\graph}$\;
      $\graphp =\text{FindCausalClusters}(\hat{\graph},\set{X}_\graph)$\; 
    }
    \Return{$\graphp$}
  }
\end{algorithm}

%% file: neurips_2024.bbl
\begin{thebibliography}{85}
\providecommand{\natexlab}[1]{#1}
\providecommand{\url}[1]{\texttt{#1}}
\expandafter\ifx\csname urlstyle\endcsname\relax
  \providecommand{\doi}[1]{doi: #1}\else
  \providecommand{\doi}{doi: \begingroup \urlstyle{rm}\Url}\fi

\bibitem[Gal et~al.(2022)Gal, Alaluf, Atzmon, Patashnik, Bermano, Chechik, and Cohen-or]{gal2022image}
Rinon Gal, Yuval Alaluf, Yuval Atzmon, Or~Patashnik, Amit~Haim Bermano, Gal Chechik, and Daniel Cohen-or.
\newblock An image is worth one word: Personalizing text-to-image generation using textual inversion.
\newblock In \emph{The Eleventh International Conference on Learning Representations}, 2022.

\bibitem[Jahanian et~al.(2019)Jahanian, Chai, and Isola]{jahanian2020steerability}
Ali Jahanian, Lucy Chai, and Phillip Isola.
\newblock On the" steerability" of generative adversarial networks.
\newblock In \emph{International Conference on Learning Representations}, 2019.

\bibitem[H{\"a}rk{\"o}nen et~al.(2020)H{\"a}rk{\"o}nen, Hertzmann, Lehtinen, and Paris]{härkönen2020ganspace}
Erik H{\"a}rk{\"o}nen, Aaron Hertzmann, Jaakko Lehtinen, and Sylvain Paris.
\newblock Ganspace: Discovering interpretable gan controls.
\newblock \emph{Advances in neural information processing systems}, 33:\penalty0 9841--9850, 2020.

\bibitem[Shen et~al.(2020)Shen, Gu, Tang, and Zhou]{shen2020interpreting}
Yujun Shen, Jinjin Gu, Xiaoou Tang, and Bolei Zhou.
\newblock Interpreting the latent space of gans for semantic face editing.
\newblock In \emph{Proceedings of the IEEE/CVF conference on computer vision and pattern recognition}, pages 9243--9252, 2020.

\bibitem[Wu et~al.(2021)Wu, Lischinski, and Shechtman]{wu2020stylespace}
Zongze Wu, Dani Lischinski, and Eli Shechtman.
\newblock Stylespace analysis: Disentangled controls for stylegan image generation.
\newblock In \emph{Proceedings of the IEEE/CVF conference on computer vision and pattern recognition}, pages 12863--12872, 2021.

\bibitem[Ruiz et~al.(2023)Ruiz, Li, Jampani, Pritch, Rubinstein, and Aberman]{ruiz2023dreambooth}
Nataniel Ruiz, Yuanzhen Li, Varun Jampani, Yael Pritch, Michael Rubinstein, and Kfir Aberman.
\newblock Dreambooth: Fine tuning text-to-image diffusion models for subject-driven generation.
\newblock In \emph{Proceedings of the IEEE/CVF Conference on Computer Vision and Pattern Recognition}, pages 22500--22510, 2023.

\bibitem[Burgess et~al.(2019)Burgess, Matthey, Watters, Kabra, Higgins, Botvinick, and Lerchner]{burgess2019monet}
Christopher~P Burgess, Loic Matthey, Nicholas Watters, Rishabh Kabra, Irina Higgins, Matt Botvinick, and Alexander Lerchner.
\newblock Monet: Unsupervised scene decomposition and representation.
\newblock \emph{arXiv preprint arXiv:1901.11390}, 2019.

\bibitem[Locatello et~al.(2020)Locatello, Weissenborn, Unterthiner, Mahendran, Heigold, Uszkoreit, Dosovitskiy, and Kipf]{locatello2020objectcentric}
Francesco Locatello, Dirk Weissenborn, Thomas Unterthiner, Aravindh Mahendran, Georg Heigold, Jakob Uszkoreit, Alexey Dosovitskiy, and Thomas Kipf.
\newblock Object-centric learning with slot attention.
\newblock In \emph{Advances in Neural Information Processing Systems 33 (NeurIPS 2020)}, 2020.

\bibitem[Du et~al.(2022{\natexlab{a}})Du, Li, Sharma, Tenenbaum, and Mordatch]{du2021unsupervised}
Yilun Du, Shuang Li, Yash Sharma, Joshua~B. Tenenbaum, and Igor Mordatch.
\newblock Unsupervised learning of compositional energy concepts.
\newblock In \emph{Proceedings of the International Conference on Learning Representations (ICLR 2022)}, 2022{\natexlab{a}}.

\bibitem[Du et~al.(2022{\natexlab{b}})Du, Smith, Ulman, Tenenbaum, and Wu]{du2021unsupervised3d}
Yilun Du, Kevin Smith, Tomer Ulman, Joshua Tenenbaum, and Jiajun Wu.
\newblock Unsupervised discovery of 3d physical objects from video.
\newblock In \emph{Conference on Computer Vision and Pattern Recognition (CVPR 2022)}, 2022{\natexlab{b}}.

\bibitem[Liu et~al.(2023)Liu, Du, Li, Tenenbaum, and Torralba]{liu2023unsupervised}
Nan Liu, Yilun Du, Shuang Li, Joshua~B. Tenenbaum, and Antonio Torralba.
\newblock Unsupervised compositional concepts discovery with text-to-image generative models.
\newblock In \emph{Proceedings of the 2023 Conference on Neural Information Processing Systems (NeurIPS 2023)}, 2023.

\bibitem[Koh et~al.(2020)Koh, Nguyen, Tang, Mussmann, Pierson, Kim, and Liang]{koh2020concept}
Pang~Wei Koh, Thao Nguyen, Yew~Siang Tang, Stephen Mussmann, Emma Pierson, Been Kim, and Percy Liang.
\newblock Concept bottleneck models.
\newblock In \emph{International conference on machine learning}, pages 5338--5348. PMLR, 2020.

\bibitem[Zarlenga et~al.(2022)Zarlenga, Barbiero, Ciravegna, Marra, Giannini, Diligenti, Precioso, Melacci, Weller, Lio, et~al.]{zarlenga2022concept}
Mateo~Espinosa Zarlenga, Pietro Barbiero, Gabriele Ciravegna, Giuseppe Marra, Francesco Giannini, Michelangelo Diligenti, Frederic Precioso, Stefano Melacci, Adrian Weller, Pietro Lio, et~al.
\newblock Concept embedding models.
\newblock In \emph{NeurIPS 2022-36th Conference on Neural Information Processing Systems}, 2022.

\bibitem[Oikarinen and Weng(2022)]{oikarinen2022clip}
Tuomas Oikarinen and Tsui-Wei Weng.
\newblock Clip-dissect: Automatic description of neuron representations in deep vision networks.
\newblock In \emph{ICLR 2022 Workshop on PAIR $\{$$\backslash$textasciicircum$\}$ 2Struct: Privacy, Accountability, Interpretability, Robustness, Reasoning on Structured Data}, 2022.

\bibitem[Moayeri et~al.(2023{\natexlab{a}})Moayeri, Rezaei, Sanjabi, and Feizi]{moayeri2023text}
Mazda Moayeri, Keivan Rezaei, Maziar Sanjabi, and Soheil Feizi.
\newblock Text-to-concept (and back) via cross-model alignment.
\newblock In \emph{International Conference on Machine Learning}, pages 25037--25060. PMLR, 2023{\natexlab{a}}.

\bibitem[Moayeri et~al.(2023{\natexlab{b}})Moayeri, Rezaei, Sanjabi, and Feizi]{moayeri2023text2concept}
Mazda Moayeri, Keivan Rezaei, Maziar Sanjabi, and Soheil Feizi.
\newblock Text2concept: Concept activation vectors directly from text.
\newblock In \emph{Proceedings of the IEEE/CVF Conference on Computer Vision and Pattern Recognition}, pages 3743--3748, 2023{\natexlab{b}}.

\bibitem[Radford et~al.(2021)Radford, Kim, Hallacy, Ramesh, Goh, Agarwal, Sastry, Askell, Mishkin, Clark, et~al.]{radford2021learning}
Alec Radford, Jong~Wook Kim, Chris Hallacy, Aditya Ramesh, Gabriel Goh, Sandhini Agarwal, Girish Sastry, Amanda Askell, Pamela Mishkin, Jack Clark, et~al.
\newblock Learning transferable visual models from natural language supervision.
\newblock In \emph{International conference on machine learning}, pages 8748--8763. PMLR, 2021.

\bibitem[Xie et~al.(2022)Xie, Huang, Chen, He, Geng, and Zhang]{xie2022identification}
Feng Xie, Biwei Huang, Zhengming Chen, Yangbo He, Zhi Geng, and Kun Zhang.
\newblock Identification of linear non-gaussian latent hierarchical structure.
\newblock In \emph{International Conference on Machine Learning}, pages 24370--24387. PMLR, 2022.

\bibitem[Huang et~al.(2022)Huang, Low, Xie, Glymour, and Zhang]{huang2022latent}
Biwei Huang, Charles Jia~Han Low, Feng Xie, Clark Glymour, and Kun Zhang.
\newblock Latent hierarchical causal structure discovery with rank constraints.
\newblock \emph{Advances in Neural Information Processing Systems}, 35:\penalty0 5549--5561, 2022.

\bibitem[Dong et~al.(2023)Dong, Huang, Ng, Song, Zheng, Jin, Legaspi, Spirtes, and Zhang]{dong2023versatile}
Xinshuai Dong, Biwei Huang, Ignavier Ng, Xiangchen Song, Yujia Zheng, Songyao Jin, Roberto Legaspi, Peter Spirtes, and Kun Zhang.
\newblock A versatile causal discovery framework to allow causally-related hidden variables.
\newblock In \emph{The Twelfth International Conference on Learning Representations}, 2023.

\bibitem[Pearl(1988)]{Pearl88}
J.~Pearl.
\newblock \emph{Probabilistic Reasoning in Intelligent Systems: Networks of Plausible Inference}.
\newblock Morgan Kaufmann, 1988.

\bibitem[Choi et~al.(2011)Choi, Tan, Anandkumar, and Willsky]{choi2011learning}
Myung~Jin Choi, Vincent~YF Tan, Animashree Anandkumar, and Alan~S Willsky.
\newblock Learning latent tree graphical models.
\newblock \emph{Journal of Machine Learning Research}, 12:\penalty0 1771--1812, 2011.

\bibitem[Gu and Dunson(2023)]{gu2023bayesian}
Yuqi Gu and David~B. Dunson.
\newblock Bayesian pyramids: Identifiable multilayer discrete latent structure models for discrete data.
\newblock In \emph{Journal of the Royal Statistical Society Series B: Statistical Methodology}, 2023.

\bibitem[Kivva et~al.(2021)Kivva, Rajendran, Ravikumar, and Aragam]{kivva2021learning}
Bohdan Kivva, Goutham Rajendran, Pradeep Ravikumar, and Bryon Aragam.
\newblock Learning latent causal graphs via mixture oracles.
\newblock \emph{Advances in Neural Information Processing Systems}, 34:\penalty0 18087--18101, 2021.

\bibitem[Drton et~al.(2017)Drton, Lin, Weihs, and Zwiernik]{drton2015marginal}
Mathias Drton, Shaowei Lin, Luca Weihs, and Piotr Zwiernik.
\newblock Marginal likelihood and model selection for gaussian latent tree and forest models.
\newblock \emph{Bernoulli}, pages 1202--1232, 2017.

\bibitem[Zhang(2004)]{zhang2004hierarchical}
Nevin~L Zhang.
\newblock Hierarchical latent class models for cluster analysis.
\newblock \emph{The Journal of Machine Learning Research}, 5:\penalty0 697--723, 2004.

\bibitem[Anandkumar et~al.(2013{\natexlab{a}})Anandkumar, Hsu, Javanmard, and Kakade]{anandkumar13learning}
Animashree Anandkumar, Daniel Hsu, Adel Javanmard, and Sham Kakade.
\newblock Learning linear bayesian networks with latent variables.
\newblock In Sanjoy Dasgupta and David McAllester, editors, \emph{Proceedings of the 30th International Conference on Machine Learning}, volume~28 of \emph{Proceedings of Machine Learning Research}, pages 249--257, Atlanta, Georgia, USA, 17--19 Jun 2013{\natexlab{a}}. PMLR.
\newblock URL \url{https://proceedings.mlr.press/v28/anandkumar13.html}.

\bibitem[Rombach et~al.(2022)Rombach, Blattmann, Lorenz, Esser, and Ommer]{rombach2021highresolution}
Robin Rombach, Andreas Blattmann, Dominik Lorenz, Patrick Esser, and Björn Ommer.
\newblock High-resolution image synthesis with latent diffusion models.
\newblock In \emph{Proceedings of the 2022 IEEE/CVF Conference on Computer Vision and Pattern Recognition (CVPR 2022)}, 2022.

\bibitem[Yuksekgonul et~al.(2023)Yuksekgonul, Wang, and Zou]{yuksekgonulpost}
Mert Yuksekgonul, Maggie Wang, and James Zou.
\newblock Post-hoc concept bottleneck models.
\newblock In \emph{The Eleventh International Conference on Learning Representations}, 2023.

\bibitem[Kim et~al.(2023)Kim, Jung, Park, Kim, and Yoon]{kim2023probabilistic}
Eunji Kim, Dahuin Jung, Sangha Park, Siwon Kim, and Sungroh Yoon.
\newblock Probabilistic concept bottleneck models.
\newblock In \emph{International Conference on Machine Learning}, pages 16521--16540. PMLR, 2023.

\bibitem[Havasi et~al.(2022)Havasi, Parbhoo, and Doshi-Velez]{havasi2022addressing}
Marton Havasi, Sonali Parbhoo, and Finale Doshi-Velez.
\newblock Addressing leakage in concept bottleneck models.
\newblock In Alice~H. Oh, Alekh Agarwal, Danielle Belgrave, and Kyunghyun Cho, editors, \emph{Advances in Neural Information Processing Systems}, 2022.
\newblock URL \url{https://openreview.net/forum?id=tglniD_fn9}.

\bibitem[Shang et~al.(2024)Shang, Zhou, Zhang, Ni, Yang, and Wang]{shang2024incremental}
Chenming Shang, Shiji Zhou, Hengyuan Zhang, Xinzhe Ni, Yujiu Yang, and Yuwang Wang.
\newblock Incremental residual concept bottleneck models.
\newblock In \emph{Proceedings of the IEEE/CVF Conference on Computer Vision and Pattern Recognition}, pages 11030--11040, 2024.

\bibitem[Chauhan et~al.(2023)Chauhan, Tiwari, Freyberg, Shenoy, and Dvijotham]{chauhan2023interactive}
Kushal Chauhan, Rishabh Tiwari, Jan Freyberg, Pradeep Shenoy, and Krishnamurthy Dvijotham.
\newblock Interactive concept bottleneck models.
\newblock In \emph{Proceedings of the AAAI Conference on Artificial Intelligence}, volume~37, pages 5948--5955, 2023.

\bibitem[Kong et~al.(2023)Kong, Huang, Xie, Xing, Chi, and Zhang]{kong2023identification}
Lingjing Kong, Biwei Huang, Feng Xie, Eric Xing, Yuejie Chi, and Kun Zhang.
\newblock Identification of nonlinear latent hierarchical models.
\newblock \emph{Advances in Neural Information Processing Systems}, 36, 2023.

\bibitem[Khemakhem et~al.(2020{\natexlab{a}})Khemakhem, Kingma, Monti, and Hyvarinen]{khemakhem2020variational}
Ilyes Khemakhem, Diederik Kingma, Ricardo Monti, and Aapo Hyvarinen.
\newblock Variational autoencoders and nonlinear ica: A unifying framework.
\newblock In \emph{International Conference on Artificial Intelligence and Statistics}, pages 2207--2217. PMLR, 2020{\natexlab{a}}.

\bibitem[Khemakhem et~al.(2020{\natexlab{b}})Khemakhem, Monti, Kingma, and Hyvarinen]{khemakhem2020icebeem}
Ilyes Khemakhem, Ricardo Monti, Diederik Kingma, and Aapo Hyvarinen.
\newblock Ice-beem: Identifiable conditional energy-based deep models based on nonlinear ica.
\newblock \emph{Advances in Neural Information Processing Systems}, 33:\penalty0 12768--12778, 2020{\natexlab{b}}.

\bibitem[Hyvarinen and Morioka(2016)]{hyvarinen2016unsupervised}
Aapo Hyvarinen and Hiroshi Morioka.
\newblock Unsupervised feature extraction by time-contrastive learning and nonlinear ica.
\newblock \emph{Advances in neural information processing systems}, 29, 2016.

\bibitem[Hyvarinen et~al.(2019)Hyvarinen, Sasaki, and Turner]{hyvarinen2019nonlinear}
Aapo Hyvarinen, Hiroaki Sasaki, and Richard Turner.
\newblock Nonlinear ica using auxiliary variables and generalized contrastive learning.
\newblock In \emph{The 22nd International Conference on Artificial Intelligence and Statistics}, pages 859--868. PMLR, 2019.

\bibitem[Brady et~al.(2024)Brady, Zimmermann, Sharma, Schölkopf, von Kügelgen, and Brendel]{brady2023provably}
Jack Brady, Roland~S. Zimmermann, Yash Sharma, Bernhard Schölkopf, Julius von Kügelgen, and Wieland Brendel.
\newblock Provably learning object-centric representations.
\newblock In \emph{International Conference on Learning Representations (ICLR 2024)}, 2024.

\bibitem[Lachapelle et~al.(2024)Lachapelle, Mahajan, Mitliagkas, and Lacoste-Julien]{lachapelle2023additive}
S{\'e}bastien Lachapelle, Divyat Mahajan, Ioannis Mitliagkas, and Simon Lacoste-Julien.
\newblock Additive decoders for latent variables identification and cartesian-product extrapolation.
\newblock \emph{Advances in Neural Information Processing Systems}, 36, 2024.

\bibitem[Kivva et~al.(2022)Kivva, Rajendran, Ravikumar, and Aragam]{kivva2022identifiability}
Bohdan Kivva, Goutham Rajendran, Pradeep Ravikumar, and Bryon Aragam.
\newblock Identifiability of deep generative models without auxiliary information.
\newblock \emph{Advances in Neural Information Processing Systems}, 35:\penalty0 15687--15701, 2022.

\bibitem[Von~K{\"u}gelgen et~al.(2021)Von~K{\"u}gelgen, Sharma, Gresele, Brendel, Sch{\"o}lkopf, Besserve, and Locatello]{vonkugelgen2021selfsupervised}
Julius Von~K{\"u}gelgen, Yash Sharma, Luigi Gresele, Wieland Brendel, Bernhard Sch{\"o}lkopf, Michel Besserve, and Francesco Locatello.
\newblock Self-supervised learning with data augmentations provably isolates content from style.
\newblock \emph{Advances in neural information processing systems}, 34:\penalty0 16451--16467, 2021.

\bibitem[Kong et~al.(2022)Kong, Xie, Yao, Zheng, Chen, Stojanov, Akinwande, and Zhang]{kong2022partial}
Lingjing Kong, Shaoan Xie, Weiran Yao, Yujia Zheng, Guangyi Chen, Petar Stojanov, Victor Akinwande, and Kun Zhang.
\newblock Partial disentanglement for domain adaptation.
\newblock In \emph{International Conference on Machine Learning}, pages 11455--11472. PMLR, 2022.

\bibitem[Arora et~al.(2012)Arora, Ge, and Moitra]{arora2012learning}
Sanjeev Arora, Rong Ge, and Ankur Moitra.
\newblock Learning topic models--going beyond svd.
\newblock In \emph{2012 IEEE 53rd annual symposium on foundations of computer science}, pages 1--10. IEEE, 2012.

\bibitem[Arora et~al.(2013)Arora, Ge, Halpern, Mimno, Moitra, Sontag, Wu, and Zhu]{arora2012practical}
Sanjeev Arora, Rong Ge, Yonatan Halpern, David Mimno, Ankur Moitra, David Sontag, Yichen Wu, and Michael Zhu.
\newblock A practical algorithm for topic modeling with provable guarantees.
\newblock In \emph{International conference on machine learning}, pages 280--288. PMLR, 2013.

\bibitem[Moran et~al.(2021)Moran, Sridhar, Wang, and Blei]{moran2021identifiable}
Gemma~E Moran, Dhanya Sridhar, Yixin Wang, and David~M Blei.
\newblock Identifiable variational autoencoders via sparse decoding.
\newblock \emph{arXiv preprint arXiv:2110.10804}, 2021.

\bibitem[Zheng et~al.(2022)Zheng, Ng, and Zhang]{zheng2022identifiability}
Yujia Zheng, Ignavier Ng, and Kun Zhang.
\newblock On the identifiability of nonlinear ica: Sparsity and beyond.
\newblock \emph{arXiv preprint arXiv:2206.07751}, 2022.

\bibitem[Gu(2022)]{gu2023blessing}
Yuqi Gu.
\newblock Blessing of dependence: Identifiability and geometry of discrete models with multiple binary latent variables.
\newblock \emph{arXiv preprint arXiv:2203.04403}, 2022.

\bibitem[Sullivant et~al.(2010)Sullivant, Talaska, and Draisma]{Sullivant_2010}
Seth Sullivant, Kelli Talaska, and Jan Draisma.
\newblock Trek separation for gaussian graphical models.
\newblock \emph{The Annals of Statistics}, 38\penalty0 (3), June 2010.
\newblock ISSN 0090-5364.
\newblock \doi{10.1214/09-aos760}.
\newblock URL \url{http://dx.doi.org/10.1214/09-AOS760}.

\bibitem[Cohen and Rothblum(1993)]{cohen1993nonnegative}
Joel~E Cohen and Uriel~G Rothblum.
\newblock Nonnegative ranks, decompositions, and factorizations of nonnegative matrices.
\newblock \emph{Linear Algebra and its Applications}, 190:\penalty0 149--168, 1993.

\bibitem[Spirtes et~al.(2001)Spirtes, Glymour, and Scheines]{spirtes2001causation}
Peter Spirtes, Clark Glymour, and Richard Scheines.
\newblock \emph{Causation, prediction, and search}.
\newblock MIT press, 2001.

\bibitem[Anandkumar et~al.(2012)Anandkumar, Hsu, Huang, and Kakade]{anandkumar2012learning}
A.~Anandkumar, D.~Hsu, F.~Huang, and S.~M. Kakade.
\newblock Learning high-dimensional mixtures of graphical models, 2012.

\bibitem[Mazaheri et~al.(2023)Mazaheri, Gordon, Rabani, and Schulman]{mazaheri2023causal}
Bijan Mazaheri, Spencer Gordon, Yuval Rabani, and Leonard Schulman.
\newblock Causal discovery under latent class confounding.
\newblock \emph{arXiv preprint arXiv:2311.07454}, 2023.

\bibitem[Xie et~al.(2020)Xie, Cai, Huang, Glymour, Hao, and Zhang]{xie2020generalized}
Feng Xie, Ruichu Cai, Biwei Huang, Clark Glymour, Zhifeng Hao, and Kun Zhang.
\newblock Generalized independent noise condition for estimating latent variable causal graphs.
\newblock \emph{Advances in Neural Information Processing Systems}, 33:\penalty0 14891--14902, 2020.

\bibitem[Lemeire and Janzing(2013)]{lemeire2013replacing}
Jan Lemeire and Dominik Janzing.
\newblock Replacing causal faithfulness with algorithmic independence of conditionals.
\newblock \emph{Minds and Machines}, 23:\penalty0 227--249, 2013.

\bibitem[Vincent(2011)]{vincent2011connection}
Pascal Vincent.
\newblock A connection between score matching and denoising autoencoders.
\newblock \emph{Neural computation}, 23\penalty0 (7):\penalty0 1661--1674, 2011.

\bibitem[Song and Ermon(2019)]{song2019generative}
Yang Song and Stefano Ermon.
\newblock Generative modeling by estimating gradients of the data distribution.
\newblock \emph{Advances in neural information processing systems}, 32, 2019.

\bibitem[Vincent et~al.(2008)Vincent, Larochelle, Bengio, and Manzagol]{vincent2008extracting}
Pascal Vincent, Hugo Larochelle, Yoshua Bengio, and Pierre-Antoine Manzagol.
\newblock Extracting and composing robust features with denoising autoencoders.
\newblock In \emph{Proceedings of the 25th international conference on Machine learning}, pages 1096--1103, 2008.

\bibitem[Vincent et~al.(2010)Vincent, Larochelle, Lajoie, Bengio, Manzagol, and Bottou]{vincent2010stacked}
Pascal Vincent, Hugo Larochelle, Isabelle Lajoie, Yoshua Bengio, Pierre-Antoine Manzagol, and L{\'e}on Bottou.
\newblock Stacked denoising autoencoders: Learning useful representations in a deep network with a local denoising criterion.
\newblock \emph{Journal of machine learning research}, 11\penalty0 (12), 2010.

\bibitem[Pathak et~al.(2016)Pathak, Krahenbuhl, Donahue, Darrell, and Efros]{pathak2016context}
Deepak Pathak, Philipp Krahenbuhl, Jeff Donahue, Trevor Darrell, and Alexei~A. Efros.
\newblock Context encoders: Feature learning by inpainting, 2016.

\bibitem[He et~al.(2021)He, Chen, Xie, Li, Dollár, and Girshick]{he2021masked}
Kaiming He, Xinlei Chen, Saining Xie, Yanghao Li, Piotr Dollár, and Ross Girshick.
\newblock Masked autoencoders are scalable vision learners, 2021.

\bibitem[Park et~al.(2023)Park, Kwon, Choi, Jo, and Uh]{park2023understanding}
Yong-Hyun Park, Mingi Kwon, Jaewoong Choi, Junghyo Jo, and Youngjung Uh.
\newblock Understanding the latent space of diffusion models through the lens of riemannian geometry.
\newblock \emph{Advances in Neural Information Processing Systems}, 36:\penalty0 24129--24142, 2023.

\bibitem[Esser et~al.(2021)Esser, Rombach, and Ommer]{esser2021taming}
Patrick Esser, Robin Rombach, and Björn Ommer.
\newblock Taming transformers for high-resolution image synthesis.
\newblock In \emph{Proceedings of the IEEE/CVF Conference on Computer Vision and Pattern Recognition (CVPR 2021)}, 2021.

\bibitem[Gu et~al.(2022)Gu, Chen, Bao, Wen, Zhang, Chen, Yuan, and Guo]{gu2022vector}
Shuyang Gu, Dong Chen, Jianmin Bao, Fang Wen, Bo~Zhang, Dongdong Chen, Lu~Yuan, and Baining Guo.
\newblock Vector quantized diffusion model for text-to-image synthesis, 2022.

\bibitem[Gandikota et~al.(2023)Gandikota, Materzy\'nska, Zhou, Torralba, and Bau]{gandikota2023sliders}
Rohit Gandikota, Joanna Materzy\'nska, Tingrui Zhou, Antonio Torralba, and David Bau.
\newblock Concept sliders: Lora adaptors for precise control in diffusion models.
\newblock \emph{arXiv preprint arXiv:2311.12092}, 2023.

\bibitem[Hu et~al.(2021)Hu, Shen, Wallis, Allen-Zhu, Li, Wang, Wang, and Chen]{hu2021lora}
Edward~J. Hu, Yelong Shen, Phillip Wallis, Zeyuan Allen-Zhu, Yuanzhi Li, Shean Wang, Lu~Wang, and Weizhu Chen.
\newblock Lora: Low-rank adaptation of large language models.
\newblock In \emph{Proceedings of the 2021 International Conference on Learning Representations (ICLR 2021)}, 2021.

\bibitem[Ding et~al.(2023)Ding, Lv, Wang, Chen, Zhou, Liu, and Sun]{ding2023sparse}
Ning Ding, Xingtai Lv, Qiaosen Wang, Yulin Chen, Bowen Zhou, Zhiyuan Liu, and Maosong Sun.
\newblock Sparse low-rank adaptation of pre-trained language models.
\newblock In \emph{Proceedings of the 2023 Conference on Empirical Methods in Natural Language Processing (EMNLP 2023)}, 2023.

\bibitem[Anandkumar et~al.(2013{\natexlab{b}})Anandkumar, Hsu, Javanmard, and Kakade]{anandkumar2013learning}
Animashree Anandkumar, Daniel Hsu, Adel Javanmard, and Sham Kakade.
\newblock Learning linear bayesian networks with latent variables.
\newblock In \emph{International Conference on Machine Learning}, pages 249--257. PMLR, 2013{\natexlab{b}}.

\bibitem[S{\o}nderby et~al.(2016)S{\o}nderby, Raiko, Maal{\o}e, S{\o}nderby, and Winther]{sonderby2016ladder}
Casper~Kaae S{\o}nderby, Tapani Raiko, Lars Maal{\o}e, S{\o}ren~Kaae S{\o}nderby, and Ole Winther.
\newblock Ladder variational autoencoders.
\newblock \emph{Advances in neural information processing systems}, 29, 2016.

\bibitem[Zhao et~al.(2017)Zhao, Song, and Ermon]{zhao2017learning}
Shengjia Zhao, Jiaming Song, and Stefano Ermon.
\newblock Learning hierarchical features from deep generative models.
\newblock In \emph{International Conference on Machine Learning}, pages 4091--4099. PMLR, 2017.

\bibitem[Li et~al.(2018)Li, Chen, Poon, and Zhang]{li2019learning}
Xiaopeng Li, Zhourong Chen, Leonard~KM Poon, and Nevin~L Zhang.
\newblock Learning latent superstructures in variational autoencoders for deep multidimensional clustering.
\newblock \emph{arXiv preprint arXiv:1803.05206}, 2018.

\bibitem[Leeb et~al.(2022)Leeb, Lanzillotta, Annadani, Besserve, Bauer, and Sch{\"o}lkopf]{leeb2024structure}
Felix Leeb, Giulia Lanzillotta, Yashas Annadani, Michel Besserve, Stefan Bauer, and Bernhard Sch{\"o}lkopf.
\newblock Structure by architecture: Structured representations without regularization.
\newblock In \emph{The Eleventh International Conference on Learning Representations}, 2022.

\bibitem[Ross and Doshi-Velez(2021)]{ross2022benchmarks}
Andrew Ross and Finale Doshi-Velez.
\newblock Benchmarks, algorithms, and metrics for hierarchical disentanglement.
\newblock In \emph{International Conference on Machine Learning}, pages 9084--9094. PMLR, 2021.

\bibitem[Sohl-Dickstein et~al.(2015)Sohl-Dickstein, Weiss, Maheswaranathan, and Ganguli]{sohldickstein2015deep}
Jascha Sohl-Dickstein, Eric~A. Weiss, Niru Maheswaranathan, and Surya Ganguli.
\newblock Deep unsupervised learning using nonequilibrium thermodynamics.
\newblock In \emph{Proceedings of the International Conference on Machine Learning (ICML 2015)}, 2015.

\bibitem[Ho et~al.(2020)Ho, Jain, and Abbeel]{ho2020denoising}
Jonathan Ho, Ajay Jain, and Pieter Abbeel.
\newblock Denoising diffusion probabilistic models.
\newblock In \emph{Advances in Neural Information Processing Systems 33 (NeurIPS 2020)}, 2020.

\bibitem[Song et~al.(2022)Song, Meng, and Ermon]{song2022denoising}
Jiaming Song, Chenlin Meng, and Stefano Ermon.
\newblock Denoising diffusion implicit models.
\newblock In \emph{International Conference on Learning Representations (ICLR 2022)}, 2022.

\bibitem[Dhariwal and Nichol(2021)]{dhariwal2021diffusion}
Prafulla Dhariwal and Alex Nichol.
\newblock Diffusion models beat gans on image synthesis.
\newblock In \emph{Advances in Neural Information Processing Systems 34 (NeurIPS 2021)}, 2021.

\bibitem[Nichol and Dhariwal(2021)]{nichol2021improved}
Alex Nichol and Prafulla Dhariwal.
\newblock Improved denoising diffusion probabilistic models.
\newblock In \emph{Proceedings of the International Conference on Machine Learning (ICML 2021)}, 2021.

\bibitem[Kwon et~al.(2022)Kwon, Jeong, and Uh]{kwon2023diffusion}
Mingi Kwon, Jaeseok Jeong, and Youngjung Uh.
\newblock Diffusion models already have a semantic latent space.
\newblock In \emph{The Eleventh International Conference on Learning Representations}, 2022.

\bibitem[Choi et~al.(2022)Choi, Lee, Shin, Kim, Kim, and Yoon]{choi2022perception}
Jooyoung Choi, Jungbeom Lee, Chaehun Shin, Sungwon Kim, Hyunwoo Kim, and Sungroh Yoon.
\newblock Perception prioritized training of diffusion models.
\newblock In \emph{Proceedings of the IEEE/CVF Conference on Computer Vision and Pattern Recognition}, pages 11472--11481, 2022.

\bibitem[Daras and Dimakis(2022)]{daras2022multiresolution}
Giannis Daras and Alexandros~G Dimakis.
\newblock Multiresolution textual inversion.
\newblock \emph{arXiv preprint arXiv:2211.17115}, 2022.

\bibitem[Wu et~al.(2023)Wu, Liu, Zhao, Kale, Bui, Yu, Lin, Zhang, and Chang]{wu2022uncovering}
Qiucheng Wu, Yujian Liu, Handong Zhao, Ajinkya Kale, Trung Bui, Tong Yu, Zhe Lin, Yang Zhang, and Shiyu Chang.
\newblock Uncovering the disentanglement capability in text-to-image diffusion models.
\newblock In \emph{Proceedings of the IEEE/CVF Conference on Computer Vision and Pattern Recognition}, pages 1900--1910, 2023.

\bibitem[Sclocchi et~al.(2024)Sclocchi, Favero, and Wyart]{sclocchi2024phase}
Antonio Sclocchi, Alessandro Favero, and Matthieu Wyart.
\newblock A phase transition in diffusion models reveals the hierarchical nature of data.
\newblock \emph{arXiv preprint arXiv:2402.16991}, 2024.

\bibitem[Pearl(2009)]{pearl2009causality}
Judea Pearl.
\newblock \emph{Causality}.
\newblock Cambridge university press, 2009.

\bibitem[Di(2009)]{di2009t}
Yanming Di.
\newblock t-separation and d-separation for directed acyclic graphs.
\newblock \emph{preprint}, 2009.

\end{thebibliography}
